\documentclass[a4paper,11pt]{article}
\usepackage[top=2.5cm, bottom=2.5cm, left=2cm, right=2cm]{geometry}
\usepackage{settings_custom}
\allowdisplaybreaks
\usepackage{marginnote}

\title{Fundamental Limits of Matrix Sensing:\\ Exact Asymptotics, Universality, and Applications}
\date{\today}
\author{
Yizhou Xu$^{1,3}$, Antoine Maillard$^{2}$, Lenka Zdeborov\'a$^{1}$ and Florent Krzakala$^{3}$\\
$^1$\textit{Statistical Physics of Computation Laboratory (SPOC), EPFL, Switzerland}\\
$^2$\textit{Inria Paris, DI ENS, PSL University, Paris, France}\\
$^3$\textit{Information, Learning, and Physics Laboratory (IDEPHICS), EPFL, Switzerland}
}

\date{}

\begin{document}

\maketitle
\begin{abstract}%
In the matrix sensing problem, one wishes to reconstruct a matrix from (possibly noisy) observations of its linear projections along given directions. 
We consider this model in the high-dimensional limit:
while previous works on this model primarily focused on the recovery of low-rank matrices, we consider in this work more general classes of structured signal matrices with potentially large rank, e.g.\ a product of two matrices of sizes proportional to the dimension.
We provide rigorous asymptotic equations characterizing the Bayes-optimal learning performance from a number of samples which is proportional to the number of entries in the matrix.
Our proof is composed of three key ingredients: $(i)$ we prove universality properties to handle structured sensing matrices, related to the ``Gaussian equivalence'' phenomenon in statistical learning, $(ii)$ we provide a sharp characterization of Bayes-optimal learning in generalized linear models with Gaussian data and structured matrix priors, generalizing previously studied settings, and $(iii)$ we leverage previous works on the problem of matrix denoising.
  The generality of our results allow for a variety of applications: notably, we mathematically establish predictions obtained via non-rigorous methods from statistical physics in~\cite{erba2024bilinear} regarding Bilinear Sequence Regression, a benchmark model for learning from sequences of tokens, and in~\cite{maillard2024bayes} on Bayes-optimal learning in neural networks with quadratic activation function, and width proportional to the dimension.
\end{abstract}
\tableofcontents

\vspace{-3mm}
\section{Introduction}\label{sec:intro}
\vspace{-2mm}
\subsection{Setting of the problem}
\vspace{-1mm}
Consider the problem of learning a matrix 
$S^\star\in\bbR^{d\times L}$ from noisy and possibly non-linear observations, $\mu=1,\ldots,n$, generated from the following model with $\Phi_\mu\in\bbR^{L\times d}$:
\begin{equation}\label{eq:matrix_glm} Y_\mu = \varphi(\Tr[\Phi_\mu S^\star], a_\mu) + \sqrt{\Delta} Z_\mu, \end{equation}
where $\Delta > 0$, and $\{a_\mu\}_{\mu=1}^n \iid P_A$ and $\{Z_\mu\}_{\mu=1}^n \iid \mcN(0, 1)$ represent the noise in the observations. 
We assume that the data samples $\Phi_\mu$ are drawn from a known probability distribution $\{\Phi_\mu\}_{\mu=1}^{n}\iid P_\Phi$. 
The signal $S^\star$ is drawn from a rotationally-invariant prior $S^\star\sim P_0$.
This setting encompasses cases of particular interest such as Wishart matrices, and products of random matrices. We will consider a quite general data distribution $P_\Phi$ (as specified later, see Assumptions~\ref{ass:prior} and \ref{ass:prior_rec}), with two notable motivating examples: $\Phi$ as an i.i.d.\ matrix with standard Gaussian elements, and $\Phi=(x x^\T - \Id_d)/\sqrt{d}$, with $x\sim\mcN(0,\Id_d)$, two models which connect to settings studied non-rigorously in previous works~\cite{maillard2024bayes,erba2024bilinear}.
Notice that we may also write the output channel of eq.~\eqref{eq:matrix_glm} in the equivalent form:
\begin{equation}\label{eq:matrix_glm_Pout}
  Y_\mu\sim P_{\out}(\cdot|\Tr[\Phi_\mu S^\star]), \, \,  {\rm with} \, \,    P_\out(y|z) \coloneqq \frac{1}{\sqrt{2\pi\Delta}} \int P_A(\rd a) \exp\left\{-\frac{1}{2\Delta} (y - \varphi(z, a))^2\right\}. 
\end{equation}
In this paper, we consider the Bayes-optimal, or information theoretic, scenario. The observed dataset $\mcD = \{Y_\mu, \Phi_\mu\}_{\mu=1}^n$ is given to the statistician who is provided all the information on how the data has been generated, i.e.\ the value of $\Delta$, the form of the function $\varphi(\cdot,\cdot)$, and the prior distributions $P_0$, $P_A$, from which she must learn and provide the best estimates of the values of the weights $S^\star$ (or of the prediction error on the output for a test input sample), by averaging over the posterior distribution. The key interest of such an information-theoretic analysis is to provide a sharp characterization of the best reconstruction performance that {\it any} algorithm can achieve.

\myskip
We consider a high-dimensional setting where $L,d\to\infty$, with fixed ratio $d/L \to \beta > 0$. Without loss of generality, we suppose that $L\geq d$ (and thus $\beta \leq 1$): 
if $\beta>1$ one can consider instead $\Tr[\Phi^T_\mu (S^\star)^T]=\Tr[\Phi_\mu S^\star]$.
While our results apply to more generic prior distributions $P_0$ of the signal $S^\star$, we will especially focus on the case where $S^\star$ is a product of random matrices of size $d \times m$ and $m \times L$, in which case the width $m$ is such that $m/d \to  \kappa >0$, and quantifies the amount of ``structure'' in the signal. 
The number of samples $n$ also satisfies $n\to\infty$, with a fixed ratio $n / (Ld) \to \alpha > 0$: this ensures that the number of samples scales linearly with the number of unknown parameters, and that one can perform non-trivial estimation of $S^\star$.

\myskip
Our aim is to characterize the Bayes-optimal performance in the stated limit. Our analysis leverages the fact that in many quantities of interest, such as the prediction and estimation error per sample, concentrate as $d \to \infty$ on deterministic values. We derive low-dimensional equations from which these values of interest can be readily extracted.  

\myskip 
\textbf{A symmetric variant --}
While all our analysis holds for the model of eq.~\eqref{eq:matrix_glm}, we shall also consider in what follows a ``symmetrized'' variant of eq.~\eqref{eq:matrix_glm}, where $L = d$, and both $S^\star$ and $\Phi_\mu$ are symmetric. We will start the discussion of our results in this symmetric variant and then, in Section~\ref{subsec:rectangular}, we come back to the model of eq.~\eqref{eq:matrix_glm}, and state our main results for it, as well as sketch the straightforward generalization of the proofs in the symmetric model to the non-symmetric case.

\vspace{-3mm}
\subsection{Motivations and related work}
\vspace{-1mm}

The problem we consider connects to multiple aspects of the existing literature; for instance, eq.~\eqref{eq:matrix_glm} generalizes classical linear and generalized linear models. In the statistics and computer science literature, this model is often referred to as matrix sensing \cite{recht2010guaranteed,fazel2002matrix,candes2012exact,bhojanapalli2016global,parker2016parametric,gunasekar2017implicit,li2018algorithmic,romanov2018near}. Signal recovery using minimum nuclear norm algorithms has been widely studied \cite{recht2010guaranteed,candes2012exact}, as well as recovery via gradient descent \cite{bhojanapalli2016global}, particularly in relation to overparametrization and the implicit regularization of gradient descent \cite{gunasekar2017implicit,li2018algorithmic}. 

\myskip
Comparing the bulk of literature on the problem to our setting and approach, we trade stronger distributional assumptions for sharper results. While we assume randomness in both the data and the signal, as stated in our setup, we derive results that explicitly include the leading-order constants in the high-dimensional limit. In contrast, most prior works impose much weaker distributional assumptions but, in exchange, can only derive results up to unspecified constants or, more often, up to logarithmic factors in dimensionality.
Another key distinction between our work and much of the literature is the scaling of the width/rank parameter $m$: while most studies consider the regime where $m \ll d$, we instead analyze the case where $m$ and $d$ are proportional.

\myskip
Only a few works study sharp optimal constants in the high-dimensional limit for the model of eq.~\eqref{eq:matrix_glm}. Some focus exclusively on the very low-rank case~\cite{schulke2016phase,parker2016parametric}, while others analyze the nuclear norm minimization algorithm without characterizing the Bayes-optimal performance \cite{donoho2013phase}.
To the best of our knowledge, the only two works that consider special cases of our setting, and characterize sharply the Bayes-optimal performance in the high-dimensional limit, are \cite{maillard2024bayes} and \cite{erba2024bilinear}.
However, both of these works rely on analytical non-rigorous methods from statistical physics, and do not provide formal proofs of their results. 
The main motivation of our work is to fill this gap by establishing mathematically these predictions, while extending them to a more general setup. 
We note that, although the authors of~\cite{maillard2024bayes} propose a tentative sketch of a rigorous analysis, we take here an approach different from the one proposed in this paper, which suggested using the high-dimensional analysis of an Approximate Message Passing algorithm, while we rather rely on an adaptive interpolation method~\cite{barbier2019adaptive}.

\myskip
 The applications studied in \cite{maillard2024bayes} and \cite{erba2024bilinear} serve as our primary motivation, and we will describe them in more detail in Section~\ref{subsec:applications}. The authors of \cite{maillard2024bayes} exploit a known relationship between the matrix sensing problem and a two-layer neural network with quadratic activation to establish sharp results on the optimal prediction error of such networks. Meanwhile, \cite{erba2024bilinear} introduce eq.~\eqref{eq:matrix_glm} as the simplest regression model for inputs that are either sequences of length $L$ or $d$-dimensional tokens. This formulation is inspired by the growing interest in sequence models such as transformers \cite{vaswani2017attention}.
Our findings also extend to matrix compressed sensing \cite{schulke2016phase,parker2016parametric}, which typically assumes the signal is a product of two matrices, or a low-rank matrix without any prior distribution.

\myskip
Methodologically, our proof techniques build on the analysis of high-dimensional generalized linear models from~\cite{barbier2019optimal}. Specifically, our model can be seen as a generalization of this model, where the signal is a high-rank rotationally invariant matrix, a non-separable prior far from the i.i.d. signals considered in~\cite{barbier2019optimal}.
To handle this complex signal structure, we leverage recent advances in the analysis of matrix denoising~\cite{maillard2022perturbative,troiani2022optimal,pourkamali2024matrix,pourkamali2024rectangular}.
Additionally, we build on fundamental Gaussian universality results~\cite{goldt2022gaussian,hu2022universality,montanari2022universality,maillard2023exact}, enabling us to tackle more general distributions $P_\Phi$ of $\Phi_\mu$.

\vspace{-3mm}
\subsection{Main contributions}
\vspace{-1mm}
Our contributions are twofold. First, we prove a low-dimensional formula for both the mutual information and the Minimal Mean-Square Error (MMSE) in matrix sensing problems with rotationally-invariant signals, in the high-dimensional limit. Our proof integrates recent advances in Gaussian universality \cite{montanari2022universality,maillard2023exact}, adaptive interpolation \cite{barbier2019adaptive,barbier2019optimal}, and matrix denoising \cite{troiani2022optimal,pourkamali2024matrix,pourkamali2024rectangular}.

\myskip
Second, we demonstrate how these formulas extend beyond standard signal processing applications \cite{gross2010quantum,candes2011tight}. In particular, we discuss in Section~\ref{subsec:applications} how our results allow us to prove heuristic results relevant to machine learning. Specifically, we address: (a) an open problem posed by \cite{cui2023bayes} on the learnability of "neural network-like" functions with $d^2$ parameters from $n = O(d^2)$ samples (at least for a subclass of functions) 
(b) the Bayes-optimal formula for two-layer neural networks with quadratic activation \cite{maillard2024bayes}; and (c) the information-theoretic limit of bilinear sequence regression problems \cite{erba2024bilinear}. Given its generality, we anticipate that our approach will find applications in a variety of other domains.

\vspace{-3mm}
\subsection{Notations}
\vspace{-1mm}
We denote $[d] \coloneqq \{1, \cdots, d\}$ the set of integers from $1$ to $d$,
and $\mcS_d$ the set of $d \times d$ real symmetric matrices.
For a function $\mathcal{V} : \bbR \to \bbR$, and $S \in \mcS_d$ with eigenvalues $(\lambda_i)_{i=1}^d$, we define $\mathcal{V}(S)$ as the matrix with the same eigenvectors as $S$,
and eigenvalues $(\mathcal{V}(\lambda_i))_{i=1}^d$. 
For $S \in \mcS_d$, we denote $\mu_S\coloneqq(1/d)\sum_{i=1}^d\delta_{\lambda_i}$ the empirical eigenvalue distribution of $S$. 
We denote $\Tr S\coloneqq\sum_{i=1}^d\lambda_i$ the trace of $S$ and $\tr S\coloneqq\frac{1}{d}\Tr S$ the normalized trace. $\|\cdot\|$ denotes the $\ell_2$ norm of vectors and the Frobenius norm of matrices and tensors.
$B_\op(M)$ denotes the set of matrices with eigenvalues (or singular values for rectangular matrices) in $[-M,M]$.
We refer to Appendix~\ref{subsec_app:definitions} for some classical definition of rotationally invariant ensembles, specifically the \emph{Gaussian Orthogonal Ensemble} $\GOE(d)$ and the \emph{Wishart Ensemble} $\mathcal{W}_{m,d}$. For $X, Y$ two random variables, we say $X \deq Y$ if $X, Y$ share the same distribution.
$\pto$ denotes limit in probability.
Finally, we generically denote constants as $C > 0$, whose value may vary from line to line.

\vspace{-3mm}
\section{Main results and applications}\label{sec:main_results}
\vspace{-2mm}
\subsection{Assumptions}\label{subsec:assumptions}
\vspace{-1mm}

We detail the different assumptions made on the model of eq.~\eqref{eq:matrix_glm} in order for our main results to apply.
In subsections~\ref{subsec:assumptions}--\ref{sec:MMSE} we consider the \emph{symmetric} variant of the model, where $L = d$, and $S^\star$, $\Phi_\mu$ are symmetric $d \times d$ matrices, and we assume that $n/d^2 \to \alpha > 0$.

\vspace{-2mm}
\begin{assumption}[Assumptions on $P_0$]
\label{ass:prior}
\noindent
We consider rotationally-invariant priors $P_0$,
more precisely either one of the following two cases holds:
\begin{itemize}
\item[$(i)$] 
$P_0$ has density:
\begin{equation*}
 P_0(S)\propto\exp(-d \, \Tr[\mathcal{V}(S)]),
\end{equation*}
for a real and $\mcC^2$ potential $\mathcal{V} : \mathcal{B} \to \bbR$, where $\mathcal{B}\subseteq\bbR$ is an interval. We assume that $\mathcal{V}''(s)\geq1/c$ for all $s\in\mathcal{B}$ and some constant $c> 0$.

\item[$(ii)$] Let $\kappa \in (0,1]$, and denote $m \coloneqq \lceil \kappa d \rceil$.
We assume that $P_0$ is supported on the space $\mcS_{d,m}^+$ of positive-semidefinite matrices $S \in \mcS_d$ with rank $m$ and distinct positive eigenvalues, with density:
\begin{equation*}
 P_0(S)\propto\exp(-d \, \Tr[\mathcal{V}(\Lambda)]),
\end{equation*}
where $\Lambda = \Diag(\{\lambda_i\}_{i=1}^m)$ denotes the non-zero eigenvalues of $S$.
Here $V : \mathcal{B} \to \bbR$ is a
a real and $\mcC^2$ potential, and $\mathcal{B}\subseteq\bbR^+$ is an interval.
We assume $\mathcal{V}''(s)+(1-\kappa)/s^2\geq 1/c$ for all $s\in\mathcal{B}$ and a constant $c > 0$.
\end{itemize} 
\end{assumption}
Notice that a mathematically complete description of Assumption \ref{ass:prior}(ii) would require defining the volume measure on $\mcS_{d,m}^+$: we refer to~\cite{uhlig1994singular} (Theorem~2) for a precise description of this point. 
Assumption~\ref{ass:prior} encompasses many natural random matrix distributions, including notably $\GOE(d)$ matrices. 
It also applies to a large class of Wishart matrices $\mcW_{m,d}$\footnote{When $m/d \to \kappa \in (0, \infty) \backslash \{1\}$. 
The case $\kappa = 1$ is not covered by our proof at the moment, due to a technical difficulty in obtaining an equivalent to the Gaussian Poincar\'e inequality in this setting, see Remark~\ref{remark:wishart_kappa_1}.},
once one slightly weakens the assumption on $\mathcal{V}$ to $\mathcal{V}''(s) \geq 1/c$ only inside a large bounded interval, an assumption under which our proofs generalize directly: we discuss such 
relaxations of Assumption~\ref{ass:prior}, as well as consequences of this assumption, in Appendix~\ref{sec_app:technical}. 
In particular, by Proposition~\ref{prop:properties_P0_sym}, under Assumption~\ref{ass:prior}, the empirical eigenvalue distribution $\mu_S$ almost surely weakly converges to a distribution $\mu_0$ with compact support.

\myskip 
\textbf{Assumptions on $P_\Phi$ --}
The main assumption we make on the data distribution $P_\Phi$ is referred to as a \emph{uniform one-dimensional Central Limit Theorem (CLT)}.
\begin{assumption}[Uniform one-dimensional CLT]
\label{assum:CLT}
\noindent
For any $M > 0$, the following holds.
For $\Phi\sim P_\Phi$, $G\sim\GOE(d)$, we have $\EE[\Phi]=\EE[G]=0$, $\EE[\Phi_{ij}\Phi_{kl}]=\EE[G_{ij}G_{kl}]=\delta_{ik}\delta_{jl}(1+\delta_{ijkl})/d$, and 
\begin{equation*}
\lim_{d\to\infty}\sup_{S\in B_\op(M)}|\EE_\Phi[\psi(\Tr[\Phi S])]-\EE_G[\psi(\Tr[GS])]|=0,
\end{equation*}
for any bounded Lipschitz function $\psi$. 
\end{assumption}
Assumption~\ref{assum:CLT} encompasses different distributions $P_\Phi$: two examples used here are matrices with i.i.d.\ entries (under some bounded moments conditions), as well as centered rank-one matrices $\Phi_\mu = (x_\mu x_\mu^\T - \Id_d)/\sqrt{d}$ with $x_\mu \sim \mcN(0, \Id_d)$ for which Assumption~\ref{assum:CLT} is proven in \cite[Lemma~4.8]{maillard2023exact}. These are related to applications of our results discussed in Section~\ref{subsec:applications}.
We leave a more exhaustive analysis of the distributions satisfying Assumption~\ref{assum:CLT} (akin to \cite[Section~3]{montanari2022universality} for vector distributions) to future work.

\myskip
We note that it is actually sufficient to assume Assumption~\ref{assum:CLT} for some $M>0$ large enough, depending only on the choice of $P_0$, given by Proposition~\ref{prop:properties_P0_sym}-$(ii)$ 
in Appendix~\ref{sec_app:technical}.
We will show that under Assumption \ref{assum:CLT}, one can effectively study an equivalent model to eq.~\eqref{eq:matrix_glm}, with $\{\Phi_\mu\}_{\mu=1}^n$ replaced by $\{G_\mu\}_{\mu=1}^n\iid \GOE(d)$, following insights of a long line of work~\cite{mei2022generalization,gerace2020generalisation,goldt2022gaussian,hu2022universality,montanari2022universality,maillard2023exact,dandi2024universality}.

\myskip 
\textbf{Activation function --}
Finally, we make the following generic assumption, akin to the one of~\cite{barbier2019optimal}, on the activation $\varphi$ in eq.~\eqref{eq:matrix_glm}.
\begin{assumption}[Activation function]
$\varphi(\cdot,a)$ is continuous almost everywhere, almost surely over $a\sim P_A$. Moreover, for $\Phi\sim P_\Phi$, $S^\star \sim P_0$ and $a \sim P_A$, there exists $\gamma>0$ such that the sequence $\{\EE[|\varphi(\Tr[\Phi S^\star],a)|^{2+\gamma}]\}_{d\geq1}$ is bounded.
\label{assum:phi_weak}
\end{assumption}

\vspace{-2mm}
\subsection{Limit of the mutual information}
\vspace{-1mm}

In this section, we shall sometimes use a nomenclature originating from \emph{information theory} and \emph{statistical physics}  (we refer to \cite{zdeborova2016statistical} for more details and other applications of this point of view to problems of learning and inference).   

\myskip
Recall that we consider the model of eq.~\eqref{eq:matrix_glm}, which is stated in an equivalent form in eq.~\eqref{eq:matrix_glm_Pout}.
We first define the \emph{partition function} $\mcZ(Y, \Phi)$, and the expected value of the log-partition function, called the \emph{free entropy} $f_d$:
\begin{equation}\label{eq:def_fentropy}
\mcZ(Y,\Phi)\coloneqq\EE_{S \sim P_0}\prod_{\mu=1}^nP_{\out}(Y_\mu|\Tr[\Phi_\mu S]), \quad {\rm and} \quad f_d\coloneqq\frac{1}{d^2}\EE_{Y,\Phi}\log\mcZ(Y,\Phi).
\end{equation}
Notice that the free entropy is directly related to the  \emph{mutual information} $I(S^\star,Y|\Phi)$ of information theory by: 
\begin{equation}\label{eq:mutual_inf_fentropy}
\lim_{d\to\infty}\frac{1}{d^2}I(S^\star,Y|\Phi) = -\lim_{d\to\infty}f_d + \alpha\EE_{V} \int_\bbR \rd \tY P_\out(\tY | \sqrt{2\rho} V) \log  P_\out(\tY | \sqrt{2\rho} V),
\end{equation}
where $V \sim \mcN(0,1)$ and $\rho \coloneqq \lim_{d \to \infty} (1/d) \EE_{S \sim P_0}\Tr[S^2]$ (see more on the definition of $\rho$ in Appendix~\ref{subsec_app:properties_prior}, eq.~\eqref{eq:def_rho}). We refer to \cite[Corollary 1]{barbier2019optimal} for a simple proof of eq.~\eqref{eq:mutual_inf_fentropy}, plugging in our CLT (Lemma \ref{lemma:CLT_new}). Notice that
the second term in the right-hand side of eq.~\eqref{eq:mutual_inf_fentropy} is a simple function of the one-dimensional noise channel $P_\out$.

\myskip 
Let us now introduce the main quantities which will characterize the limit of the mutual information, or free entropy.
We define the \emph{replica-symmetric free entropy} as:
\begin{equation}\label{eq:def_fRS}
f_{\RS}(q,r)\coloneqq\psi_{P_0}(r)+\alpha\Psi_\out(q)+\frac{r(\rho-q)+1}{4},
\end{equation}
where $q \in [0,\rho]$, $r \geq 0$, and
\begin{equation}
\label{eq:psi_0_out}
\begin{dcases}
\psi_{P_0}(r) &\coloneqq -\frac{1}{2}\Sigma(\mu_{1/r})-\frac{1}{4}\log r-\frac{3}{8}, \\
\Psi_{\out}(q) &\coloneqq \EE_{V,W,\tY} \log \int \mcD w P_\out(\tY|\sqrt{2q}V+w\sqrt{2(\rho-q)}).
\end{dcases}
\end{equation}
In eq.~\eqref{eq:psi_0_out}, we used the notation
 $\tY \sim P_{\out}(\cdot|\sqrt{2q}V+\sqrt{2(\rho-q)}W)$, $W,V\iid\mcN(0,1)$ and $\mcD w\coloneqq (e^{-w^2/2}/\sqrt{2\pi})\rd w$ denotes the standard Gaussian measure. Moreover, $\Sigma(\mu)\coloneqq\int\mu(\rd x)\mu(\rd y)\log|x-y|$, and $\mu_{t}\coloneqq\mu_0\boxplus\mu_{s.c.,\sqrt{t}}$ for $t\geq0$ is the free convolution of $\mu_0$ and a semicircular distribution of variance $t$: details on these classical definitions are given in Appendix~\ref{subsec_app:definitions}.

\myskip
Our first main theorem is a proof that the replica-symmetric free entropy $f_\RS$ is the limit of $f_d$ (or of the mutual information) as $d\to\infty$.
\begin{theorem}[Limit of the free entropy and mutual information]
\label{thm:fentropy_symmetric}
\noindent
Under Assumptions \ref{ass:prior}, \ref{assum:CLT} and \ref{assum:phi_weak}, we have (recall eq.~\eqref{eq:mutual_inf_fentropy}):
\begin{equation}
\lim_{d\to\infty}f_d=\sup_{q\in[0,\rho]}\inf_{r\geq0}f_{\RS}(q,r); \, \, \, \,  \, \, \, \, \lim_{d\to\infty}\frac{1}{d^2}I(S^\star,Y|\Phi)=-\sup_{q\in[0,\rho]}\inf_{r\geq0}f_{\RS}(q,r)+\alpha\Psi_\out(\rho).
\label{eq:RS_GLM}
\end{equation}
\end{theorem}
We now investigate the consequences of this result:
the strategy of the proof of Theorem~\ref{thm:fentropy_symmetric} is laid out in details in Section~\ref{sec:proof}.

\begin{remark}
\label{remark:psi_P0}
\noindent
\cite[Theorem 4.3]{maillard2024bayes} shows that
\begin{equation}\label{eq:psi_0_denoising}
\psi_{P_0}(r) =
\lim_{d\to\infty}\frac{1}{d^2}\EE_{Y'}\log\EE_{S\sim P_0} \left[e^{-\frac{d}{4}\Tr[(Y'-\sqrt{r}S)^2]}\right],
 \end{equation}
with $Y'\coloneqq\sqrt{r}S^\star+Z'$, $S^\star \sim P_0$, and $Z'\sim\GOE(d)$. Informally, $\psi_{P_0}(r)$ corresponds to the free entropy of the problem of \emph{denoising} the matrix $S^\star$ from the observation $Y'$.
In this regard, the replica-symmetric free entropy  eq.~\eqref{eq:def_fRS} is the natural generalization of the one defined in~\cite{barbier2019optimal}, which considers vector signals with i.i.d.\ prior, and for which the function $\psi_{P_0}$ is the free entropy of the (scalar) problem of denoising from the prior.
Similar generalizations for other types of structured prior (e.g.\ generative) have been developed in the literature~\cite{gabrie2019entropy,aubin2019precise,aubin2020spiked}.
\end{remark}

\vspace{-3mm}
\subsection{Limit of the MMSE}
\label{sec:MMSE}
\vspace{-1mm}

In this section, we state our main result, which sharply characterizes the Bayes-optimal mean squared error of the problem of eqs.~\eqref{eq:matrix_glm}, \eqref{eq:matrix_glm_Pout} in the high-dimensional limit.
\begin{theorem}[Limit of the MMSE]
\noindent
Suppose that assumptions \ref{ass:prior}, \ref{assum:CLT} and \ref{assum:phi_weak} hold, and that $P_{\out}$ is informative\footnotemark.
Let
\begin{equation*}
D^\star\coloneqq\{\alpha>0 \, | \,\inf_{r \geq 0} f_{\RS}(q,r)\ \text{has a unique maximizer} \ q^\star(\alpha) \in [0,\rho]\}.
\end{equation*}
Then $D^\star$ is equal to $(0, \infty)$ minus some countable set, and for any $\alpha\in D^\star$:
\begin{equation}\label{eq:limit_mmse}
\MMSE_d\coloneqq\frac{1}{d^2}\EE\left[||S^\star\otimes S^\star-\EE[S\otimes S|Y,\Phi]||^2\right]\overset{d\to \infty}{\longrightarrow}
\rho^2-q^\star(\alpha)^2,
\end{equation}
where $S \otimes S$ denotes the tensor product, and $\|\cdot\|$ the Frobenius norm.
\label{theo:overlap}
\end{theorem}
\footnotetext{i.e.\ there exists $y\in\bbR$ such that $P_{\out}(y|\cdot)$ is not equal almost everywhere to a constant. If $P_\out$ is not informative, estimation is impossible.}
The proof of Theorem~\ref{theo:overlap} leverages Theorem~\ref{thm:fentropy_symmetric}, using the connection of the free entropy to the mutual information, and the I-MMSE relation of information theory~\cite{guo2005mutual}.
The proof of the MMSE limit from the mutual information follows the approach developed in \cite[Theorem 2]{barbier2019optimal}. We provide a sketch of how this generalizes to the present context in Appendix \ref{sec_app:sketch_overlap}. See \cite{maillard2024bayes} for a detailed analysis of the MMSE expression.
\newpage
\begin{remark}
\label{remark:MMSE_PSD}
\noindent
Theorem \ref{theo:overlap} comes from the fact that
\begin{equation}\label{eq:convergence_overlap}
\left|\frac{1}{d}\Tr[S S^\star]\right| = \left|\frac{1}{d}\sum_{i,j=1}^ds_{ij}S_{ij}^\star\right|\underset{d\to \infty}{\pto}q^\star(\alpha),
\end{equation}
where $s$ is sampled from the posterior distribution $P(\cdot|Y,\Phi)$. 
This is proven in Appendix \ref{sec_app:sketch_overlap}, 
In general the absolute value here cannot be removed, and thus we consider the estimation error on the tensor $(S^\star \otimes S^\star)_{ijkl} \coloneqq S^\star_{ij} S^\star_{kl}$ in eq.~\eqref{eq:limit_mmse}: this is due to the potential presence of symmetries  (i.e., if $\varphi(z,a)=\varphi(-z,a)$), in which case it is only possible to estimate $S^\star$ up to a global sign.

However, notice that under Assumption \ref{ass:prior}(ii), the overlap $\frac{1}{d}\Tr[S S^\star]$ is always nonnegative, and thus we can remove the absolute value in eq.~\eqref{eq:convergence_overlap}, 
and obtain from it the convergence of the estimation error on $S^\star$:
\begin{equation*}
\MMSE_d\coloneqq\frac{1}{d}\EE\left[||S^\star-\EE[S|Y,\Phi]||^2\right]\overset{d\to\infty}{\longrightarrow}
\rho-q^\star(\alpha).
\end{equation*}
\end{remark}

\vspace{-3mm}
\subsection{Generalization to the rectangular model}\label{subsec:rectangular}
\vspace{-1mm}

In this section we come back to the original model of eqs.~\eqref{eq:matrix_glm}, \eqref{eq:matrix_glm_Pout} with rectangular matrices, and generalize our main results stated above in the symmetric model to this setting.
Recall that we consider the large system limit $L,d,n\to\infty$ with $n / (Ld)\to\alpha$ and $d/L \to\beta$, and that without loss of generality we assume $L \geq d$.
We can now adapt our main assumptions (see Section~\ref{subsec:assumptions}) to this context.
We assume that $P_0$ is a bi-rotationally-invariant distribution, as clarified in the following statement. 
\begin{assumption}[Assumptions on $P_0$]
\label{ass:prior_rec}
\noindent
Let $\kappa \in (0,1]$, and denote $m \coloneqq \lceil \kappa d \rceil$.
We assume that $P_0$ is supported on the space of matrices $S \in \bbR^{d \times L}$ with rank $m$ and distinct positive singular values, and has density of the form:
\begin{equation*}
 P_0(S)\propto\exp(-d \, \Tr[\mathcal{V}(\Sigma^2)]),
\end{equation*}
where $\Sigma = \Diag(\{\sigma_i\}_{i=1}^m)$ denotes the non-zero singular values of $S$, i.e.\ $\{\sigma_i^2\}$ are the non-zero eigenvalues of $S S^\T$.
Here $\mathcal{V} : \mathcal{B} \to \bbR$ is a
a real and $\mcC^2$ potential, and $\mathcal{B}\subseteq\bbR^+$ is an interval.
We assume $\frac{\rd^2}{\rd s^2}\mathcal{V}(s^2)+2(1-\kappa)/s^2\geq 1/c$ for all $s\in\mathcal{B}$, and a constant $c > 0$.
\end{assumption}
\begin{remark}\label{remark:rec_sym_prior}
\noindent
    Assumption~\ref{ass:prior_rec} is simply Assumption~\ref{ass:prior} applied to the matrix $S S^\T$.
\end{remark}
As in the symmetric model, we make the following uniform one-dimensional CLT assumption, the counterpart to Assumption~\ref{assum:CLT}.
\begin{assumption}[Uniform one-dimensional CLT]
\label{assum:CLT_rec}
\noindent
For $\Phi\sim P_\Phi$,
we have $\EE[\Phi] = 0$ and $\EE[\Phi_{ij} \Phi_{kl}] =\delta_{ik} \delta_{jl}$ (for all indices $i,j,k,l$).
Moreover, letting
$\{G_{ij}\}_{i,j=1}^{L,d}\sim\mcN(0,1/\sqrt{Ld})$, we assume that for any $M > 0$ and any bounded Lipschitz function $\psi$:
\begin{equation*}
\lim_{d\to\infty}\sup_{S\in B_\op(M)}|\EE_\Phi[\psi(\Tr[\Phi S])]-\EE_G[\psi(\Tr[GS])]|=0.
\end{equation*}
\end{assumption}
Finally, we assume that the activation $\varphi$ satisfies Assumption~\ref{assum:phi_weak}, as in the symmetric model.
Following the same terminology as in the symmetric model, we define the free entropy as
\begin{equation}\label{eq:def_frec}
f_d^{\rec}\coloneqq\frac{1}{dL}\EE_{Y,\Phi}\log
\EE_S\prod_{\mu=1}^nP_\out(Y_\mu|\Tr[\Phi_\mu S]),
\end{equation}
with $Y_\mu$ generated according to eq.~\eqref{eq:matrix_glm}.
The replica-symmetric free entropy functional is defined similarly to eq.~\eqref{eq:def_fRS}, with $\rho \coloneqq \lim_{d \to \infty} (1/\sqrt{dL}) \EE_{S \sim P_0}\Tr[S^2]$:
\begin{align}\label{eq:def_fRS_rec}
f_\RS^{\rec}(q,r)&\coloneqq\psi^{\rec}_{P_0}(r)+\alpha\Psi^\rec_\out(q)+\frac{r(\rho-q)}{2}, \\
\label{eq:def_Psiout_rec}
\Psi_{\out}^\rec(q)&\coloneqq\EE_{V,W,\tY}\int_\bbR \mcD wP_\out(\tY|\sqrt{q}V+\sqrt{\rho-q}w),
\end{align}
with $\tY\sim P_\out(\cdot|\sqrt{q}V+\sqrt{\rho-q}W)$, and $W,V\iid\mcN(0,1)$. 
Finally, $\psi^{\rec}_{P_0}(r)$ is defined as
\begin{equation}\label{eq:psi_P_rec}
\psi^{\rec}_{P_0}(r)\coloneqq-\frac{1}{2}\log r-(1-\beta)\int\log|x|\tmu_{1/r}(\rd x)-\beta\Sigma[\tmu_{1/r}]+C,
\end{equation}
where $C > 0$ is a constant that is independent of $r$. 
$\tmu_t$ (for $t \geq 0$) is a measure akin to the free convolution which appears in eq.~\eqref{eq:psi_0_out}: its precise definition is given in Appendix \ref{subsec_app:definitions}, see eq.~\eqref{eq:hmu_symmetrized}.

\myskip
We can now state the generalization of Theorems~\ref{thm:fentropy_symmetric} and \ref{theo:overlap} to the rectangular setting.
Their proofs are straightforward transpositions of the ones developed in the symmetric model: 
we sketch the main steps in Appendix~\ref{sec_app:generalization_rectangular}.
\begin{theorem}[Limit of the free entropy and mutual information]
\label{thm:limit_fentropy_rec}
\noindent
Under Assumptions~\ref{assum:phi_weak}, \ref{ass:prior_rec}, and \ref{assum:CLT_rec}, we have:
\begin{equation*}
\lim_{d\to\infty}f_d^{\rec}=\sup_{q\in[0,\rho]}\inf_{r\geq0}f^{\rec}_\RS(q,r); \, \, \, \,  \, \, \, \, \lim_{d\to\infty}\frac{1}{d^2}I(S^\star,Y|\Phi)=-\sup_{q\in[0,\rho]}\inf_{r\geq0}f^\rec_{\RS}(q,r)+\alpha\Psi^\rec_\out(\rho).
\end{equation*}
\end{theorem}
Again, via the I-MMSE theorem, Theorem~\ref{thm:limit_fentropy_rec} leads to a precise characterization of the Bayes-optimal mean squared error.
\begin{theorem}[Limit of the MMSE]\label{thm:limit_mmse_rec}
\noindent
Suppose that Assumptions~\ref{assum:phi_weak}, \ref{ass:prior_rec} and \ref{assum:CLT_rec} hold, and that $P_\out$ is informative.
Let
\begin{equation*}
D^\star\coloneqq\{\alpha>0 \, | \, \inf_{r \geq 0} f_\RS(q,r)\ \text{has a unique maximizer} \ q^\star(\alpha) \in [0,\rho]\}.
\end{equation*}
Then $D^\star$ is equal to $(0, \infty)$ minus some countable set, and for any $\alpha\in D^\star$:
\begin{equation*}
\MMSE_d\coloneqq\frac{1}{dL}\EE\left[||S^\star\otimes S^\star-\EE[S^\star\otimes S^\star|Y,W]||^2\right]\overset{d \to \infty}{\longrightarrow}
\rho^2-q^\star(\alpha)^2.
\end{equation*}
\end{theorem}

\vspace{-3mm}
\subsection{Applications}\label{subsec:applications}
\vspace{-1mm}

The general results of Theorems~\ref{thm:fentropy_symmetric} and \ref{theo:overlap} (as well as their counterparts in the rectangular model) can be applied to a variety of settings. We focus now on two such applications, where our results mathematically establish predictions obtained by non-rigorous analytical methods. 

\myskip
\textbf{Extensive-width neural networks with quadratic activation --}
\cite{maillard2024bayes} conjectures the limiting free entropy and Bayes-optimal MMSE in a two-layer neural network with a quadratic activation. Specifically, we consider a dataset of $n$ samples $\mcD=\{Y_\mu,x_\mu\}_{\mu=1}^n$ where the input data is normal Gaussian of dimension $d$: $(x_\mu)_{\mu=1}^n \iid \mathcal{N}(0,\Id_d)$. 
We let $m = \kappa d$, with $\kappa > 0$ and $\kappa\neq 1$. We then draw i.i.d.\ $d$-dimensional teacher-weight vectors $(w_k^\star )_{k=1}^m \iid \mathcal{N}(0,\Id_d)$, 
and denote $W^\star \in \bbR^{m \times d}$ the matrix with rows $(w_k^\star)$, and we draw the noise $(z_{\mu,k})_{\mu,k=1}^{n,m} \iid \mathcal{N}(0,1)$. 
We fix a ``noise strength'' $\Delta > 0$.
Finally, the output labels $(Y_\mu)_{\mu=1}^n$ are obtained by a one-hidden layer teacher network with $m$ hidden units and quadratic activation:
\begin{equation} 
Y_\mu = f_{W^\star}(x_\mu) \coloneqq \frac{1}{m}\sum_{k=1}^m \left[\frac{1}{\sqrt{d}}(w_k^\star )^\T x_\mu + \sqrt{\Delta} z_{\mu,k}\right]^2+\sqrt{\Delta_0}\zeta_\mu.
\label{eq:teacher_def}
\end{equation}
We introduced in eq.~\eqref{eq:teacher_def} a post-activation noise $\{\zeta_{\mu}\}_{\mu=1}^n\iid\mcN(0,1)$ with $\Delta_0>0$: this allows to simplify some parts of the proof, but we believe that this assumption can be relaxed. 
Following~\cite{maillard2024bayes},
we define the generalization error of the Bayes optimal estimator to be
\begin{align} \label{eq:mmse_def_y}
    \MMSE_{d} &\coloneqq \frac{m}{2} \EE_{W^\star , \mcD} \EE_{Y_{\rm test},x_{\rm test}}\left[\left(Y_{\rm test} - \hat{Y}_\mcD^{\rm BO}(x_{\rm test} )\right)^2\right] - \Delta(2+\Delta),
\end{align}
where we rescaled the error according to~\cite{maillard2024bayes}, and where
\begin{equation}\label{eq:haty_BO}
    \hat{Y}_{\mcD}^{\rm BO}(x_{\rm test}) \coloneqq \EE\left[Y_{\mathrm{test}} \middle | x_{\mathrm{test}}, \mcD \right] = \int \, \EE_{z}[f_{W}(x_{\rm test})]\, \bbP(W | \mcD)\,{\rm d}W\, ,
\end{equation}
with $\bbP(W | \mcD)$ the posterior distribution of $W$ given the dataset $\mcD$.

\myskip
The following theorem rigorously establishes the predictions for the limiting Bayes-optimal MMSE that were obtained in~\cite{maillard2024bayes}\footnote{Modulo the technical restriction $m/d \to \kappa \in (0,\infty)\backslash\{1\}$, see the discussion around Assumption~\ref{ass:prior}.}.
\begin{theorem}[Limit of the generalization MMSE]
\begin{equation*}
\lim_{d\to\infty}\MMSE_{d}=
\kappa(\rho-q^\star(\alpha)),
\end{equation*}
where $q^\star(\alpha)$ is the unique maximizer of the replica free entropy defined in Theorem \ref{theo:overlap} (which requires $\alpha\in D^\star$), with the channel $\varphi(x)=x$, the noise level $\tDelta \coloneqq \frac{2\Delta(2+\Delta)}{\kappa}+\Delta_0$, and the Wishart prior $\mcW_{m,d}$.
\label{theo:2layerNN}
\end{theorem}
We sketch the proof strategies of Theorem \ref{theo:2layerNN} in Section \ref{subsec:2layerNN}, with its details given in Appendix \ref{sec_app:reduction_2nn}.

\myskip
We note that our results are also related to a conjecture of~\cite{cui2023bayes}, which suggested that a function defined by a two-layer neural network with a quadratic (i.e.\ $\mcO(d^2)$, where $d$ is the input dimension) number of parameters can be learned from $n=\mcO(d^2)$ samples. For quadratic polynomial functions $f(x) = f_0 + \langle u^\star, x \rangle + x^\T S^\star x + \text{noise}$, we answer by the affirmative. Indeed, a simple plug-in estimator can learn the scalar $f_0$ and the $d-$dimensional vector $u^\star$ as soon as $n = \omega(d)$ (for instance
$\hat u \coloneqq (d/n) \sum_{\mu=1}^n y_\mu x_\mu$ converges to $u^\star_i$ with vanishing error, see \cite[Appendix~A.1]{JMLR:v25:23-1543}),
so that the question simply reduces to the case $f(x) = x^\T S^\star x$, for which we provide the asymptotic MMSE, as discussed above.

\myskip 
\textbf{Bilinear sequence regression --}
Another model of interest is Bilinear Sequence Regression (BSR), a matrix sensing model studied recently in~\cite{erba2024bilinear} as a toy model of learning from sequences of tokens.
The observations in this model are generated as 
\vspace{-1mm}
\begin{equation}\label{eq:bsr}
Y_\mu\sim P_{\out}\left(\cdot\Big|\frac{1}{r\sqrt[4]{dL}}\sum_{\gamma=1}^r\sum_{a,i=1}^{L,d}X^\mu_{ia}U_{i\gamma}^\star V_{\gamma a}^\star\right),
\end{equation}
where $\{X^\mu_{ia}\}_{\mu,i,a=1}^{n,L,d}\iid\mcN(0,1)$ can be interpreted as $n$ samples of sequences of length $L$ of token embedded in dimension $d$. 
The weights $U^\star \in\bbR^{d\times r}$  and $V^\star \in\bbR^{r\times L}$ are drawn i.i.d.\ from $\mcN(0,1)$.
Denoting $S^\star\coloneqq(1/r)U^\star V^\star$, its distribution satisfies Assumption~\ref{ass:prior_rec}. Since the input samples $X^\mu_{ia}$ are assumed to be standard Gaussian, the model of eq.~\eqref{eq:bsr} falls under our general analysis. More specifically, the results presented in Section~\ref{subsec:rectangular} rigorously establish the predictions for the Bayes-optimal MMSE described in~\cite{erba2024bilinear}.

\vspace{-3mm}
\section{Sketch of proofs of the main results}\label{sec:proof}
The proof of Theorem~\ref{thm:fentropy_symmetric} contains three crucial ingredients.

\setlength{\leftmargin}{5cm}
\begin{itemize}
    \item[$(i)$] 
    First, we show that, under Assumption~\ref{assum:CLT}, one can replace the data matrices $\Phi_\mu\iid P_\Phi $ with $G_\mu\iid\GOE(d)$, without changing the $d \to \infty$ limit of the free entropy $f_d$. This \emph{universality} argument is inspired by recent results on empirical risk minimization (see e.g.~\cite{goldt2022gaussian,hu2022universality,montanari2022universality,gerace2024gaussian} and references therein) and on the ellipsoid fitting problem~\cite{maillard2023exact}. 
    \item[$(ii)$] The universality property above implies that we can consider what is essentially a generalized linear model with \emph{Gaussian data} and a rotationally-invariant prior supported on symmetric matrices. We extend results on generalized linear models which were developed for i.i.d.\ priors~\cite{barbier2019optimal} to thi setting, by adapting an interpolation argument.
    \item[$(iii)$] 
    Finally, the analysis of point $(ii)$ yields that the limiting free entropy is expressed as a function of the free entropy of a \emph{matrix denoising} problem. We then use recent results on denoising of symmetric~\cite{bun2016rotational,maillard2022perturbative,pourkamali2024matrix,semerjian2024matrix} and non-symmetric~\cite{troiani2022optimal,pourkamali2024rectangular} matrices to finish the proof of Theorem~\ref{thm:fentropy_symmetric}.
\end{itemize}
\vspace{-1mm}
To deduce Theorem~\ref{theo:overlap} from Theorem~\ref{thm:fentropy_symmetric}, we establish the limiting free entropy in a so-called ``spiked tensor'' model, with rotationally invariant signals. We then leverage this result by adding this model as a small side-information to our setting, and use then the I-MMSE theorem to establish the limit of the MMSE. 
This follows the approach of \cite{barbier2019optimal}, and is sketched in Section~\ref{subsec:limit_overlap}, with details defered to Appendix~\ref{sec_app:sketch_overlap}. To deduce Theorem \ref{theo:2layerNN} from Theorem \ref{theo:overlap}, we need to interpolate between the two-layer neural networks and a matrix sensing model: the argument is sketched in Section \ref{subsec:2layerNN}, with details postponed to Appendix \ref{sec_app:reduction_2nn}.

\vspace{-3mm}
\subsection{\texorpdfstring{Universality and reduction to Gaussian sensing matrices -- ingredient $(i)$}{}}\label{subsec:universality}
\vspace{-1mm}

In this part, we use the representation of the output channel via an explicit random function $\varphi$, see eq.~\eqref{eq:matrix_glm}.
Following a strategy introduced in~\cite{montanari2022universality},
we define the following ``interpolated'' model
\begin{equation}\label{eq:interp_model_univerality}
Y_\mu=\varphi(\Tr[U_\mu(t) S^\star],A_\mu)+\sqrt{\Delta}Z_\mu,
\end{equation}
for $t \in [0, \pi/2]$,
where $U_\mu(t)\coloneqq\cos(t)\Phi_\mu+\sin(t)G_\mu$, and $\{G_\mu\}_{\mu=1}^n\iid\GOE(d)$. Note that $t=0$ recovers the original model.
We naturally define the interpolated free entropy as:
\begin{equation}
F_d(U(t))\coloneqq\frac{1}{d^2}\log \EE_{S \sim P_0} \left[e^{-\sum_{\mu=1}^n\ell(\Tr[U_\mu(t)S],\Tr[U_\mu(t) S^\star],Z_\mu)}\right]-\frac{1}{d^2}\sum_{\mu=1}^n\frac{Z_\mu^2}{2\Delta},
\label{eq:free_entropy_universality}
\end{equation}
where
\begin{equation}\label{eq:def_ell}
\ell(x, X,Z)\coloneqq -\log\int dP_A(a)\exp\left(-\frac{(\varphi(x, a)-\varphi(X, A))^2-2Z(\varphi(x, a)-\varphi(X, A))}{2\Delta}\right).
\end{equation}
Note that $F_d(U)$ implicitly relies on the realization of the random variables $Z,A,S^\star$, and that $f_d=\EE F_d(U(0))$, where $f_d$ is defined in eq.~\eqref{eq:def_fentropy}.
Our setting essentially corresponds to a ``planted'' (i.e.\ with an underlying signal $S^\star$) variant of~\cite{maillard2023exact}: on the other hand, \cite{montanari2022universality} considers a planted problem as we do, but under stronger assumptions on the prior distribution (see a discussion of this point in~\cite[Section~2.2]{maillard2023exact}). 

\myskip
We first state our universality lemma under two assumptions that will be weakened in the end.
\begin{assumption}
\noindent
Under Assumption \ref{ass:prior}, we further assume that $\mcB \subseteq [-M, M]$ for some $M > 0$.
\label{assum:S_strong}
\end{assumption}
\begin{assumption}
\noindent
$\varphi$ is a bounded function with bounded first and second derivatives w.r.t. its first argument. 
\label{assum:phi_strong}
\end{assumption}
Essentially, these two assumptions can be weakened to Assumption~\ref{ass:prior} and \ref{assum:phi_weak}, as we can build a bounded $\hat{\varphi}$ (with good regularity properties) such that the output difference $\EE[(\varphi(\Tr[\Phi S^\star],a)-\hat{\varphi}(\Tr[\Phi S^\star],a))^2]$ is arbitrarily small, and show that this implies a bound on the difference of the free entropies. To deal with the eigenvalues of $S^\star$, we interpolate between the original potential $\mathcal{V}$ in Assumption~\ref{ass:prior} and another potential $\tilde{\mathcal{V}}$ that is infinity outside a bounded interval. See Appendix \ref{sec_app:relax} for details.
\begin{lemma}[Universality]
\label{lemma:universality}
\noindent
Under Assumptions \ref{assum:CLT}, \ref{assum:S_strong} and \ref{assum:phi_strong}, we have
\begin{equation}
\lim_{d\to\infty}|\EE[\psi(F_d(\Phi))]-\EE[\psi(F_d(G))]|=0
\label{eq:lemma_universality}
\end{equation}
for any bounded function $\psi$ with bounded Lipschitz derivative, where the free entropy is defined in eq.~\eqref{eq:free_entropy_universality}.
\end{lemma}
Let us now briefly sketch the proof of Lemma~\ref{lemma:universality}.
Clearly, it is sufficient to prove
\begin{equation}\label{eq:sufficient_universality}
\lim_{d\to\infty}\sup_{\{A_\mu\}_{\mu=1}^n}\sup_{S^\star\in B_{\op}(M)}|\EE_{\Phi,Z}[\psi(F_d(\Phi))]-\EE_{G,Z}[\psi(F_d(G))]|=0.
\end{equation}
Using the interpolation path described above and the triangular inequality, we can write:
\begin{equation}\label{eq:interpolation_fentropy}
|\EE_{\Phi,Z}[\psi(F_d(\Phi))]-\EE_{G,Z}[\psi(F_d(G))]|\leq\int_0^{\pi/2}\left|\EE_{G,\Phi,Z}\frac{\partial\psi(F_d(U(t)))}{\partial t}\right|\rd t,
\end{equation}
and we will prove that the right-hand side of eq.~\eqref{eq:interpolation_fentropy} goes to zero uniformly in $S^\star \in B_\op(M)$ and $\{A_\mu\}_{\mu=1}^n$. This then essentially follows from an adaptation of the arguments of~\cite{maillard2023exact} on the problem of ellipsoid fitting, that we detail in Appendix~\ref{sec_app:universality}.

\vspace{-3mm}
\subsection{\texorpdfstring{Matrix sensing with Gaussian data -- $(ii)$ and $(iii)$}{}}
\vspace{-1mm}

We sketch here the end of the proof of Theorem~\ref{thm:fentropy_symmetric}.
Thanks to Lemma~\ref{lemma:universality}, we focus on the model
\begin{equation}
\label{eq:msensing_gaussian}
Y_\mu\sim P_{\out}(\cdot|\Tr[G_\mu S^\star]),
\end{equation}
where $(G_\mu)_{\mu=1}^n \iid \GOE(d)$.
We use the adaptive interpolation method \cite{barbier2019optimal}, and introduce:
\begin{equation}\label{eq:interpolated_model}
Y_t\sim P_\out(\cdot|J_t);  \quad \quad \quad   
Y_t'=\sqrt{R_1(t)}S^\star+Z',
\end{equation}
with $Z' \sim \GOE(d)$, and where we defined 
\begin{equation*}
J_{t,\mu}\coloneqq\sqrt{1-t}\Tr[G_\mu S^\star]+\sqrt{2R_2(t)}V_\mu+\sqrt{2\rho t-2R_2(t)+2s_d}U_\mu^\star,
\end{equation*}
with $\{V_\mu,U_\mu^\star\}_{\mu=1}^n\iid\mcN(0,1)$. 
The goal is now to recover \emph{both} $S^\star$ and $U^\star$ from the observations of $(Y_t, Y'_t)$.
The choice of the functions $R_1(t), R_2(t)$ will be precised later on, we simply require at the moment that $R_1(t) \geq 0$, $R_2(t) \in [0, \rho]$, and that $R_1(0) = R_2(0) = 0$.

\myskip
Notice that eq.~\eqref{eq:interpolated_model} reduces to \eqref{eq:msensing_gaussian} for $t=0$, while for $t = 1$ its free entropy reduces to the replica-symmetric prediction of eq.~\eqref{eq:def_fRS}, with $q = R_2(1)$ and $r = R_1(1)$.
This last point is a consequence of the analysis of the \emph{matrix denoising} problem, which arises from the observation of $Y'_t$ in eq.~\eqref{eq:interpolated_model}.

\myskip
Our goal reduces now to chose the functions $R_1(t), R_2(t)$, so as to control the derivative of the free entropy of the interpolated model with respect to $t$.
Following~\cite{barbier2019optimal}, a so-called \emph{adaptive} choice  allows to sharply control this derivative, see Lemmas~\ref{lemma:lowerbound} and \ref{lemma:upperbound} in Appendix~\ref{sec_app:proof_GLM}.
The main difference with our setting is that the signal $S^\star$ is a rotationally-invariant matrix (see Assumption~\ref{ass:prior}), while~\cite{barbier2019optimal} focuses on simpler i.i.d.\ priors.
Beyond the analysis of the matrix denoising problem, the other key result needed to extend their approach is
a Poincar\'e inequality for rotationally-invariant distributions satisfying Assumption~\ref{ass:prior}, see Lemma~\ref{lemma:poincare} and~\cite{chafai2020poincare}.
It essentially reads as:
\begin{equation}\label{eq:poincare_informal}
\Var(g(S^\star))\leq \frac{c}{\kappa d}\EE\left[\sum_{i=1}^d\left(\frac{\partial g}{\partial \Lambda^\star_{i}}\right)^2\right],
\end{equation}
where $\{\Lambda^\star_{i}\}_{i=1}^d$ are the eigenvalues of $S^\star$, for any ``regular enough'' function $g$.
Eq.~\eqref{eq:poincare_informal} allows us to prove concentration properties for the free entropy of the interpolated model of eq.~\eqref{eq:interpolated_model}, which in turns makes possible the generalization of the adaptive interpolation approach of~\cite{barbier2019optimal}.
A detailed treatment of the arguments above, which form the proof of Theorem~\ref{thm:fentropy_symmetric}, is given in Appendix~\ref{sec_app:proof_GLM}.

\vspace{-3mm}
\subsection{Limit of the MMSE: sketch of proof of Theorem~\ref{theo:overlap}}
\label{subsec:limit_overlap}
\vspace{-1mm}

Using the I-MMSE theorem~\cite{guo2005mutual}, the MMSE can be related to the free entropy by considering the following observation model, in which we added a ``side information'':
\begin{equation}\label{eq:model_side_info_tensor}
Y_\mu\sim P_{\out}(\cdot|\Tr[G_\mu S^\star]), \, \, (\mu \in [n]); \quad \quad \quad
Y'=\sqrt{\frac{\lambda}{d^{2p-1}}}(S^\star)^{\otimes2p}+Z',
\end{equation}
where $Z'\in \mcS_d^{\otimes2p}\iid\mcN(0,1)$ is a tensor, and we chose an integer $p \geq 1$ and a signal-to-noise ratio $\lambda \geq 0$. The additional side information comes from a so-called \emph{spiked tensor} model.
In Appendix~\ref{sec_app:spike} we detail the proof of the limiting free entropy for the model of eq.~\eqref{eq:model_side_info_tensor} (under Assumptions~\ref{ass:prior}, \ref{assum:CLT} and \ref{assum:phi_weak}), which leverages as well the adaptive interpolation method of~\cite{barbier2019adaptive}. Finally, we describe in Appendix~\ref{sec_app:sketch_overlap} how this can be used to end the proof of Theorem~\ref{theo:overlap}.

\subsection{Reduction of the two-layer network: sketch of proof of Theorem \ref{theo:2layerNN}}
\label{subsec:2layerNN}

We sketch here the proof, with details given in Appendix~\ref{sec_app:reduction_2nn}.
To prove Theorem \ref{theo:2layerNN}, we will show that we can replace the model of eq.~\eqref{eq:teacher_def} by a matrix sensing problem in terms of 
$S^\star \coloneqq (1/m) \sum_{k=1}^m w_k^\star (w_k^\star)^\T$.
This establishes rigorously the heuristic argument sketched in Appendix~C.3 of \cite{maillard2024bayes}.

\myskip
Specifically, we will prove that both models have the same free entropy, and then deduce the equivalence of MMSE using the I-MMSE theorem (by adding a side information channel, whose signal-to-noise ratio eventually goes to $0$). We first rewrite the model of eq.~\eqref{eq:teacher_def} as
\begin{align}\label{eq:vi}
    v_\mu &= \Tr[\Phi_\mu S^\star] + \sqrt{d}(\tr S^\star - 1) + \sqrt{d} \Delta \left(\frac{\|z_\mu\|^2}{m} - 1\right) 
    + \frac{2\sqrt{\Delta}}{m} \sum_{k=1}^m z_{\mu,k} \left(\frac{x_\mu^\T w_k^\star}{\sqrt{d}}\right)+\sqrt{\Delta_0}\zeta_\mu,
\end{align}
where $v_\mu \coloneqq \sqrt{d}(Y_\mu - \Delta - 1)$ and 
$\Phi_\mu \coloneqq (x_\mu x_\mu^\T - \Id_d)/\sqrt{d}$. For notational simplicity, we define $\tr(\cdot) \coloneqq (1/d)\Tr(\cdot)$.
We use then an interpolation argument. For $t \in [t_0,1]$, we define 
\begin{align*}
    \nonumber
    v_{t,\mu} &\coloneqq \Tr[\Phi_\mu S^\star] + (1-t)\sqrt{d} [\tr \, S^\star - 1] 
    + \sqrt{d(1-t)} \Delta \left(\frac{\|z_\mu\|^2}{m} - 1\right)  \\ 
    &+ \frac{2\sqrt{\Delta d(1-t)}}{m} \sum_{k=1}^m z_{\mu,k} \left(\frac{x_\mu^\T w_k^\star}{\sqrt{d}}\right)+ \sqrt{\tDelta t} \zeta_\mu,
\end{align*}
where $\tDelta \coloneqq 2\Delta(2+\Delta)/\kappa+\Delta_0$ and $t_0\coloneqq \Delta_0 / \tDelta <1$. Notice that, for $t = t_0$, $v_{t_0,\mu} = v_\mu$ is given by eq.~\eqref{eq:vi}, 
while for $t = 1$, $v_{1,\mu} = \Tr[\Phi_\mu S^\star] + \sqrt{\tDelta} \zeta_\mu$ is a ``matrix GLM'' problem, with Gaussian additive noise.
The technical introduction of the post-activation noise $\Delta_0 > 0$ allows to have $t$ lower bounded: while we believe this not to be necessary, it allows for simplifications in part of our proof.
We denote $v_{t,\mu} \sim P_\out^{(t)}(\cdot | W^\star,x_\mu)$ the distribution of these labels.

\myskip
Moreover, We introduce a small ``side-information'' channel to the observations, 
similarly e.g.\ to~\cite{barbier2019optimal,maillard2020phase}. 
More precisely, for $\Lambda \geq 0$ we add the observation of $Y' \in \bbR$, which is built as:
\begin{align}\label{eq:side_info_channel}
    Y' = \sqrt{\Lambda} S^\star_{12} + \frac{\xi}{\sqrt{d}}, 
\end{align}
where $\xi \sim \mcN(0,1)$.
Adding this side information channel, we will then be able to leverage the I-MMSE theorem \cite{guo2005mutual}, before eventually taking $\Lambda \downarrow 0$.
We define the interpolated free entropy as
\begin{align}\label{eq:def_fnd}
    f_{d}(t,\Lambda) &\coloneqq \EE_{W^*,Y',x} \frac{1}{d^2} \log \int_{\bbR^{m \times d}} \mcD W \left[\prod_{\mu=1}^n P_\out^{(t)}(v_{\mu,t} | W,x_\mu) \right] \, e^{-\frac{d}{4} (Y' - \sqrt{\Lambda} S_{12})^2}.
\end{align}
We denote $\langle \cdot \rangle_{t,\Lambda}$ the average with respect to the Gibbs-like measure with weight proportional to the integrand of eq.~\eqref{eq:def_fnd}. By the I-MMSE theorem and the Nishimori identity, we prove that, for any $t \in [t_0,1]$:
\begin{align}
\label{eq:partial_lambda_f}
\left|\frac{\partial}{\partial\Lambda} f_d(t,0) + \frac{1}{2} \EE \, \tr[(S^\star - \langle S \rangle_{t,0})^2] \right|\leq \frac{C(\kappa)}{d},
\end{align}
for some $C(\kappa)>0$. The MSE on $S^\star$ is connected to the generalization MMSE via 
\begin{equation}
\label{eq:MMSEy_MMSES}
   \left|\MMSE_{d} - \kappa \EE\tr[(S^\star - \langle S \rangle_{t_0,0})^2 \right| \leq \frac{C(\kappa)}{n}.
\end{equation}
for some $C(\kappa)>0$: eq.~\eqref{eq:MMSEy_MMSES} is proven in \cite[Lemma D.1]{maillard2024bayes}.
Moreover, by Theorem \ref{theo:overlap} and Remark \ref{remark:MMSE_PSD} we have 
\begin{equation}
\EE \, \tr[(S^\star - \langle S \rangle_{1,0})^2] \underset{d\to \infty}{\pto}
\rho-q^\star(\alpha).
\label{eq:t=1}
\end{equation}
It remains then to prove the following lemma.
\begin{lemma}
\begin{equation}
\lim_{d\to\infty}\left|\frac{\partial}{\partial\Lambda} f_d(1,0) - \frac{\partial}{\partial\Lambda} f_d(t_0,0) \right|=0.
\label{eq:diff_t0_1}
\end{equation}
\label{lemma:diff_t0_1}
\end{lemma}
Combining eqs.~\eqref{eq:partial_lambda_f}, \eqref{eq:MMSEy_MMSES}, \eqref{eq:t=1} and \eqref{eq:diff_t0_1}, we can finish the proof of Theorem \ref{theo:2layerNN}.

\myskip
The proofs of eq.~\eqref{eq:partial_lambda_f} and Lemma \ref{lemma:diff_t0_1} are given in Appendix \ref{sec_app:reduction_2nn}. For Lemma \ref{lemma:diff_t0_1}, we will utilize the interpolation model, and show that $\frac{\partial}{\partial t}f_d(t,\Lambda)$ goes to zero uniformly in $[t_0,1]$ and $\Lambda \geq 0$, implying that $|f_d(1,\Lambda)-f_d(t_0,\Lambda)|$ goes to zero uniformly in $\Lambda$. Using the convexity of the free entropy as a function of $\Lambda$, we then deduce that $\left|\frac{\partial}{\partial\Lambda} f_d(1,0) - \frac{\partial}{\partial\Lambda} f_d(t_0,0) \right|$ also goes to zero.
\vspace{-1mm}

\section*{Acknowledgements}
We acknowledge funding from the Swiss National Science Foundation grants SNFS SMArtNet (grant number 212049), OperaGOST (grant number 200021 200390) and DSGIANGO.

\newpage
\bibliographystyle{alpha}
\bibliography{main}
\newpage 

\appendix
\section{Classical definitions and remarks}\label{sec_app:technical}
\subsection{Classical definitions}
\label{subsec_app:definitions}
The following two standard random matrix distributions over the $\mathcal{S}_d$ are of particular interest:
\begin{itemize}
\item We say that a random matrix $Y \in \mcS_d$ is generated from the \emph{Gaussian Orthogonal Ensemble} $\GOE(d)$ if $Y_{ij} \iid \mcN(0, (1+\delta_{ij})/d)$ for $1\leq i \leq j\leq d$.
It clearly satisfies Assumption~\ref{ass:prior}, with $\mathcal{V}(x) = x^2/4$.
\item For an integer $m \geq 1$, we say that a random matrix $Y \in \mcS_d$ is generated from the \emph{Wishart Ensemble} $\mcW_{m,d}$ if $Y=(1/m)W^\T W$, with $W\in\bbR^{m\times d}$ and $W_{ij}\iid\mcN(0,1)$ for $i\in[m],j\in[d]$. 
Notice that if $m, d \to \infty$ with $m/d \to \kappa \in (0, \infty)\backslash\{1\}$, then $Y$ satisfies Assumption~\ref{ass:prior} (see e.g.~\cite{anderson2010introduction,uhlig1994singular}):
\begin{itemize}
    \item If $\kappa > 1$, is satisfies case $(i)$, 
    with\footnote{The probability density actually satisfies $P_0(S) \propto e^{-d \Tr[V_d(S)]}$, with $V_d(x) \coloneqq V(x) + (1/2d) \log x$. However, 
    this second term only changes the constant $c$ in the Gaussian Poincar\'e inequality (Lemma~\ref{lemma:gaussian-poincare}) by an additive $\smallO_d(1)$, and thus does not influence our results.}
    \begin{equation}\label{eq:V_Wishart}
     \mathcal{V}(x) = \begin{cases}
         \frac{\kappa}{2} x - \frac{\kappa-1}{2} \log x & \textrm{ if } x > 0, \\
         + \infty & \textrm{ otherwise.}
     \end{cases}
    \end{equation}
    \item If $\kappa < 1$, is satisfies case $(ii)$~\cite[Theorem~6]{uhlig1994singular}, with $\mathcal{V}(x)$ still given by eq.~\eqref{eq:V_Wishart}.
\end{itemize}
\end{itemize} 
Notice that in the Wishart setting we do not have $\mathcal{V}''(x)  = (\kappa-1)/(2x^2) \geq 1/c$ for all $x \in \bbR$. However for all $M > 0$ we have $\mathcal{V}''(x) \geq 1/c$ for $x \in (0, M)$ and some $c = c(M) > 0$, which we show to be sufficient for our analysis, see Remark~\ref{remark:weakening_ass_prior}.

\myskip
\begin{remark}[The case $\kappa  = 1$]\label{remark:wishart_kappa_1}
\noindent
We notice that the Wishart distribution with $\kappa = 1$ does not satisfy Assumption~\ref{ass:prior}, since then $V$ is not strongly convex: $\mathcal{V}(x) = x/2$ for $x > 0$.
This prevents us from obtaining a Poincar\'e inequality for this prior (see Lemma~\ref{lemma:gaussian-poincare} and the discussion around this result), which is a key element of our analysis.
Since our analysis applies for any $\kappa \neq 1$, we expect this condition to be an artifact of our proof techniques, and we leave a lifting of this restriction for future work.
\end{remark}

\myskip
The limiting eigenvalue distributions of the two ensembles above are well-known~\cite{anderson2010introduction}.
\begin{itemize}
\item The semi-circle law
\begin{equation}\label{eq:def_musci}
    \mu_\sci(x)\coloneqq\frac{\sqrt{4-x^2}}{2\pi}\indi\{|x|\leq2\}
\end{equation}
is the limiting eigenvalue distribution of a matrix $Y \sim \GOE(d)$.
\item 
If $m/d \to \kappa > 0$ as $d \to \infty$,
the Marchenko-Pastur law
\begin{equation}
\label{eq:def_muMP}
\mu_\MP(x)\coloneqq \begin{cases}(1-\kappa) \delta(x)+\frac{\kappa \sqrt{\left(\lambda_{+}-x\right)\left(x-\lambda_{-}\right)}}{2 \pi x} & \text { if } \kappa \leq 1, \\ \frac{\kappa \sqrt{\left(\lambda_{+}-x\right)\left(x-\lambda_{-}\right)}}{2 \pi x} & \text { if } \kappa \geq 1,\end{cases}
\end{equation}
where $\lambda_{ \pm}\coloneqq (1 \pm \kappa^{-1 / 2})^2$, is the limiting eigenvalue distribution of $Y \sim \mcW_{m,d}$.
\end{itemize}

\myskip
\textbf{Additive free convolution: symmetric case --}
We will also consider the  (additive) free convolution of measures. Informally, we can interpret the free convolution $\mu\boxplus\nu$ of two measures $\mu$ and $\nu$ as the limiting spectral measure of $A+B$, where $A,B
\in\mcS_d$ are two random matrices with limiting spectral measures $\mu$ and $\nu$ respectively, when $A$ and $B$ satisfy a condition called \emph{asymptotic freeness}. We refer to \cite{anderson2010introduction} for more
formal discussions on asymptotic freeness and free convolution. In this paper, we will only use the special case $B\sim\GOE(d)$ independently of $A$, so that $\nu = \mu_\sci$. In this case, $A,B$ are asymptotically free \cite[Theorem 5.4.5]{anderson2010introduction}.

\subsection{Properties of priors satisfying Assumptions \ref{ass:prior} and \ref{ass:prior_rec}}
\label{subsec_app:properties_prior}

\subsubsection{The symmetric case: Assumption~\ref{ass:prior}}

Let us first note that the two cases of Assumption~\ref{ass:prior} are essentially equivalent.
\begin{remark}
\noindent
Assumption \ref{ass:prior}(ii) can be seen as a special case
of Assumption \ref{ass:prior}(i). Indeed, under assumption \ref{ass:prior}(ii) the joint distribution of the non-zero eigenvalues reads (see~\cite[Theorem~2]{uhlig1994singular}):
\begin{equation}
P(\{\lambda_i\}_{i=1}^{m})\propto \Pi_{i<j}|\lambda_i-\lambda_j|e^{-d\sum_{i=1}^{m}\tilde{\mathcal{V}}(\lambda_i)},
\label{eq:tilde_V}
\end{equation}
where $\tilde{\mathcal{V}}(s)\coloneqq \mathcal{V}(s)-(1-\frac{m}{d})\log s$. 
Eq.~\eqref{eq:tilde_V} is equivalent to the joint eigenvalue distribution under assumption \ref{ass:prior} (i), replacing $d$ by $m$ and considering the potential $\tilde{\mathcal{V}}$.
\end{remark}
We now list important properties of $P_0$ under Assumption~\ref{ass:prior}.
\begin{proposition}
\label{prop:properties_P0_sym}
\noindent
Assumption \ref{ass:prior} implies the following properties.
\begin{itemize}
\item[$(i)$] According to \cite[Remark 1.3]{fan2015convergence}, there exists $\mu_0$, a probability measure with compact support such that the empirical eigenvalue distribution $\mu_S$ of $S \sim P_0$ a.s. weakly converges to $\mu_0$.
\item [$(ii)$] There exists $M,c'>0$ such that
\begin{equation}
P_0(\{\lambda_i\}_{i=1}^d\subset [-M,M])\geq 1-e^{-c'd},
\label{eq:max_eigs_deviation}
\end{equation}
where $\{\lambda_i\}_{i=1}^d$ are the eigenvalues of $S \sim P_0$.
\item [$(iii)$] Eq.~\eqref{eq:max_eigs_deviation} implies that
\begin{equation*}
P_0\left(\frac{1}{d}\sum_{i=1}^d\lambda_i^2>M^2\right)\leq e^{-c'd}.
\end{equation*}
By the Borel–Cantelli lemma, this implies that $(1/d)\sum_{i=1}^d\lambda_i^2$ is almost surely bounded, and thus uniformly integrable.

\item [$(iv)$] By \cite[Theorem 7.12-(ii)]{villani2021topics}, as $(1/d)\sum_{i=1}^d\lambda_i^2$ is a.s. bounded and $\mu_S$ a.s. weakly converges to $\mu_0$, we have
\begin{equation}
\label{eq:W2convergence}
\lim_{d\to\infty}W_2(\mu_S,\mu_0)=0,\ a.s.,
\end{equation} 
where $W_2$ denotes the Wasserstein-2 distance.
This further implies that
\begin{equation}\label{eq:def_rho}
\lim_{d\to\infty}\frac{1}{d}\sum_{i=1}^d\lambda_i^2=\rho \coloneqq \int \mu_0(\rd x) \, x^2,\ a.s.
\end{equation}
As the left side is uniformly integrable and the unordered eigenvalues $\{\lambda_i\}_{i=1}^d$ have identical marginals, we have
\begin{equation}
\label{eq:moment_converge}
\lim_{d\to\infty}\EE_{S \sim P_0}[(\lambda_1)^2]=\rho,
\end{equation}
which will be essential for our proof. Finally, one can show in a similar way that 
\begin{equation*}
\lim_{d\to\infty}\EE_{S \sim P_0}[|\lambda_1|^p]<+\infty,
\end{equation*}
for any $p\geq1$. 
\end{itemize}
\end{proposition}
Properties~$(i)$, $(iii)$ and $(iv)$ of Proposition~\ref{prop:properties_P0_sym} are elementary consequences of the mentioned results. Property~$(ii)$ on the other hand is a consequence of the following lemma.
\begin{lemma}[Boundedness of eigenvalues]
\label{lemma:boundedness_eigenvalues}
\noindent
Let $P_0$ have the density:
\begin{equation*}
 P_0(S)\propto\exp(-d \, \Tr[\mathcal{V}(S)])
\end{equation*}
for a real continuous potential $\mathcal{V} : \mathcal{B} \to \bbR$, where $\mathcal{B}\subseteq\mathbb{R}$ can be $\mathbb{R}$, an interval, or the union of finitely many disjoint intervals. If $\mathcal{B}$ is unbounded and
\begin{equation}
\liminf_{x\in\mathcal{B},|x|\to\infty}\frac{\mathcal{V}(x)}{2\log|x|}>1,\label{eq:faster_than_2logx}
\end{equation}
then there exists $M > 0$ and $c>0$ such that
\begin{equation*}
P_0(\{\lambda_i\}_{i=1}^d\subset [-M,M])\geq 1-e^{-c'd}.
\end{equation*}
\end{lemma}
\begin{proof}
\cite[Theorem 1.2]{fan2015convergence} shows that there exists a constant $C_V$ such that
\begin{equation*}
J(x)\coloneqq \mathcal{V}(x)-2\int\mu_0(dy)\log|x-y|+C_V>0
\end{equation*}
for almost every $x\in\mathcal{B}/\mathcal{S}$, where $\mathcal{S}$ is the support of the limit eigenvalue distribution $\mu_0$. Moreover, \cite[Theorem 1.4]{fan2015convergence} shows that
\begin{equation*}
\limsup_{d\to\infty}\frac{1}{d}\log P_0(\exists i,|\lambda_i|>M)\leq-\frac{1}{2}\inf_{x\in\mathcal{B},|x|\geq M}J(x)
\end{equation*}
for any $M$ such that $\mathcal{S}\subset[-M,M]$. Under \eqref{eq:faster_than_2logx}, we have $\liminf_{_{x\in\mathcal{B},|x|\geq M}}J(x)=+\infty$. Therefore, since $J$ is continuous, there exists $M,c'>0$ such that
\begin{equation*}
\limsup_{d\to\infty}\frac{1}{d}\log P_0(\exists i,|\lambda_i|>M)\leq-c',
\end{equation*}
which finishes the proof.
\end{proof}
It is easy to see that eq.~\eqref{eq:faster_than_2logx} holds under either Assumption~\ref{ass:prior} or \ref{ass:prior_rec}.
Moreover, notice that it is sufficient to assume eq.~\eqref{eq:faster_than_2logx} to deduce all properties of Proposition~\ref{prop:properties_P0_sym} (except).

\begin{remark}[Weakening of Assumption~\ref{ass:prior}]
\label{remark:weakening_ass_prior}
\noindent
For our results to hold (i.e.\ Theorem~\ref{thm:fentropy_symmetric} and \ref{theo:overlap}), Assumption~\ref{ass:prior}-$(i)$ can be weakened by assuming both: 
\begin{itemize}
    \item If $\mcB$ is unbounded:
    \begin{equation*}
    \liminf_{x \in B, |x| \to \infty}\frac{\mathcal{V}(x)}{2\log|x|}>1,
    \end{equation*}
    \item $\mathcal{V}''(s) \geq 1/c$ for all $s \in \mcB \cap [-M,M]$, where $M > 0$ is defined by Lemma~\ref{lemma:boundedness_eigenvalues}.
\end{itemize}
One can similarly define a weakening of Assumption~\ref{ass:prior}-$(ii)$ under which our results hold.
\end{remark}
Finally, recalling the definition of $\mu_0$ in Property~\ref{prop:properties_P0_sym}-$(i)$ and $\mu_\sci$ in eq.~\eqref{eq:def_musci},
we denote 
\begin{equation}
    \label{eq:def_mut}
    \mu_t\coloneqq \mu_0\boxplus\mu_{\sci,\sqrt{t}}, 
\end{equation}
 where 
$\mu_{\sci,\sqrt{t}}(x)\coloneqq t^{-1/2}\mu_\sci(x/\sqrt{t})$.

\subsubsection{The rectangular case: Assumption~\ref{ass:prior_rec}}

We directly get the following from Proposition~\ref{prop:properties_P0_sym} and Remark~\ref{remark:rec_sym_prior}.
\begin{proposition}
\label{prop:properties_P0_rec}
\noindent
Assume that $S \in \bbR^{d \times L}$ with $S \sim P_0$, satisfying Assumption~\ref{ass:prior_rec} (recall in particular that $L \geq d$ and that $d/L \to \beta \leq 1$). Then:
\begin{itemize}
\item [$(i)$] There exists $M, c > 0$ such that 
\begin{equation*}
P_0(\{\sigma_i\}_{i=1}^d\subset [0, M])\geq 1-e^{-c d},
\end{equation*}
where $\{\sigma_i\}_{i=1}^d$ are the singular values of $S$. 
\item [$(ii)$] For $S \in \bbR^{d \times L}$, letting $\{\sigma_i\}_{i=1}^d$ be its singular values, we denote 
\begin{equation}\label{eq:def_hmuS}
    \hmu_S(x) \coloneqq \frac{1}{2d} \sum_{i=1}^d [\delta_{\sigma_i} + \delta_{-\sigma_i}]
\end{equation} 
 the \emph{symmetrized} empirical singular value distribution of $S$. Then there exists $\hmu_0$, a probability measure with compact support, such that, for $S\sim P_0$: 
\begin{equation*}
\lim_{d\to\infty}W_2(\hmu_S,\hmu_0)=0,\ a.s..
\end{equation*} 
Moreover, we also have
\begin{equation*}
\rho \coloneqq\sqrt{\beta}\lim_{d\to\infty}\EE_{S \sim P_0}[(\sigma_1)^2]= \sqrt{\beta}\int \hmu_0(\rd x) \, x^2,
\end{equation*}
and
\begin{equation*}
\lim_{d\to\infty}\mathbb{E}[|\sigma_1|^p]<+\infty
\end{equation*}
for any $p\geq1$. 
\end{itemize}
\end{proposition}

\begin{remark}
\noindent
As the counterpart to Remark~\ref{remark:psi_P0}, Lemma~\ref{lemma:HCIZ} shows that
\begin{equation}
\psi^{\rec}_{P_0}(r)=\lim_{d\to\infty}\frac{1}{dL}\EE_{Y'}\log\int P_0(\rd S)e^{-\frac{1}{2}\sqrt{dL}r\Tr[(S)^TS]+\sqrt{dLr}\Tr[(Y')^TS]},
\label{eq:HCIZ}
\end{equation}
where $Y'\coloneqq\sqrt{r}S^\star+Z'$ with $S^\star\sim P_0$ and $Z'\iid\mcN(0,1/\sqrt{dL})$. The right-hand side of eq.~\eqref{eq:HCIZ} can be regarded as the definition of $\psi_{P_0}^\rec$ as well. Similarly to the symmetric setting, it is a natural generalization of the free entropy of the scalar denoising problem in \cite{barbier2019optimal} to the problem of rectangular matrix denoising~\cite{troiani2022optimal}.
\end{remark}

\myskip
\textbf{Additive free convolution: rectangular case --}
We refer to~\cite{benaych2009rectangular} for the general definition of free convolution for rectangular matrices.
Similarly to the symmetric setting, we will consider the matrix $S+\sqrt{t}Z$, where $S,Z\in\bbR^{d\times L}$, $S$ satisfies Assumption~\ref{ass:prior_rec}
and $Z_{ij}\iid\mcN(0,1/\sqrt{dL})$, independently of $S$. 
$S$ has a well-defined limit symmetrized singular value distribution $\hmu_0$ (see Proposition~\ref{prop:properties_P0_rec}), and $Z$ also admits a well-defined limit 
symmetrized singular value distribution $\hmu_{\MP,\rec}$, where the latter is closely related to $\mu_\MP$ of eq.~\eqref{eq:def_muMP}, since $Z^\T Z$ is (up to a global scaling) a Wishart matrix. 
Generalizing eq.~\eqref{eq:def_mut}, we denote 
the limit symmetrized singular value distribution of $S+\sqrt{t}Z$ as $\tmu_t$, which is given by:
\begin{equation}\label{eq:hmu_symmetrized}
  \tmu_t\coloneqq \hmu_0\boxplus_\beta\hmu_{\MP,\rec,\sqrt{t}},  
\end{equation}
where $\hmu_{\MP,\rec,\sqrt{t}}(x)\coloneqq t^{-1/2}\hmu_{\MP,\rec}(x/\sqrt{t})$, and $\boxplus_\beta$ denotes the rectangular free convolution~\cite{benaych2009rectangular}.

\section{Universality: proof of Lemma~\ref{lemma:universality}}\label{sec_app:universality}
In this section we prove Lemma~\ref{lemma:universality}. Recall that Lemma~\ref{lemma:universality} relies on Assumptions \ref{assum:CLT}, \ref{assum:S_strong}, and \ref{assum:phi_strong}. Consequently, all lemmas in this section are also based on these assumptions. Recall that
\begin{equation*}
\ell(x, X,Z)\coloneqq -\log\int \rd P_A(a)\exp\left(-\frac{(\varphi(x, a)-\varphi(X, A))^2-2Z(\varphi(x, a)-\varphi(X, A))}{2\Delta}\right).
\end{equation*}
One can then check easily that
\begin{equation}\label{eq:bound_ell_1}
|\ell(x,X,Z)| \leq \frac{2\sup|\varphi|^2+2|Z|\sup|\varphi|}{\Delta},
\end{equation}
and
\begin{equation}\label{eq:bound_dxell}
\begin{aligned}
\left|\frac{\partial\ell(x,X,Z)}{\partial x}\right|&\leq\frac{\int \rd P_A(a)(4\sup|\varphi|\sup|\varphi'|+2|Z|\sup|\varphi'|)\exp\left(\frac{4|\sup|\varphi|^2+4|Z|\sup|\varphi|}{2\Delta}\right)}{\int \rd P_A(a)\exp\left(-\frac{4|\sup|\varphi|^2+4|Z|\sup|\varphi|}{2\Delta}\right)}\\
&=(4\sup|\varphi|\sup|\varphi'|+2|Z|\sup|\varphi'|)\exp\left(\frac{4|\sup|\varphi|^2+4|Z|\sup|\varphi|}{\Delta}\right),
\end{aligned}
\end{equation}
where $\varphi'$ denote the derivative of $\varphi(x, a)$ w.r.t.\ its first argument.
Eqs.~\eqref{eq:bound_ell_1},\eqref{eq:bound_dxell} imply that, for $X \in \{\Phi, G\}$:
\begin{equation}
\label{eq:bound_ell}
\begin{cases}
|\ell(X_\mu)| &\leq C_1+C_2|Z_\mu|, \\ 
|\partial_1\ell(X_\mu)|,|\partial_2\ell(X_\mu)|&\leq C_3e^{C_4|Z_\mu|}, 
\end{cases}
\end{equation}
where we use the shorthand $\ell(X_\mu)\coloneqq\ell(\Tr[X_\mu S],\Tr[X_\mu S^\star],Z_\mu)$ and $\partial_1\ell(X_\mu),\partial_2\ell(X_\mu)$ being the derivatives of $\ell$ w.r.t. its first and second arguments. $C_1, C_2,C_3,C_4$ are nonnegative constants that only depend on $\varphi$ and $\Delta$. 

\myskip
More specifically, we will prove that 
\begin{equation}
\lim_{d\to\infty}\sup_{\{A_\mu\}_{\mu=1}^n}\sup_{S^\star\in B_{\op}(M)}|\EE_{\Phi,Z}[\psi(F_d(\Phi))]-\EE_{G,Z}[\psi(F_d(G))]|=0,\label{eq:universality}
\end{equation}
which implies eq.~\eqref{eq:lemma_universality} by the dominated convergence theorem.

\myskip
We will consider the following interpolation model, for $t \in [0, \pi/2]$:
\begin{equation*}
\begin{dcases}
    U_\mu(t) &\coloneqq\cos(t)\Phi_\mu+\sin(t)G_\mu,\\
    \tU_\mu(t) &\coloneqq-\sin(t)\Phi_\mu+\cos(t)G_\mu.
\end{dcases}
\end{equation*}
It is worth noticing that the first and second moments of $U_\mu,\tU_\mu$ match the ones of $\GOE(d)$ by Assumption~\ref{assum:CLT}.
We start with a first technical lemma.
\begin{lemma}
    \noindent
    Denote $\Theta^{(1)}$ to be the set of all the variables $\{\{G_\mu,\Phi_\mu\}_{\mu=2}^n, \{A_\mu\}_{\mu=1}^n, S^\star\in B_{\op}(M)\}$.
    Then 
\begin{align*}
\sup_{d\geq1}\sup_{t \in [0,\pi/2]}\sup_{\Theta^{(1)}} \EE_{G_1,\Phi_1,Z_1}\left[\frac{\langle e^{-\ell(U_1)}|\Tr[S\tU_1(t)]\partial_1\ell(U_1(t))+\Tr[S^\star\tU_1(t)]\partial_2\ell(U_1(t))|\rangle_1}{\langle e^{-\ell(U_1)}\rangle_1}\right] <+\infty,
\end{align*}
where we denoted
\begin{equation*}
\langle\cdot\rangle_\mu\coloneqq\frac{\int P_0(\rd S)e^{-\sum_{\nu\neq\mu}\ell(U_\nu(t))}(\cdot)}{\int P_0(\rd S)e^{-\sum_{\nu\neq\mu}^n\ell(U_\nu(t))}}.
\end{equation*}
\label{lemma:sup_bound}
\end{lemma}
\begin{proof}[Proof of Lemma~\ref{lemma:sup_bound}]
We have:
\begin{equation*}
\begin{aligned}
&\EE_{G_1,\Phi_1,Z_1}\frac{\langle e^{-\ell(U_1)}|\Tr[S\tU_1(t)]\partial_1\ell(U_1(t))+\Tr[S^\star\tU_1(t)]\partial_2\ell(U_1(t))|\rangle_1}{\langle e^{-\ell(U_1)}\rangle_1}\\
&\leq \EE_{Z_1}C_3e^{C_1+(C_2+C_4)|Z_1|}\EE_{G_1,\Phi_1}\langle(|\Tr[S\tU_1(t)]|+|\Tr[S^\star\tU_1(t)]|\rangle_1\\
&\leq C\langle(\EE_{G_1,\Phi_1}[\Tr[S\tU_1(t)]]^2)^{1/2}+(\EE_{G_1,\Phi_1}[\Tr[S^\star\tU_1(t)]]^2)^{1/2}\rangle_1\\
&=C\langle(\Tr[S^2]/d)^{1/2}+(\Tr[(S^\star)^2]/d)^{1/2}\rangle_1\\
&\leq CM,
\end{aligned}
\end{equation*}
where $C$ only depends on $\varphi,\Delta$. We used eq.~\eqref{eq:bound_ell} for the first inequality, the Cauchy-Schwarz inequality, and the Jensen inequality for the second inequality, Assumption \ref{assum:CLT} for the third equality, and Assumption \ref{assum:S_strong} for the last inequality.
\end{proof}
We will need the following domination lemma to control eq.~\eqref{eq:interpolation_fentropy}.
\begin{lemma}[Domination]
\label{lemma:universality1}
\begin{equation*}
\int_0^{\pi/2}\sup_{d\geq1}\sup_{\{A_\mu\}_{\mu=1}^n}\sup_{S^\star\in B_{\op}(M)}\left|\EE_{G,W,Z}\frac{\partial\psi(F_d(U(t)))}{\partial t}\right|\rd t<\infty.
\end{equation*}
\end{lemma}
\begin{proof}[Proof of Lemma~\ref{lemma:universality1}]
By definition, we have
\begin{equation*}
\begin{aligned}
&\frac{\partial\psi(F_d(U(t)))}{\partial t}=-\frac{\psi'(F_d(U(t)))}{d^2}\times\\&\quad\sum_{\mu=1}^n\frac{\int P_0(\rd S)e^{-\sum_{\nu=1}^n\ell(U_\nu(t))}(\Tr[S\tU_\mu(t)]\partial_1\ell(U_\mu(t))+\Tr[S^\star\tU_\mu(t)]\partial_2\ell(U_\mu(t)))}{\int P_0(\rd S)e^{-\sum_{\nu=1}^n\ell(U_\nu(t))}},
\end{aligned}
\end{equation*}
which gives
\begin{equation*}
\begin{aligned}
&\EE_{G,\Phi,Z}\left[\frac{\partial\psi(F_d(U(t)))}{\partial t}\right]\\&\quad=-\frac{n}{d^2}\EE\left[\psi'(F_d(U(t)))\frac{\langle e^{-\ell(U_1)}(\Tr[S\tU_1(t)]\partial_1\ell(U_1(t))+\Tr[S^\star\tU_1(t)]\partial_2\ell(U_1(t)))\rangle_1}{\langle e^{-\ell(U_1)}\rangle_1}\right].
\end{aligned}
\end{equation*}
Then we have
\begin{equation*}
\begin{aligned}
&\sup_{d\geq1}\sup_{t \in [0,\pi/2]}\sup_{\{A_\mu\}_{\mu=1}^n}\sup_{S^\star\in B_{\op}(M)}\left|\EE_{G,\Phi,Z}\left[\frac{\partial\psi(F_d(U(t)))}{\partial t}\right]\right|\\
&\leq\sup_{d\geq1}\sup_{t \in [0,\pi/2]}\sup_{\{A_\mu\}_{\mu=1}^n}\sup_{S^\star\in B_{\op}(M)}\frac{n}{d^2}||\psi'||_\infty \\ 
& \hspace{2cm} \times \EE_{G,\Phi,Z}\frac{\langle e^{-\ell(U_1)}|\Tr[S\tU_1(t)]\partial_1\ell(U_1(t))+\Tr[S^\star\tU_1(t)]\partial_2\ell(U_1(t))|\rangle_1}{\langle e^{-\ell(U_1)}\rangle_1}<+\infty,
\end{aligned}
\end{equation*}
which finishes the proof. The last inequality follows from Lemma \ref{lemma:sup_bound}.
\end{proof}
The following lemma allows us to control the expectation of the derivative of the free entropy along the interpolation path.
It is a straightforward combination of \cite[Lemma~D.3]{maillard2023exact} and \cite[Lemma~2]{montanari2022universality}\footnote{\cite[Lemma~D.3]{maillard2023exact} deals with rotationally-invariant matrix priors like us, but in a non-planted model, while \cite[Lemma 3]{montanari2022universality} deals with the additional planted signal and randomness. Note that our loss function $\ell$ is locally Lipschitz, satisfying the assumption used in \cite{montanari2022universality}.}, 
and forms the core of the interpolation argument, where Assumption~\ref{assum:CLT} is primarily used. We refer to the two works~\cite{montanari2022universality,maillard2023exact} for more details.
\begin{lemma}
    \noindent
    Recall that we denote $\Theta^{(1)}$ to be the set of all the variables $\{\{G_\mu,\Phi_\mu\}_{\mu=2}^n, \{A_\mu\}_{\mu=1}^n, S^\star\in B_{\op}(M)\}$.
    Then:
\begin{equation*}
\lim_{d\to\infty}\sup_{\Theta^{(1)}}\left\langle\EE_{G_1,\Phi_1,Z_1}\frac{e^{-\ell(U_1)}(\Tr[S\tU_1(t)]\partial_1\ell(U_1(t))+\Tr[S^\star\tU_1(t)]\partial_2\ell(U_1(t)))}{\langle e^{-\ell(U_1)}\rangle_1}\right\rangle_1=0.
\end{equation*}
\label{lemma:finite_demension_CLT}
\end{lemma}
\begin{proof}[Proof of Lemma~\ref{lemma:finite_demension_CLT}]
Following
\cite[Lemma D.3]{maillard2023exact} and \cite[Lemma~2]{montanari2022universality}, we have:
\begin{align}
    \label{eq:replace_Phi_G}
    &
    \hspace{-10pt}
    \lim_{d\to\infty}\sup_{\Theta^{(1)}}\left\langle\EE_{G_1,\Phi_1,Z_1}\frac{e^{-\ell(U_1)}(\Tr[S\tU_1(t)]\partial_1\ell(U_1(t))+\Tr[S^\star\tU_1(t)]\partial_2\ell(U_1(t)))}{\langle e^{-\ell(U_1)}\rangle_1}\right\rangle_1\\
    \nonumber
    =&\lim_{d\to\infty}\sup_{\Theta^{(1)}}\left\langle\EE_{G_1,\tilde{G}_1,Z_1}\frac{e^{-\ell(V_1)}(\Tr[S\tV _1(t)]\partial_1\ell(V_1(t))+\Tr[S^\star\tV _1(t)]\partial_2\ell(V_1(t)))}{\langle e^{-\ell(V_1)}\rangle_1}\right\rangle_1=0,
\end{align}
where $V_1\coloneqq\cos(t)\tilde{G}_1+\sin(t)G_1,\tV _1\coloneqq-\sin(t)\tilde{G}_1+\cos(t)G_1$ with $\tilde{G}_1\sim\text{GOE(d)}$ independent of $G_1$. 
Informally, this is a consequence of Assumption~\ref{assum:CLT}, which allows to replace in the left-hand side of 
eq.~\eqref{eq:replace_Phi_G} the matrix $\Phi_1$ by a $\GOE(d)$ matrix $\tG_1$.
In detail, the proof follows directly from the arguments of~\cite[Lemma~2]{montanari2022universality}, to condition on the event $|Z_1|\leq B$ for an arbitrary $B > 0$, and 
then directly using the proof arguments of~\cite[Lemma~D.3]{maillard2023exact} (see also~\cite[Lemma~3]{montanari2022universality}) under this event. Taking the limit $B \to \infty$ allows then to end the proof.
The last equality uses that $\EE[\Tr[S\tV _1(t)]]=\EE[\Tr[S^\star\tV _1(t)]]=0$ since $V_1(t)$ and $\tV_1(t)$ are now \emph{independent} $\GOE(d)$ matrices.
\end{proof}
A direct consequence of Lemma~\ref{lemma:finite_demension_CLT} is a control of the limit of the time derivative in eq.~\eqref{eq:interpolation_fentropy}.
\begin{lemma}
\begin{equation*}
\lim_{d\to\infty}\sup_{\{A_\mu\}_{\mu=1}^n}\sup_{S^\star\in B_{\op}(M)}\EE_{G,\Phi,Z}\frac{\partial\psi(F_d(U(t)))}{\partial t}=0.
\end{equation*}
\label{lemma:universality2}
\end{lemma}
\begin{proof}[Proof of Lemma~\ref{lemma:universality2}]
We have
\begin{equation*}
\left|\EE_{G,\Phi,Z}\frac{\partial\psi(F_d(U(t)))}{\partial t}\right|\leq\frac{n}{d^2}(I_1+I_2),
\end{equation*}
The first term is
\begin{equation*}
\begin{aligned}
I_1&\coloneqq\left|\EE_{G,\Phi,Z}\left[(\psi'(F_d(U))-\psi'(F_d(U^{(1)})))\times\right.\right.\\&\left.\left.\left(\frac{\langle e^{-\ell(U_1)}(\Tr[S\tU_1(t)]\partial_1\ell(U_1(t))+\Tr[S^\star\tU_1(t)]\partial_2\ell(U_1(t)))\rangle_1}{\langle e^{-\ell(U_1)}\rangle_1}\right)\right]\right|,
\end{aligned}
\end{equation*}
where $U^{(1)}$ denotes $\{\circ,U_2,\cdots,U_n\}$, the symbol $\circ$ denoting that we remove the terms corresponding to $U_1$ and $Z_1$ in eq.~\eqref{eq:free_entropy_universality}. 
Concretely:
\begin{align*}
    F_d(U) - F_d(U^{(1)}) &= \frac{1}{d^2} \log \left\langle e^{-\ell(\Tr[G_1S],\Tr[G_1S^\star],Z_1)}\right\rangle_1 - \frac{Z_1^2}{2 \Delta d^2} .
\end{align*}
Repeating the proof of the bound of Lemma \ref{lemma:sup_bound} with the additional upper bound (which uses again eq.~\eqref{eq:bound_ell}):
\begin{equation*}
\begin{aligned}
    |\psi'(F_d(U))-\psi'(F_d(U^{(1)}))|&\leq\frac{||\psi''||_\infty}{d^2}\left|\log\left\langle e^{-\ell(\Tr[G_1S],\Tr[G_1S^\star],Z_1)}\right\rangle_1 + \frac{Z_1^2}{2\Delta}\right|, 
    \\&\leq\frac{||\psi''||_\infty}{d^2} \, \left(e^{C_1+C_2|Z_1|} + \frac{Z_1^2}{2\Delta}\right),
\end{aligned}
\end{equation*}
we reach $I_1\to0$, uniformly in $S^\star, \{A_\mu\}_{\mu=1}^n$.
The second term is (since $F_d(U^{(1)})$ is now independent of $G_1, \Phi_1, Z_1$):
\begin{align*}
        &I_2\coloneqq\left|\EE_{G,\Phi,Z}\left[\psi'(F_d(U^{(1)})) \, \left(\frac{\langle e^{-\ell(U_1)}(\Tr[S\tU_1(t)]\partial_1\ell(U_1(t))+\Tr[S^\star\tU_1(t)]\partial_2\ell(U_1(t)))\rangle_1}{\langle e^{-\ell(U_1)}\rangle_1}\right)\right]\right|, \\ 
        &\leq \EE_{\{G_\mu,\Phi_\mu,Z_\mu\}_{\mu=2}^n} \left[|\psi'(F_d(U^{(1)}))| \, \left| \EE_{G_1, \Phi_1, Z_1}\left(\frac{\langle e^{-\ell(U_1)}(\Tr[S\tU_1(t)]\partial_1\ell(U_1(t))+\Tr[S^\star\tU_1(t)]\partial_2\ell(U_1(t)))\rangle_1}{\langle e^{-\ell(U_1)}\rangle_1}\right)\right|\right], \\ 
        & \to0
\end{align*}
by Lemma \ref{lemma:finite_demension_CLT}, where the limit is also uniform in $S^\star, \{A_\mu\}_{\mu=1}^n$. 
This finishes the proof of Lemma~\ref{lemma:universality2}.
\end{proof}
Finally, as 
\begin{equation*}
|\EE_{\Phi,Z}[\psi(F(\Phi))]-\EE_{G,Z}[\psi(F(G))]|\leq\int_0^{\pi/2}\left|\EE_{G,\Phi,Z}\frac{\partial\psi(F(U(t)))}{\partial t}\right|\rd t,
\end{equation*}
we prove eq.~\eqref{eq:sufficient_universality} by combining Lemma \ref{lemma:universality1}, Lemma \ref{lemma:universality2}, and using the dominated convergence theorem.
This ends our proof of Lemma~\ref{lemma:universality}.

\section{Proof of Theorem~\ref{thm:fentropy_symmetric}}\label{sec_app:proof_GLM}
In this section we prove Theorem~\ref{thm:fentropy_symmetric}. As mentioned in the main text, our proof uses an adaptive interpolation argument, as we generalize the proof approach of~\cite{barbier2019optimal} to our more general setting of structured matrix priors. Nevertheless, a significant portion of our proof arguments follow from direct transpositions of the arguments of~\cite{barbier2019optimal}, and we will refer to them when necessary.

\subsection{Proof of Theorem~\ref{thm:fentropy_symmetric} via the adaptive interpolation method}
\label{sec:interpolation}
We will first prove Theorem \ref{thm:fentropy_symmetric} under Assumptions \ref{assum:S_strong} and \ref{assum:phi_strong} and for Gaussian data $\{G_\mu\}_{\mu=1}^n \iid \GOE(d)$. All lemmas in Sections \ref{sec:interpolation}, \ref{sec:concentration_free_entropy} and \ref{sec:overlap_concentration} are under such assumptions. Thanks to Lemma \ref{lemma:universality}, Gaussian data can be replaced by more general data in the end. Assumptions \ref{assum:S_strong} and \ref{assum:phi_strong} can then be relaxed as detailed in Section \ref{sec_app:relax}, by leveraging Lemmas~\ref{lemma:relaxV} and \ref{lemma:relaxphi}.

\myskip
Let us consider the following interpolation model
\begin{equation}
\left\{  
 \begin{aligned}
&Y_{t,\mu}\sim P_{\out}(\cdot|J_{t,\mu}),\ \mu=1,\cdots,n  \\  
&Y_t'=\sqrt{d}\left(\sqrt{R_1(t)}S^\star +Z'\right)\in\mcS_d,
 \end{aligned}  
\right.  
\label{eq:GLM_interpolate}
\end{equation}
with 
\begin{equation*}
J_{t,\mu}\coloneqq \sqrt{1-t}\Tr[G_\mu S^\star ]+\sqrt{2R_2(t)}V_\mu+\sqrt{2\rho t-2R_2(t)+2\iota_d}U_\mu^\star ,
\end{equation*}
$Z'\sim\text{GOE}(d)$ and $\{V_\mu,U_\mu^\star \}_{\mu=1}^n\overset{iid}{\sim}\mathcal{N}(0,1)$. The problem is to recover $S^\star,U^\star$ from the observations $\{Y_{t,\mu}\}_{\mu=1}^n,Y_t'$ given the knowledge of $\{G_\mu\}_{\mu=1}^n$ and $\{V_\mu\}_{\mu=1}^n$. Thus the $\sqrt{d}$ scaling in \eqref{eq:GLM_interpolate} is only for simplicity and will not influence the free entropy. The interpolating functions are given by
\begin{equation*}
R_1(t)\coloneqq \epsilon_1+\int_0^tr(v)\rd v,\ R_2(t)\coloneqq \epsilon_2+\int_0^t q(v)\rd v,
\end{equation*}
where $(\epsilon_1,\epsilon_2)\in B_d\coloneqq [\iota_d,2\iota_d]\otimes[\iota_d,2\iota_d]$ are two small quantities and $q:[0,1]\to[0,\rho]$, $r:[0,1]\to[0,r_{\max}]$ are two continuous functions (that might depend on $\epsilon\coloneqq (\epsilon_1,\epsilon_2)$ as well) with $r_{\max}\coloneqq 4\alpha\sup_{q\in[0,\rho]}\Psi_{\out}'(q)=4\alpha\Psi_{\out}'(\rho)$ (by the convexity of $\Psi_\out$ shown in Lemma~\ref{lemma:Psi}). 
We will choose $\iota_d \to 0$ as $d \to \infty$ slowly enough, it remains arbitrary for now.

\myskip
The interpolating free entropy reads
\begin{equation*}
f_{d,\epsilon}(t)\coloneqq \frac{1}{d^2}\EE\log \mathcal{Z}_{t,\epsilon}(Y,Y',G,V)\coloneqq \frac{1}{d^2}\EE\log\int P_0(\rd s)\mcD u \, e^{-H_{t,\epsilon}(s,u,Y,Y',G,V)},
\end{equation*}
where $\mcD u\coloneqq(2\pi)^{-n/2}e^{-\sum_{\mu=1}^nu_\mu^2/2}\rd u$ is the standard Gaussian measure, and
\begin{equation*}
H_{t,\epsilon}(s,u,Y,Y',G,V)\coloneqq \frac{1}{2}\sum_{\mu=1}^{n}u_{Y_{t,\mu}}(j_{t,\mu})+\frac{1}{4}\sum_{i,j=1}^{d}(Y_{ij}'-\sqrt{dR_1(t)}s_{ij})^2
\end{equation*}
is the Hamitonian. Moreover, we have defined
\begin{equation*}
u_{Y_{t,\mu}}(j_t)\coloneqq \log P_{\text{out}}(Y_{t,\mu}|j_{t,\mu})
\end{equation*}
with $j_{t,\mu}\coloneqq \sqrt{1-t}\Tr[G_\mu s]+\sqrt{2R_2(t)}V_\mu+\sqrt{2\rho t-2R_2(t)+2\iota_d}u_\mu$. Accordingly, the Gibbs bracket is defined as
\begin{equation*}
\langle g(s,u)\rangle\coloneqq \frac{1}{\mathcal{Z}_{t,\epsilon}(Y,Y',G,V)}\int P_0(\rd s)\mcD u \, g(s,u)e^{-H_{t,\epsilon}(s,u,Y,Y',G,V)}.
\end{equation*}
The following lemma connects the interpolation model to the original model.
\begin{lemma}
\begin{equation*}
\begin{aligned}
&f_{d,\epsilon}(0)=f_d-\frac{1}{4}+O(\iota_d),\\
&f_{d,\epsilon}(1)=\psi_{P_0}\left(\int_0^1r(t)\rd t\right)+\alpha\Psi_{\out}\left(\int_0^1q(t)\rd t\right)+O(\iota_d).
\end{aligned}
\end{equation*}
\label{lemma:f0_f1}
\end{lemma}
\begin{proof}[Proof of Lemma \ref{lemma:f0_f1}]
We have
\begin{equation*}
\begin{aligned}
\left|\frac{\rd f_{d,\epsilon}(0)}{\rd\epsilon_1}\right|&=\frac{1}{4d}\left|\EE\sum_{i,j=1}^d\langle\epsilon_1(S_{ij}-s_{ij})+\sqrt{\frac{1}{\epsilon_1}}(S_{ij}-s_{ij})Z_{ij}'\rangle\right|\\
&=\frac{1}{2d}\left|\EE\sum_{i,j=1}((S_{ij}^\star )^2-\langle s_{ij}\rangle^2)\right|\leq M^2
\end{aligned}
\end{equation*}
where we use the Nishimori identity (Proposition~\ref{prop:nishimori}) and the fact that $\sum_{i,j=1}^d(S_{ij}^\star )^2=\Tr[(S^\star )^2]\leq dM^2$, and the same for $\langle s\rangle$ (according to Assumption \ref{assum:S_strong}). Moreover,
\begin{equation*}
\left|\frac{df_{d,\epsilon}(0)}{d\epsilon_2}\right|=\frac{1}{2d^2}\sum_{\mu=1}^n|\EE[u_{Y_{0,\mu}}'(J_{t,\mu})\langle u_{Y_{0,\mu}}'(j_{t,\mu})\rangle]|\leq C(\varphi,M,\alpha),
\end{equation*}
where the inequality follows from Assumption \ref{assum:phi_strong} ( see \cite[Section A.6]{barbier2019optimal}) and $C(\varphi, M, \alpha)$ is a generic non-negative constant depending only on $\varphi, M$ and $\alpha$. Therefore, we have $|f_{d,\epsilon}(0)-f_{d,0}(0)|\leq C(\varphi,M,\alpha)\iota_d$. As $f_{d,0}(0)=f_d-\frac{1}{4}$, we obtain the first equality of Lemma~\ref{lemma:f0_f1}. The second equality is obtained through
\begin{equation*}
f_{d,\epsilon}(1)=\psi_{P_0}(R_1(1))+\alpha\Psi_{\out}(R_2(1)+\iota_d)
\end{equation*}
and using the Lipschitz property of $\psi_{P_0}$ and $\Psi_{\out}$ (Lemmas \ref{lemma:Lipschitz_psi}, \ref{lemma:Psi}).
\end{proof}

\myskip
Two important concentration lemmas are presented below, where we use $C(\varphi, M, \alpha,\kappa)$ for a generic non-negative constant dependent only on $\varphi, M$, $\alpha$, and $\kappa$.
\begin{lemma}
\begin{equation*}
\EE\left[\left(\frac{1}{d^2}\log\mathcal{Z}_{t,\epsilon}-f_{d,\epsilon}(t)\right)^2\right]\leq\frac{C(\varphi,M,\alpha,\kappa)}{d^2}.
\end{equation*}
\label{lemma:concentration_free_entropy}
\end{lemma}
Lemma \ref{lemma:concentration_free_entropy} is proved in Section \ref{sec:concentration_free_entropy}.

\begin{lemma}
\begin{equation*}
\frac{1}{\iota_d^2}\int_{B_d} \rd\epsilon\int_0^1\rd t\EE\langle(Q-\EE\langle Q\rangle)^2\rangle\leq C(\varphi,M,\alpha,\kappa) \cdot \smallO_d(1),
\end{equation*}
where $Q\coloneqq \tr[s S^\star] = (1/d)\sum_{i,j=1}^ds_{ij}S_{ij}^\star $.
\label{lemma:concentration_Q}
\end{lemma}
Lemma~\ref{lemma:concentration_Q} is proven in Section \ref{sec:overlap_concentration}.
We now prove the following result on the derivative of the free entropy along the interpolation path:
\begin{lemma}[Derivative of $f_{d,\epsilon}(t)$]
\noindent
We have, uniformly in $t \in [0,1]$:
\begin{align}
\nonumber
\frac{\rd f_{d,\epsilon}(t)}{\rd t} &=-\frac{1}{2}\EE\left\langle\left(\frac{1}{d^2}\sum_{\mu=1}^nu'_{Y_{t,\mu}}(J_{t,\mu})u'_{Y_{t,\mu}}(j_{t,\mu})-r(t)\right)(Q-q(t))\right\rangle\\ 
&+\frac{r(t)}{4}(q(t)-\rho)+o_d(1).
\label{eq:f_d_epsilon_derivative}
\end{align}
\label{lemma:derivative}
\end{lemma}
\begin{proof}[Proof of Lemma~\ref{lemma:derivative}]
By definition, we have
\begin{equation*}
\frac{\rd f_{d,\epsilon}(t)}{\rd t}=-\frac{1}{d^2}\EE[H_{t,\epsilon}'\log\mathcal{Z}_{t,\epsilon}]-\frac{1}{d^2}\EE\langle H_{t,\epsilon}'\rangle.
\end{equation*}
Following \cite[Section A.5.1]{barbier2019optimal}\footnote{We use the Gaussian integration by parts property shown in Lemma~\ref{lemma:stein}.}, we have
\begin{equation*}
\begin{aligned}
\frac{\rd f_{d,\epsilon}(t)}{\rd t}&=\EE\left[\frac{1}{d^2}\frac{P_{\out}''(Y_{t,\mu}|J_{t,\mu})}{P_{\out}(Y_{t,\mu}|J_{t,\mu})}\left(\frac{1}{d}\sum_{i,j=1}^d(S_{ij}^\star )^2-\rho\right)\log\mathcal{Z}_{t,\epsilon}\right]+\frac{r(t)}{4}(q(t)-\rho)\\
&\qquad+\EE\left\langle\left(\frac{1}{d}\sum_{i,j=1}^dS_{ij}^\star s_{ij}-q(t)\right)\left(u_{Y_{t,\mu}}'(J_{t,\mu})u_{Y_{t,\mu}}'(j_{t,\mu})-\frac{r(t)}{4}\right)\right\rangle + \smallO_d(1).
\end{aligned}
\end{equation*}
The $\smallO_d(1)$ term is uniform in $t$.
Now we denote the first term to be $A_{d,\epsilon}$ and we only need to show that it goes to zero uniformly as $d\to\infty$. We have
\begin{equation}
\EE\left[\left.\frac{P_{\out}''(Y_{t,\mu}|J_{t,\mu})}{P_{\out}(Y_{t,\mu}|J_{t,\mu})}\right|S^\star ,J_t\right]=\int \rd Y_{t,\mu}P_{\out}''(Y_{t,\mu}|J_{t,\mu})=0,
\label{eq:Pout_center}
\end{equation}
since $P_\out(Y|J)$ is a probability distribution, and thus by the law of total expectation:
\begin{equation*}
\EE\left[\frac{1}{d^2}\frac{P_{\out}''(Y_{t,\mu}|J_{t,\mu})}{P_{\out}(Y_{t,\mu}|J_{t,\mu})}\left(\frac{1}{d}\sum_{i,j=1}^d(S_{ij}^\star )^2-\rho\right)\right]=0.
\end{equation*}
Consequently we have
\begin{equation}
\begin{aligned}
|A_{d,\epsilon}|&=\left|\EE\left[\sum_{\mu=1}^n\frac{P_{\out}''(Y_{t,\mu}|J_{t,\mu})}{P_{\out}(Y_{t,\mu}|J_{t,\mu})}\left(\frac{1}{d}\sum_{i,j=1}^d(S_{ij}^\star )^2-\rho\right)\left(\frac{1}{d^2}\log\mathcal{Z}_{t,\epsilon}-f_{d,\epsilon}(t)\right)\right]\right|\\
&\leq\EE\left[\left(\sum_{\mu=1}^n\frac{P_{\out}''(Y_{t,\mu}|J_{t,\mu})}{P_{\out}(Y_{t,\mu}|J_{t,\mu})}\right)^2\left(\tr[(S^\star)^2] - \rho\right)^2\right]^{1/2}\EE\left[\left(\frac{1}{d^2}\log\mathcal{Z}_{t,\epsilon}-f_{d,\epsilon}(t)\right)^2\right]^{1/2}.
\label{eq:A_d_epsilon1}
\end{aligned}
\end{equation}
Recall that $\tr(S^2) = (1/d)\Tr[S^2]$.
As (conditionally on $J_t$) $\{\frac{P_{\out}''(Y_{t,\mu}|J_{t,\mu})}{P_{\out}(Y_{t,\mu}|J_{t,\mu})}\}_{\mu=1,2,\cdots,n}$ are i.i.d.\ and centered (see \eqref{eq:Pout_center}) random variables, we have
\begin{equation}
\EE\left[\left(\left.\sum_{\mu=1}^n\frac{P_{\out}''(Y_{t,\mu}|J_{t,\mu})}{P_{\out}(Y_{t,\mu}|J_{t,\mu})}\right)^2\right|J_{t,\mu},S^\star\right]\leq C(\varphi)n,
\label{eq:A_d_epsilon2}
\end{equation}
where we use the boundedness of $\varphi,\varphi',\varphi''$. Moreover, we have
\begin{equation}
\EE\left[\left(\tr[(S^\star)^2] - \rho\right)^2\right] = \smallO(1)
\label{eq:A_d_epsilon3}
\end{equation}
as a consequence of Lemma~\ref{lemma:CLT_new}-$(c)$. Taking eqs. \eqref{eq:A_d_epsilon2} and \eqref{eq:A_d_epsilon3} into eq. \eqref{eq:A_d_epsilon1}, using Lemma \ref{lemma:concentration_free_entropy}, we have
\begin{equation*}
|A_{d,\epsilon}|\leq C(\varphi,M,\alpha,\kappa) \cdot \smallO(1),
\end{equation*}
which goes to zero uniformly in $t$.
\end{proof}
The combination of the following two lemmas give us a lower bound on the free entropy of the original problem.
\begin{lemma}
\noindent
If $q(t)=\EE\langle Q\rangle$, then
\begin{equation*}
\begin{aligned}
f_d&=\frac{1}{\iota_d^2}\int_{B_d} \rd\epsilon\left(\psi_{P_0}\left(\int_0^1r(t)\rd t\right)+\alpha\Psi_{\out}\left(\int_0^1q(t)\rd t\right)-\frac{1}{4}\int_0^1r(t)(\rho-q(t))\rd t\right)\\&\qquad\qquad+\frac{1}{4}+o_d(1).
\end{aligned}
\end{equation*}
\label{lemma:sum_rule}
\end{lemma}
\begin{proof}[Proof of Lemma \ref{lemma:sum_rule}]
We begin from controlling the first term in eq. \eqref{eq:f_d_epsilon_derivative}. We have
\begin{equation*}
\begin{aligned}
&\left(\frac{1}{\iota_d^2}\int_{B_d} \rd\epsilon\int_0^1\rd t\EE\left\langle\left(\frac{1}{d^2}\sum_{\mu=1}^nu'_{Y_{t,\mu}}(J_{t,\mu})u'_{Y_{t,\mu}}(j_{t,\mu})-r(t)\right)(Q-q(t))\right\rangle\right)^2\\
&\leq\left(\frac{1}{\iota_d^2}\int_{B_d} \rd\epsilon\int_0^1\rd t\EE\left\langle\left(\frac{1}{d^2}\sum_{\mu=1}^nu'_{Y_{t,\mu}}(J_{t,\mu})u'_{Y_{t,\mu}}(j_{t,\mu})-r(t)\right)^2\right\rangle\right)\times\\&\qquad\left(\frac{1}{\iota_d^2}\int_{B_d}\rd\epsilon\int_0^1\rd t\EE\left\langle (Q-q(t))^2\right\rangle\right),
\end{aligned}
\end{equation*}
where we use the Cauchy-Schwarz inequality. By \cite[Section A.6]{barbier2019optimal} we have 
\begin{equation*}
\EE\left\langle\left(\frac{1}{d^2}\sum_{\mu=1}^nu'_{Y_{t,\mu}}(J_{t,\mu})u'_{Y_{t,\mu}}(j_{t,\mu})-r(t)\right)^2\right\rangle\leq C(\alpha,\varphi).
\end{equation*}
Then by Lemma \ref{lemma:concentration_Q}, we have
\begin{equation*}
\begin{aligned}
\left|\frac{1}{\iota_d^2}\int_{B_d} \rd\epsilon\int_0^1\rd t\EE\left\langle\left(\frac{1}{d^2}\sum_{\mu=1}^nu'_{Y_{t,\mu}}(J_{t,\mu})u'_{Y_{t,\mu}}(j_{t,\mu})-r(t)\right)(Q-\EE\langle Q\rangle)\right\rangle\right|=o_d(1).
\end{aligned}
\end{equation*}
Using now Lemma \ref{lemma:derivative}, we reach
\begin{equation*}
\frac{1}{\iota_d^2}\int_{B_d} \rd\epsilon\int_0^1\rd t\frac{\rd f_{d,\epsilon}(t)}{\rd t}=\frac{1}{4}\int_0^1r(t)(q(t)-\rho)\rd t+o_d(1).
\end{equation*}
We finish the proof by combining this result with Lemma \ref{lemma:f0_f1}.
\end{proof}

\myskip
We can now state the aforementioned lower bound.
\begin{lemma}
\label{lemma:lowerbound}
\begin{equation*}
\liminf_{d\to\infty}f_d\geq\sup_{r\geq0}\inf_{q\in[0,\rho]}f_{\RS}(q,r).
\end{equation*}
\end{lemma}
\begin{proof}[Proof of Lemma \ref{lemma:lowerbound}]
Following \cite[Proposition 7]{barbier2019optimal}, we can choose $r(t)=r$ and $q(t)=\EE\langle Q\rangle$. Then Lemma \ref{lemma:sum_rule} gives
\begin{equation*}
\begin{aligned}
\liminf_{d\to\infty}f_d=\liminf_{d\to\infty}\frac{1}{\iota_d^2}\int_{B_d} \rd\epsilon f_{\RS}\left(\int_0^1q(t)\rd t,r\right)\geq\inf_{q\in[0,\rho]}f_{\RS}(q,r),
\end{aligned}
\end{equation*}
and thus (since this holds for any $r \geq 0$)
\begin{equation*}
\liminf_{d\to\infty}f_d\geq\sup_{r\in[0,r_{\max}]}\inf_{q\in[0,\rho]}f_{\RS}(q,r).
\end{equation*}
Recall that we chose $r_{\max}=4\alpha\Psi_{\out}'(\rho)\geq4\alpha\Psi_{\out}'(q)$, so for $r\geq r_{\max}$ and $q\in[0,\rho]$ we have $\frac{\partial}{\partial q} f_{\RS}(q,r)=\alpha\Psi'_{\out}(q)-\frac{1}{4}r\leq0$. Therefore for $r>r_{\max}$,
\begin{equation*}
\frac{\partial}{\partial r}\inf_{q\in[0,\rho]}f_{\RS}(q;r)=\frac{\partial}{\partial r}f_{\RS}(\rho;r)=\psi_{P_0}'(r)\leq0,
\end{equation*}
where we use Lemma \ref{lemma:Lipschitz_psi}. Therefore:
\begin{equation*}
\liminf_{d\to\infty}f_d\geq\sup_{r\in[0,r_{\max}]}\inf_{q\in[0,\rho]}f_{\RS}(q,r)=\sup_{r\geq0}\inf_{q\in[0,\rho]}f_{\RS}(q,r).
\end{equation*}
\end{proof}
The following lemma gives the corresponding upper bound.
\begin{lemma}
\label{lemma:upperbound}
\begin{equation*}
\limsup_{d\to\infty}f_d\leq\sup_{r\geq0}\inf_{q\in[0,\rho]}f_{\RS}(q,r).
\end{equation*}
\end{lemma}
\begin{proof}[Proof of Lemma \ref{lemma:upperbound}]
Following \cite[Proposition 8]{barbier2019optimal}, we can choose $q(t)=\EE\langle Q\rangle$ and $r(t)=4\alpha\Psi'_{\out}(q(t))\leq r_{\max}$. As in Lemma \ref{lemma:lowerbound}, we have
\begin{equation*}
\begin{aligned}
\limsup_{d\to\infty}f_d&=\limsup_{d\to\infty}\frac{1}{\iota_d^2}\int_{B_d} \rd\epsilon\left(\psi_{P_0}\left(\int_0^1r(t)\rd t\right)+\alpha\Psi_{\out}\left(\int_0^1q(t)\rd t\right)\right.\\&\left.\qquad\qquad-\frac{1}{4}\int_0^1r(t)(\rho-q(t))\rd t\right)+\frac{1}{4},\\
&\leq\limsup_{d\to\infty}\frac{1}{\iota_d^2}\int_{B_d} \rd\epsilon\int_0^1\rd t f_{\RS}(q(t),r(t)),
\end{aligned}
\end{equation*}
where we use Jensen's inequality because $\psi_{P_0},\Psi_{\out}$ are convex according to Lemmas \ref{lemma:Lipschitz_psi} and \ref{lemma:Psi}.
As $\alpha\Psi_{\out}(q)-\frac{1}{4}r(t)q$ is convex with respect to $q$ by Lemma \ref{lemma:Psi}, and $r(t)=4\alpha\Psi'_{\out}(q(t))$, we have
\begin{equation*}
\alpha\Psi_{\out}(q(t))-\frac{1}{4}r(t)q(t)=\inf_{q\in[0,\rho]}\left[\alpha\Psi_{\out}(q)-\frac{1}{4}r(t)q\right],
\end{equation*}
which gives
\begin{equation*}
\limsup_{d\to\infty}f_d\leq\limsup_{d\to\infty}\frac{1}{\iota_d^2}\int_{B_d} \rd\epsilon\int_0^1\rd t \inf_{q\in[0,\rho]}f_{\RS}(q,r(t))\leq\sup_{r\geq0}\inf_{q\in[0,\rho]}f_{\RS}(q,r).
\end{equation*}
\end{proof}

\myskip
{
\begin{proof}[Proof of Theorem \ref{thm:fentropy_symmetric}]
We note that $\sup_{r\geq0}\inf_{q\in[0,\rho]}f_{\RS}(q,r)=\sup_{q\in[0,\rho]}\inf_{r\geq0}f_{\RS}(q,r)$ by~\cite[Corollary 8]{barbier2019optimal}, combined with our Lemmas~\ref{lemma:Lipschitz_psi} and \ref{lemma:Psi}. Combining Lemmas \ref{lemma:lowerbound} and \ref{lemma:upperbound}, we have
\begin{equation}
\lim_{d\to\infty}f_d=\sup_{q\in[0,\rho]}\inf_{r\geq0}f_{\RS}(q,r),
\label{eq:lim_fd}
\end{equation}
which gives Theorem \ref{thm:fentropy_symmetric} for Gaussian data, under the Assumptions \ref{assum:S_strong} and \ref{assum:phi_strong}.

\myskip
Lemma \ref{lemma:concentration_free_entropy} together with eq.~\eqref{eq:lim_fd} shows that $F_d(G)$ converges in probability to $\sup_{q\in[0,\rho]}\inf_{r\geq0}f_{\RS}(q,r)$ under Assumptions \ref{assum:S_strong} and \ref{assum:phi_strong}. Combining it with Lemma~\ref{lemma:universality}, it is classical to show that $F_d(\Phi)$ also converges in probability (see e.g.\ \cite[Section A.1.3]{montanari2022universality}) to $\sup_{q\in[0,\rho]}\inf_{r\geq0}f_{\RS}(q,r)$. Moreover, we have
\begin{equation*}
    |F_d(\Phi)|\leq\frac{1}{d^2}\sum_{\mu=1}^n\frac{(2\sup|\varphi|+|Z_\mu|)^2}{2\Delta},
\end{equation*}
so $F_d(\Phi)$ is uniformly integrable, which gives
\begin{equation*}
\lim_{d\to\infty}\EE[F_d(\Phi)]=\sup_{q\in[0,\rho]}\inf_{r\geq0}f_{\RS}(q,r).
\end{equation*}
Recall that $f_d\coloneqq\EE[F_d(\Phi)]$, and thus we finish the proof under Assumptions \ref{assum:S_strong} and \ref{assum:phi_strong}. The arguments in Section \ref{sec_app:relax} imply then that one can relax Assumptions \ref{assum:S_strong} and \ref{assum:phi_strong} to Assumptions~\ref{ass:prior} and \ref{assum:phi_weak}.
\end{proof}
}

\subsection{Concentration of the free entropy}
\label{sec:concentration_free_entropy}
In this section we will prove Lemma \ref{lemma:concentration_free_entropy}. Recall that we are using Assumptions \ref{assum:S_strong}, \ref{assum:phi_strong} and Gaussian data. We can rewrite the free entropy as
\begin{equation}\label{eq:Z_Zhat}
\frac{1}{d^2}\log \mathcal{Z}_{t,\epsilon}=\frac{1}{d^2}\log\hat{\mathcal{Z}}_{t,\epsilon}-\frac{1}{2d^2}\sum_{\mu=1}^{n}Z_\mu^2-\frac{1}{4d}\sum_{i,j=1}^{d}(Z_{ij}')^2-\frac{\alpha}{2}\log(2\pi),
\end{equation}
where
\begin{equation*}
\frac{1}{d^2}\log\hat{\mathcal{Z}}_{t,\epsilon}\coloneqq \frac{1}{d^2}\log\int P_0(\rd s)P_A(\rd a)\mcD ue^{-\hat{H}_t(s,u,a)},
\end{equation*}
\begin{equation*}
\begin{aligned}
\hat{H}_t(s,u,a)&\coloneqq\frac{1}{2}\sum_{\mu=1}^n(\Gamma_{t,\mu}(s,u_\mu,a_\mu)^2+2Z_\mu\Gamma_{t,\mu}(s,u_\mu,a_\mu))\\&+\frac{d}{4}\sum_{i,j=1}^{d}(R_1(t)(S_{ij}^\star -s_{ij})^2+2Z_{ij}'\sqrt{R_1(t)}(S_{ij}^\star -s_{ij}))
\end{aligned}
\end{equation*}
and
\begin{equation*}
\begin{aligned}
\Gamma_{t,\mu}(s,u_\mu, a_\mu)&\coloneq\varphi(\sqrt{1-t}\Tr[G_\mu S^\star ]+k_1(t)V_\mu+k_2(t)U_\mu^\star ,A_\mu)\\&-\varphi(\sqrt{1-t}\Tr[G_\mu s]+k_1(t)V_\mu+k_2(t)u_\mu,a_\mu).
\end{aligned}
\end{equation*}
Here $A_\mu\iid P_A$, $k_1(t)\coloneqq \sqrt{2R_2(t)}$ and $k_2(t)\coloneqq \sqrt{2\rho t-2R_2(t)+2\iota_d}$. 
Since all the terms but $(1/d^2) \log \hat{\mcZ}_{t, \epsilon}$ in eq.~\eqref{eq:Z_Zhat} are easily shown to have variance $\mcO(1/d^2)$ -- as they are sums of i.i.d.\ Gaussian variables -- Lemma \ref{lemma:concentration_free_entropy} results from the following lemma.

\begin{lemma}
\begin{equation*}
\text{Var}\left(\frac{1}{d^2}\log\hat{\mathcal{Z}}_{t,\epsilon}\right)\leq\frac{C(\varphi,M,\alpha,\kappa)}{d^2}.
\end{equation*}
\label{lemma:concentration_free_entropy2}
\end{lemma}
\begin{proof}[Proof of Lemma \ref{lemma:concentration_free_entropy2}]
We first consider $g\coloneqq \log\hat{\mathcal{Z}}_{t,\epsilon}/d^2$ as a function of $Z,Z'$. We have
\begin{equation*}
\left|\frac{\partial g}{\partial Z_\mu}\right|=d^{-2}|\langle \Gamma_{t,\mu}\rangle|\leq2d^{-2}\sup|\varphi|,
\end{equation*}
and
\begin{equation*}
\begin{aligned}
\sum_{i,j=1}^{d}\left(\frac{\partial g}{\partial Z_{ij}'}\right)^2&=d^{-2}R_1(t)\sum_{i,j=1}^{d}(S_{ij}^\star -\langle s_{ij}\rangle)^2=d^{-2}R_1(t)\Tr(S^\star -\langle s\rangle)^2\\
&\leq 2 d^{-2}R_1(t)\Tr((S^\star )^2+\langle s\rangle^2)\leq 4d^{-1}KM^2,
\end{aligned}
\end{equation*}
where $K\coloneqq 1+\max(\rho,r_{\max})$ upper bounds $R_1,R_2$ and we use $\Tr(S^\star)^2,\Tr\langle s\rangle^2\leq dM^2$, because under Assumption \ref{assum:S_strong}, the eigenvalues of $S^\star,\langle s\rangle$ are bounded by $M$. By Lemma \ref{lemma:gaussian-poincare} we obtain
\begin{equation}
\EE\left[\left(\frac{1}{d^2}\log\hat{\mathcal{Z}}_{t,\epsilon}-\frac{1}{d^2}\EE_{Z,Z'}\log\hat{\mathcal{Z}}_{t,\epsilon}\right)^2\right]\leq\frac{C(\varphi,M,\alpha)}{d^2}.
\label{eq:concentration1}
\end{equation}
Following \cite[Lemma 27]{barbier2019optimal}, we also have
\begin{equation*}
\EE\left[\EE_{Z,Z'}\left(\frac{1}{d^2}\log\hat{\mathcal{Z}}_{t,\epsilon}-\frac{1}{d^2}\EE_{Z,Z',A}\log\hat{\mathcal{Z}}_{t,\epsilon}\right)^2\right]\leq\frac{C(\varphi,M,\alpha)}{d^2}
\end{equation*}
by the bounded difference inequality w.r.t. $A = \{A_\mu\}_{\mu=1}^n$. Then we consider $g\coloneqq \EE_{Z,Z',A}\log\hat{\mathcal{Z}}_{t,\epsilon}/d^2$ as a function of $V,U^\star ,G$. We have
\begin{equation*}
\left|\frac{\partial g}{\partial V_\mu}\right|=d^{-2}\left|\EE_{Z,Z',A}\left\langle (\Gamma_{t,\mu}+Z_\mu)\frac{\partial\Gamma_{t,\mu}}{\partial V_\mu}\right\rangle\right|\leq d^{-2}\left(2\sup|\varphi|+\sqrt{\frac{2}{\pi}}\right)2\sqrt{K}\sup|\varphi'|.
\end{equation*}
The same inequality holds for $|\frac{\partial g}{\partial U^\star _i}|$. Moreover, we have
\begin{equation*}
\begin{aligned}
\sum_{i,j=1}^d\left(\frac{\partial g}{\partial(G_{\mu, ij})}\right)^2&=d^{-5}\sum_{i,j=1}^d\left(\EE_{Z,Z',A}\left\langle(\Gamma_{t,\mu}+Z_\mu)\frac{\partial\Gamma_{t,\mu}}{\partial(G_{\mu,ij})}\right\rangle\right)^2\\
&\leq d^{-4}\EE_{Z,Z',A}\sum_{i,j=1}^d\left\langle(\Gamma_{t,\mu}+Z_\mu)\frac{\partial\Gamma_{t,\mu}}{\partial(G_{\mu,ij})}\right\rangle^2\\&\leq d^{-4}\EE_{Z,Z',A}(2\sup|\varphi|+|Z_\mu|)^2\sum_{i,j=1}^d\langle|S_{ij}^\star \varphi'(J_{t,\mu})-s_{ij}\varphi'(j_{t,\mu})|\rangle^2\\
&\leq 2d^{-4}\EE_{Z,Z',A}(2\sup|\varphi|+|Z_\mu|)^2|\sup|\varphi'|^2\sum_{i,j=1}^d((S_{ij}^\star )^2+\langle s_{ij}^2\rangle)\\
&\leq 4d^{-3}\left(2\sup|\varphi|+\sqrt{\frac{2}{\pi}}\right)^2M^2\sup|\varphi'|^2,
\end{aligned}
\end{equation*}
where the first inequality follows from Jensen's inequality. By Lemma \ref{lemma:gaussian-poincare} we have
\begin{equation}
\EE\left[\left(\frac{1}{d^2}\EE_{Z,Z',A}\log\hat{\mathcal{Z}}_{t,\epsilon}-\frac{1}{d^2}\EE_{\Theta}\log\hat{\mathcal{Z}}_{t,\epsilon}\right)^2\right]\leq\frac{C(\varphi,M,\alpha)}{d^2},
\label{eq:concentration2}
\end{equation}
where $\Theta$ represents all quenched variables except $S^\star $.
Next we write $S^\star =O\Lambda^\star O^T$ with $O$ drawn from the uniform (Haar) measure on the orthogonal group $\mcO(d)$, and $\Lambda^\star$ the diagonal matrix of eigenvalues of $S^\star$. 
We denote $\tilde{G}_\mu\coloneqq O^TG_\mu O$, $\tilde{s}\coloneqq O^TsO$ and $\tilde{Z}'=O^T\tilde{Z}'O$. 
In this way the Hamiltonian reads
\begin{equation}\label{eq:Hhat}
\begin{aligned}
\hat{H}_t(\tilde{s},u,a)&\coloneqq\frac{1}{2}\sum_{\mu=1}^n(\Gamma_{t,\mu}(\tilde{s},u_\mu,a_\mu)^2+2Z_\mu\Gamma_{t,\mu}(\tilde{s},u_\mu,a_\mu))\\&+\frac{1}{4}\sum_{i,j=1}^{d}(dR_1(t)(\Lambda_{ij}^\star -\tilde{s}_{ij})^2+2\tilde{Z}_{ij}'\sqrt{dR_1(t)}(\Lambda_{ij}^\star -\tilde{s}_{ij})),
\end{aligned}
\end{equation}
where
\begin{equation*}
\Gamma_{t,\mu}(\tilde{s},u_\mu)\coloneqq \varphi(\sqrt{1-t}\Tr[\tilde{G}_\mu \Lambda^\star ]+k_1(t)V_\mu+k_2(t)U_\mu^\star,A_\mu)-\varphi(\sqrt{1-t}\Tr[\tilde{G}_\mu \tilde{s}]+k_1(t)V_\mu+k_2(t)u_\mu^\star,a_\mu).
\end{equation*}
Note that $\tilde{G}_\mu$ and $\tilde{Z}'$ are still independent $\GOE(d)$ matrices, and the distribution of $\tilde{s}$ is the same as that of $s$, so we have
\begin{equation}\label{eq:tildeZ_independent_O}
\EE_{O,\Theta}\frac{1}{d^2}\log\hat{\mathcal{Z}}_{t,\epsilon}=\EE_{\tilde{G},\tilde{Z}',Z,A,V,U^\star}\frac{1}{d^2}\log\int P_A(\rd a)P_0(\rd\tilde{s})\mcD ue^{-\hat{H}_t(\tilde{s},u,a)},
\end{equation}
where both sides are only a function of $\Lambda^\star$, which gives
\begin{equation}
\begin{aligned}
&\EE\left[\left(\frac{1}{d^2}\EE_{\Theta}\log\hat{\mathcal{Z}}_{t,\epsilon}-\frac{1}{d^2}\EE_{O,\Theta}\log\hat{\mathcal{Z}}_{t,\epsilon}\right)^2\right]\\&\leq\EE\left[\frac{1}{d^2}\left(\EE_{\tilde{G},\tilde{Z}',Z,A,V,U^\star}\log\int P_A(\rd a)P_0(\rd\tilde{s})\mcD ue^{-\hat{H}_t(\tilde{s},u,a)}-\frac{1}{d^2}\log\int P_A(\rd a)P_0(\rd\tilde{s})\mcD ue^{-\hat{H}_t(\tilde{s},u,a)}\right)^2\right]\\
&\leq\frac{C(\varphi,M,\alpha)}{d^2}.
\end{aligned}
\label{eq:concentration3}
\end{equation}
The last inequality is from our previous calculation. Finally we consider $g\coloneqq \EE_{O,\Theta}\frac{1}{d^2}\log\hat{\mathcal{Z}}_{t,\epsilon}$ as a function of $\Lambda^\star $. We have (see eq.~\eqref{eq:Hhat}):
\begin{equation*}
\begin{aligned}
\sum_{i=1}^d\left(\frac{\partial g}{\partial\Lambda^\star _i}\right)^2&=d^{-4}\sum_{i=1}^d\left(\EE_{O,\Theta}\left\langle\frac{\partial\hat{H}_t}{\partial\Lambda_i}\right\rangle\right)^2
\\
&=d^{-4}\sum_{i=1}^d\left(\EE_{O,\Theta}\left\langle\sum_{\mu=1}^n(\Gamma_{t,\mu}+Z_\mu)\frac{\partial\Gamma_{t,\mu}}{\partial(\Lambda^\star _i)}+\frac{dR_1(t)}{2}(\Lambda^\star _i-\tilde{s}_{ii})+2\tilde{Z}_{ii}'\sqrt{dR_1(t)}\right\rangle\right)^2, 
\\
&\leq 3(I_1+I_2+I_3),
\end{aligned}
\end{equation*}
where we use 
\begin{equation}
\left(\sum_{i=1}^pv_i\right)^2\leq p\sum_{i=1}^pv_i^2\label{eq:inequality_sum_p}
\end{equation}
for any integer $p\geq1$. The first term is
\begin{align*}
I_1& \coloneqq d^{-4}\sum_{i=1}^d\left(\EE_{O,\Theta}\left\langle\sum_{\mu=1}^n(\Gamma_{t,\mu}+Z_\mu)\frac{\partial\Gamma_{t,\mu}}{\partial(\Lambda^\star _i)}\right\rangle\right)^2\\&=d^{-4}\sum_{i=1}^d\left(\EE_{O,\Theta}\left\langle\sum_{\mu=1}^n(\Gamma_{t,\mu}+Z_\mu)\sqrt{1-t}\tilde{G}_{\mu,ii}\varphi'(J_{t,\mu})\right\rangle\right)^2\\
&\overset{(a)}{\leq} 16d^{-6}\sum_{i=1}^d\left(\EE_{O,\Theta}\left\langle \sum_{\mu=1}^n\varphi'(J_{t,\mu})(\Lambda_i^\star \varphi'(J_{t,\mu})-\tilde{s}_{ii}\varphi'(j_{t,\mu}))\right\rangle\right)^2+\left(
\EE_{O,\Theta}\left\langle \Lambda_i^\star (\Gamma_{t,\mu}+Z_\mu)\varphi''(J_{t,\mu})\right\rangle\right)^2\\
&\qquad+\left(\EE_{O,\Theta}\langle(\Gamma_{t,\mu}+Z_\mu)^2\varphi'(J_{t,\mu})(\Lambda_i^\star \varphi'(J_{t,\mu})-\tilde{s}_{ii}\varphi'(j_{t,\mu})\rangle\right)^2\\&\qquad+\left(\EE_{O,\Theta}\langle(\Gamma_{t,\mu}+Z_\mu)\varphi'(J_{t,\mu})\rangle\langle(\Gamma_{t,\mu}+Z_\mu)(\Lambda_i^\star \varphi'(J_{t,\mu})-\tilde{s}_{ii}\varphi'(j_{t,\mu})\rangle\right)^2\\
&\leq 16d^{-2}\sum_{i=1}^d2\sup|\varphi'|^2\EE_{O,\Theta}((\Lambda_i^\star )^2+\langle\tilde{s}_{ii}\rangle^2)+\sup|\varphi''|^2(\EE_{O,\Theta}(\sup|\varphi|+|Z_\mu|)\Lambda_i^\star )^2\\&\qquad+4\sup|\varphi'|^2(\EE_{O,\Theta}(\sup|\varphi|+|Z_\mu|)^2(\Lambda_i^\star )^2+\langle\tilde{s}_{ii}\rangle^2))^2\\
&\overset{(b)}{\leq}\frac{C(\varphi,M,\alpha)}{d}.
\end{align*}
In (a) we use Gaussian integration by parts w.r.t. $\tilde{G}_{\mu,ii}$, the fact that $J_{t,\mu}=\varphi(\sqrt{1-t}\Tr[\tilde{G}_\mu \Lambda^\star ]+k_1(t)V_\mu+k_2(t)U_\mu^\star$, and eq. \eqref{eq:inequality_sum_p}.
In (b) we use $\sum_{i=1}^d\langle \ts_{ii}\rangle^2 =\Tr\langle \ts\rangle^2\leq dM^2$ by the boundedness of its eigenvalues for the last inequality. The second term is
\begin{equation}\label{eq:bound_I2}
\begin{aligned}
I_2&\coloneqq d^{-4}\sum_{i=1}^d\left(\EE_{O,\Theta}\left\langle\frac{dR_1(t)}{2}(\Lambda^\star _i-\tilde{s}_{ii})\right\rangle\right)^2\\
&\leq\frac{d^{-2}K^2}{2}\sum_{i=1}^d\EE_{O,\Theta}\left[(\Lambda^\star _i)^2+\langle\tilde{s}_{ii}\rangle^2\right]\leq d^{-1}K^2M^2,
\end{aligned}
\end{equation}
where we use Jensen's inequality for the first inequality. Finally, the third term is
\begin{equation}\label{eq:bound_I3}
I_3\coloneqq d^{-4}\sum_{i=1}^d\left(\EE_{O,\Theta}\left\langle2\tilde{Z}_{ii}'\sqrt{dR_1(t)}\right\rangle\right)^2=d^{-3}R_1(t)\sum_{i=1}^d(\EE_{O,\Theta}|\tilde{Z}_{ii}'|)^2\leq\frac{C(\varphi,M,\alpha)}{d^2}.
\end{equation}
Combining these three terms, we have
\begin{equation}
\EE\left[\left(\frac{1}{d^2}\EE_\Theta\log\hat{\mathcal{Z}}_{t,\epsilon}-\frac{1}{d^2}\EE\log\hat{\mathcal{Z}}_{t,\epsilon}\right)^2\right]\leq\frac{C(\varphi,M,\alpha,\kappa)}{d^2}
\label{eq:concentration4}
\end{equation}
by using the Poincar\'e inequality for rotationally-invariant priors (Lemma \ref{lemma:poincare} in                 Appendix~\ref{sec_app:technical_lemmas}).
We finish the proof of Lemma~\ref{lemma:concentration_free_entropy2} by combining eqs. \eqref{eq:concentration1}, \eqref{eq:concentration2}, \eqref{eq:concentration3} and \eqref{eq:concentration4}.
\end{proof}

\subsection{Concentration of the overlap}
\label{sec:overlap_concentration}
In this section we prove Lemma \ref{lemma:concentration_Q}. Recall that we are using Assumptions \ref{assum:S_strong}, \ref{assum:phi_strong} and Gaussian data. Concentration of the overlap results from the concentration of
\begin{equation*}
\mathcal{L}\coloneqq \frac{1}{d^2}\frac{\rd H_{t,\epsilon}}{\rd R_1}=\frac{1}{2d}\sum_{i,j=1}^d\left(\frac{s_{ij}^2}{2}-s_{ij}S_{ij}^\star -\frac{s_{ij}Z_{ij}'}{2\sqrt{R_1}}\right)
\end{equation*}
because, following \cite[Section 6]{barbier2019adaptive}, we have
\begin{equation*}
\begin{aligned}
\EE\langle (\mathcal{L}-\EE\langle\mathcal{L}\rangle)^2\rangle&=\frac{1}{16d^2}\sum_{i,j,k,l=1}^d(\EE[\langle s_{ij}s_{kl}\rangle^2]-\EE[\langle s_{ij}\rangle^2]\EE[\langle s_{kl}\rangle^2])\\&+\frac{1}{8d^2}\sum_{i,j,k,l=1}^d(\EE[\langle s_{ij}s_{kl}\rangle^2]-\EE[\langle s_{ij}s_{kl}\rangle\langle s_{ij}\rangle\langle s_{kl}\rangle])+\frac{1}{16d^2R_1}\sum_{i,j=1}^d\EE[\langle s_{ij}^2\rangle],
\end{aligned}
\end{equation*}
which leads to
\begin{equation}
\begin{aligned}
\EE\langle(Q-\EE\langle Q)^2\rangle\leq16\EE\langle (\mathcal{L}-\EE\langle\mathcal{L}\rangle)^2\rangle,
\label{eq:L&Q}
\end{aligned}
\end{equation}
where we recall that $Q\coloneqq \frac{1}{d}\sum_{i,j=1}^ds_{ij}S_{ij}^\star $. 
We refer to~\cite[Section~6]{barbier2019adaptive} for the proof of eq.~\eqref{eq:L&Q}.
Then we can bound the fluctuation of the overlap by the fluctuation of $\mathcal{L}$. The following lemma is a direct transposition of \cite[Lemma 29]{barbier2019optimal} to our setting: we refer to~\cite{barbier2019optimal} for its proof, as it does not use the prior structure and thus directly applies in our context.
\begin{lemma}
    \noindent
    For any $\iota_d \leq 1/2$:
\begin{equation*}
\int_{B_d}\rd\epsilon \, \EE\langle(\mathcal{L}-\langle\mathcal{L}\rangle)^2\rangle\leq\frac{\rho(1+\rho)}{d^2}.
\end{equation*}
\label{lemma:concentration_L1}
\end{lemma}
We now prove:
\begin{lemma}
\begin{equation*}
\int_{B_d}\rd\epsilon\EE[(\langle\mathcal{L}\rangle-\EE\langle\mathcal{L}\rangle)^2]\leq C(\varphi,M,\alpha,\kappa) \cdot \smallO_d(1).
\end{equation*}
\label{lemma:concentration_L2}
\end{lemma}
\begin{proof}[Proof of Lemma \ref{lemma:concentration_L2}]
$\mathcal{L}$ is connected to the free entropy through
\begin{equation}\label{eq:dF_dR1}
\frac{\rd F_{d,\epsilon}(t)}{\rd R_1}=-\langle\mathcal{L}\rangle-\frac{1}{4d}\sum_{i,j=1}^d((S_{ij}^\star )^2+\frac{1}{\sqrt{R_1}}S_{ij}^\star Z_{ij}'),
\end{equation}
\begin{equation*}
\frac{1}{d^2}\frac{\rd^2F_{d,\epsilon}(t)}{\rd R_1^2}=\langle\mathcal{L}^2\rangle-\langle\mathcal{L}\rangle^2-\frac{1}{8d^3R_1^{3/2}}\sum_{i,j=1}^d\langle s_{ij}\rangle Z_{ij}',
\end{equation*}
where $F_{d,\epsilon}(t)\coloneqq \frac{1}{d^2}\log\mathcal{Z}_{t,\epsilon}(Y,Y',G,V)$. Now we suppose that the eigenvalues of $Z'$ are $\{\Lambda_{Z,i}\}_{i=1}^d$, and further define 
\begin{equation*}
\tilde{F}(R_1)\coloneqq F_{d,\epsilon}(t)-\frac{\sqrt{R_1}}{2d}M\sum_{i=1}^d|\Lambda_{Z,i}|,\ \tilde{f}(R_1)\coloneqq \EE\tilde{F}(R_1)=f_{d,\epsilon}(t)-\frac{\sqrt{R_1}}{2d}M\sum_{i=1}^d\EE|\Lambda_{Z,i}|,
\end{equation*}
where we recall that $f_{d,\epsilon}(t)=\EE F_{d,\epsilon}(t)$.
$\tilde{F}(R_1)$ is a convex function because
\begin{equation*}
\frac{1}{d^2}\frac{\rd^2\tilde{F}(t)}{\rd R_1^2}=\langle\mathcal{L}^2\rangle-\langle\mathcal{L}\rangle^2+\frac{1}{8dR_1^{3/2}}\left(\sum_{i=1}^dM|\Lambda_{Z,i}|-\sum_{i,j=1}^d\langle s_{ij}\rangle Z_{ij}'\right)\geq0,
\end{equation*}
where the last inequality is a consequence of $|\Tr[AB]|\leq \|A\|_\op \sum_{i=1}^d|\lambda_{B,i}|$. Moreover, $\tilde{f}(R_1)$ is also convex because $f_{d,\epsilon}(t)$ is convex in $R_1$ (its derivative is the MMSE, which decreases with increasing $R_1$).

\myskip
By definition and eq.~\eqref{eq:dF_dR1}, we have
\begin{equation*}
\tilde{F}'(R_1)-\tilde{f}'(R_1)=\EE\langle\mathcal{L}\rangle-\langle\mathcal{L}\rangle+\frac{\rho}{4}-\frac{1}{4d}\sum_{i,j=1}^d((S_{ij}^\star )^2+\frac{1}{\sqrt{R_1}}S_{ij}^\star Z_{ij}')-\frac{1}{4\sqrt{R_1}}MA,
\end{equation*}
where
\begin{equation*}
A\coloneqq \frac{1}{d}\sum_{i=1}^d(|\Lambda_{Z,i}|-\EE|\Lambda_{Z,i}|)
\end{equation*}
satisfies $\EE[A^2]\leq ad^{-2}$ (by repeating the proof arguments of Lemma \ref{lemma:CLT_new}-$(b)$, i.e.\ directly using the Poincaré inequality of Lemma~\ref{lemma:poincare} for the $\GOE(d)$ distribution). 

\myskip
Due to the convexity of $\tilde{F}$ and $\tilde{f}$, we have \cite[Lemma 31]{barbier2019optimal}
\begin{equation*}
\begin{aligned}
|\langle\mathcal{L}\rangle-\EE\langle\mathcal{L}\rangle| &\leq \delta^{-1} \sum_{u \in\left\{R_1-\delta, R_1, R_1+\delta\right\}}\left(\left|F_{d, \epsilon}\left(t, R_1=u\right)-f_{d, \epsilon}\left(t, R_1=u\right)\right|+\frac{1}{2}M|A| \sqrt{u}\right) \\
&+C_\delta^{+}\left(R_1\right)+C_\delta^{-}\left(R_1\right)+\frac{M|A|}{4\sqrt{\epsilon_1}}+\left|\frac{\rho}{4}-\frac{1}{4d} \sum_{i,j=1}^d\left(\left(S_{ij}^\star \right)^2+\frac{1}{\sqrt{R_1}} S_{ij}^\star Z_{ij}'\right)\right|,
\end{aligned}
\end{equation*}
where $C_\delta(R_1)^+\coloneqq \tilde{f}'(R_1+\delta)-\tilde{f}'(R_1)\geq0$, $C_\delta(R_1)^-\coloneqq \tilde{f}'(R_1)-\tilde{f}'(R_1-\delta)\geq0$. Recall that $R_1 \geq \eps_1$. Denote the last term as $|B|$. We then have
\begin{equation*}
B\coloneqq \frac{\rho}{4}-\frac{1}{4d} \sum_{i=1}^d\left(\left(\Lambda_{i}^\star \right)^2+\frac{1}{\sqrt{R_1}} \Lambda_{ii}^\star \tilde{Z}_{ii}'\right),
\end{equation*}
where we denote $S^\star \coloneqq O\Lambda^\star O^T$ and $\tilde{Z}'\coloneqq O^TZ'O$. $\Lambda^\star $ is independent of $O$, so $\tilde{Z}'\sim\text{GOE}(d)$ is independent of $\Lambda^\star $. 
By Lemma~\ref{lemma:CLT_new} we have:
\begin{equation*}
\EE[B^2]\leq\frac{1}{8}\EE\left[\left(\rho-\frac{1}{d}\sum_{i=1}^d(\Lambda_i^\star )^2\right)^2\right]+\frac{1}{8R_1}\EE\left[\left(\frac{1}{d}\sum_{i=1}^d\Lambda_{ii}^\star \tilde{Z}_{ii}'\right)^2\right] 
\leq \smallO(1) + b\iota_d^{-1}d^{-1},
\end{equation*}
for a constant $b$. Then we have
\begin{equation*}
\frac{1}{10}\EE[(\langle\mathcal{L}\rangle-\EE\langle\mathcal{L}\rangle)^2]\leq 3\delta^{-2}(C(\varphi,M,\alpha,\kappa)+aM^2d)d^{-2}+C_\delta^+(R_1)^2+C_\delta^-(R_1)^2+\frac{M^2a}{4\epsilon_1d^2}+\frac{b}{\iota_dd} + \smallO(1),
\end{equation*}
where we use Lemma \ref{lemma:concentration_free_entropy2}. As 
\begin{equation*}
\left|\frac{\rd f_{d,\epsilon}(t)}{\rd R_1}\right|=\left|\frac{1}{4d}\sum_{i,j=1}^d\EE[\langle s_{ij}\rangle^2]-\frac{\rho}{4}\right|\leq\frac{\rho}{2}
\end{equation*}
and $R_1\geq\epsilon_1$, we have
\begin{equation*}
|\tilde{f}'(R_1)|\leq\frac{1}{2}(\rho+\frac{M}{\sqrt{\epsilon_1}}).
\end{equation*}
Following steps in \cite[Lemma 30]{barbier2019optimal}, we have
\begin{equation*}
\int_{B_d}d\epsilon(C_\delta^+(R_1)^2+C_\delta^-(R_1)^2)\leq\frac{1}{2}\delta(\iota_d+\rho)\left(\rho+\frac{M}{\sqrt{\iota_d}}\right)^2,
\end{equation*}
and thus
\begin{equation*}
\begin{aligned}
\EE[(\langle\mathcal{L}\rangle-\EE\langle\mathcal{L}\rangle)^2]&\leq30\delta^{-2}(C(\varphi,M,\alpha,\kappa)+aM^2d)d^{-2}+5\delta(\iota_d+\rho)\left(\rho+\frac{M}{\sqrt{\iota_d}}\right)^2\\&\qquad+5M^2a\frac{\log2}{2}\frac{\iota_d}{d}+\frac{b}{\iota_dd} + \smallO(1).
\end{aligned}
\end{equation*}
We now finish the proof by choosing $\delta=\iota_dd^{-1/4}$ (and $\iota_d=\frac{1}{2}d^{-1/8}$).
\end{proof}
Finally, we prove Lemma \ref{lemma:concentration_Q} by combining eq. \eqref{eq:L&Q}, Lemma \ref{lemma:concentration_L1} and Lemma \ref{lemma:concentration_L2}.

\subsection{Relaxation of Assumptions \ref{assum:S_strong} and \ref{assum:phi_strong}}
\label{sec_app:relax}

\begin{lemma}
\noindent
Denote $f_d(\mathcal{V})$ to be the free entropy of eq.~\eqref{eq:def_fentropy} corresponding to the prior $P_0$ with potential $\mathcal{V}$, and denote $\tilde{\mathcal{V}}_M$ the truncation of $\mathcal{V}$ to $[-M, M]$, with $\tilde{\mathcal{V}}_M(x) = +\infty$ if $|x| > M$. 
Notice that then the prior with potential $\tilde{\mathcal{V}}_M$ satisfies  Assumption \ref{assum:S_strong}. Suppose that $\mathcal{V}$ satisfies Assumption \ref{ass:prior} and $\varphi$ satisfies Assumption \ref{assum:phi_strong}. Then, for $M > 0$ large enough, we have
\begin{equation*}
    \lim_{d\to\infty}|f_d(\mathcal{V})-f_d(\tilde{\mathcal{V}}_M)|=0.
\end{equation*}
\label{lemma:relaxV}
\end{lemma}
\begin{proof}[Proof of Lemma~\ref{lemma:relaxV}]
We consider the following interpolation
\begin{equation}
\label{eq:potential_intepolation}
\begin{aligned}
&Y_{1,\mu}=\sqrt{t}\varphi(\Tr[\Phi_\mu S_1^\star],A_\mu)+\sqrt{\Delta}Z_{1,\mu},\\
&Y_{2,\mu}=\sqrt{1-t}\varphi(\Tr[\Phi_\mu S_2^\star],A_\mu)+\sqrt{\Delta}Z_{2,\mu},
\end{aligned}
\end{equation}
where $\{A_\mu\}_{\mu=1}^n\iid P_A,\{Z_{1,\mu}\}_{\mu=1}^n,\{Z_{2,\mu}\}_{\mu=1}^n\overset{iid}{\sim}\mathcal{N}(0,1)$ are independent of $S_1^\star,S_2^\star$. $S_1^\star$ is sampled from the prior with potential $V$ and $S_2^\star$ is sampled from the prior with potential $\tV_M$, and are coupled as follows. The joint distribution of $S_1^\star,S_2^\star$ is given by $S_1^\star\coloneqq O\Lambda_1^\star O^T,S_2^\star\coloneqq O\Lambda_2^\star O^T$, where $O$ is Haar distributed, and the coupling between $\Lambda_1^\star$ and $\Lambda_2^\star$ remains arbitrary for now.

\myskip
We can define the mutual information $I(t)\coloneqq I(S_1^\star,S_2^\star;Y_1,Y_2 | \Phi)$.
Following \cite[Proposition 14]{barbier2019optimal} and using the I-MMSE theorem \cite[Lemma 4.5]{barbier2020mutual}, we have
\begin{equation*}
I'(t)=\frac{1}{2}\alpha\EE\left[\sum_{\mu=1}^n(\varphi(\Tr[\Phi_\mu S_1^\star],A_\mu)-\varphi(\Tr[\Phi_\mu\langle s_1\rangle],A_\mu))^2-(\varphi(\Tr[\Phi_\mu S_2^\star],A_\mu)-\varphi(\Tr[\Phi_\mu\langle s_2\rangle],A_\mu))^2\right],
\end{equation*}
where the Gibbs bracket $\langle\cdot\rangle$ denotes the expectation with respect to the posterior distribution $\bbP(s_1, s_2 | Y_1, Y_2, \Phi)$ of the model (eq. \eqref{eq:potential_intepolation}).

\myskip
For simplicity, we denote $z_{1,\mu}\coloneqq\Tr[\Phi_\mu S_1^\star],z_{2,\mu}\coloneqq\Tr[\Phi_\mu S_2^\star],e_{1,\mu}\coloneqq\Tr[\Phi_\mu\langle s_1\rangle],e_{2,\mu}\coloneqq\Tr[\Phi_\mu\langle s_2\rangle]$ and thus
\begin{equation*}
\begin{aligned}
|I'(t)|&=\frac{\alpha}{2}|\EE[||\varphi(z_1,A)-\varphi(e_1,A)||^2-||\varphi(z_2,A)-\varphi(e_2,A)||^2]|\\
&=\frac{\alpha}{2}\EE[(||\varphi(z_1,A)-\varphi(e_1,A)||+||\varphi(z_2,A)-\varphi(e_2,A)||)(||\varphi(z_1,A)-\varphi(e_1,A)||-||\varphi(z_2,A)-\varphi(e_2,A)||)]\\
&\leq2\alpha\sqrt{n}\sup|\varphi|\EE[||\varphi(z_1,A)-\varphi(e_1,A)-\varphi(z_2,A)+\varphi(e_2,A)||]\\
&\leq2\alpha\sqrt{n}\sup|\varphi|\EE[||\varphi(z_1,A)-\varphi(z_2,A)||+||\varphi(e_1,A)-\varphi(e_2,A)||]\\
&\overset{(a)}{=}4\alpha\sqrt{n}\sup|\varphi|\EE[||\varphi(z_1,A)-\varphi(z_2,A)||]\\
&\leq 4\alpha\sqrt{n}\sup|\varphi|\sup|\varphi'|\EE[||z_1-z_2||^2]^{1/2}\\
&\overset{(b)}{=}4\alpha \sqrt{nd}\sup|\varphi|\sup|\varphi'|\EE[||\Lambda_1^\star-\Lambda_2^\star||^2]^{1/2}.
\end{aligned}
\end{equation*}
We use the Nishimori identity in equation (a), see Proposition~\ref{prop:nishimori}, and Assumption \ref{assum:CLT} in equation (b). We also used the triangle inequality, the mean value inequality, and Jensen's inequality.

\myskip
Recall that
the free entropy is related to the mutual information as (see eq.~\eqref{eq:mutual_inf_fentropy}):
\begin{equation*}
\frac{1}{d^2}I(1)=-f_d(\mathcal{V})+\alpha\Psi_{{\out}}(\rho_{S_1})+o_d(1),\ \frac{1}{d^2}I(0)=-f_d(\tilde {\mathcal{V}}_M)+\alpha\Psi_{\out}(\rho_{S_2})+o_d(1),
\end{equation*}
where we denote $\rho_{S_1}\coloneqq\lim_{d\to\infty}\frac{1}{d}\sum_{i=1}^d(\Lambda_1^\star)^2$ and $\rho_{S_2}\coloneqq\lim_{d\to\infty}\frac{1}{d}\sum_{i=1}^d(\Lambda_2^\star)^2$. This gives
\begin{equation}
|f_d(\mathcal{V})-f_d(\tilde{\mathcal{V}}_M)|\leq4\alpha \sqrt{\alpha}\sup|\varphi|\sup|\varphi'|\EE[W_2(\mu_{S_1},\mu_{S_2})^2]^{1/2}+\alpha|\Psi_{\out}(\rho_{S_1})-\Psi_{\out}(\rho_{S_2})|+o_d(1),
\label{eq:fd-fd'}
\end{equation}
where $W_2(\mu_{S_1},\mu_{S_2})\coloneqq\sqrt{\inf\frac{1}{d}||\Lambda_1^\star-\Lambda_2^\star||^2}$ denotes the Wasserstein-2 distance between the empirical distributions of $\Lambda_1^\star$ and $\Lambda_2^\star$. Note that the infimum is w.r.t. the coupling between the random variables $\Lambda_1^\star$ and $\Lambda_2^\star$. 
Given that our analysis was made for an arbitrary choice of this coupling, we chose the coupling between $\Lambda_1^\star$ and $\Lambda_2^\star$ such that $\frac{1}{d}||\Lambda_1^\star-\Lambda_2^\star||^2$ is $W_2(\mu_{S_1},\mu_{S_2})^2+\epsilon$, before taking $\epsilon\to0$: this gave rise to eq.~\eqref{eq:fd-fd'}

\myskip
Lemma~\ref{lemma:truncation} in Appendix~\ref{sec_app:technical_lemmas} shows that $\mu_{S_1},\mu_{S_2}$ both weakly converges to $\mu_0$ almost surely for $M > 0$ large enough. Thus from Proposition \ref{prop:properties_P0_sym}(iv) we have $\rho_{S_1} = \rho_{S_2}$, $\lim_{d\to\infty}W_2(\mu_{S_1},\mu_0)=0,a.s.$ and $\lim_{d\to\infty}W_2(\mu_{S_2},\mu_0)=0,a.s.$, which gives
\begin{equation*}
\lim_{d\to\infty}W_2(\mu_{S_1},\mu_{S_2})\leq\lim_{d\to\infty}W_2(\mu_{S_1},\mu_0)+\lim_{d\to\infty}W_2(\mu_{S_2},\mu_0)=0,a.s.
\end{equation*}
Moreover, by definition we have
\begin{equation*}
W_2(\mu_{S_1},\mu_{S_2})^2\leq2\left(\int\mu_{S_1}(\rd x)x^2+\int\mu_{S_2}(\rd x)x^2\right).
\end{equation*}
By Proposition \ref{prop:properties_P0_sym}(iii), both terms are uniformly integrable, so we have
\begin{equation*}
\lim_{d\to\infty}\EE[W_2(\mu_{S_1},\mu_{S_2})^2]=0
\end{equation*}
by using dominated convergence. 
Finally, plugging all these results into eq. \eqref{eq:fd-fd'}, we get
\begin{equation*}
\lim_{d\to\infty}|f_d(\mathcal{V})-f_d(\tilde{\mathcal{V}}_M)|=0,
\end{equation*}
which finishes the proof.
\end{proof}

\myskip
The following lemma states that one can relax Assumption~\ref{assum:phi_strong}. 
It directly follows from \cite[Proposition 25]{barbier2019optimal}, plugging in the central limit theorem result of Lemma \ref{lemma:CLT_new}.
\begin{lemma}
\noindent
Suppose that the potential satisfies Assumption \ref{ass:prior} and $\varphi$ satisfies Assumption \ref{assum:phi_weak}. Then for all $\varepsilon>0$, there exists a $\hat{\varphi}$ satisfying Assumption \ref{assum:phi_strong} such that $|f_d(\varphi)-f_d(\hat{\varphi})|<\varepsilon$ for $d$ large enough, where $f_d(\varphi)$ denotes the free entropy corresponding to $\varphi$.
\label{lemma:relaxphi}
\end{lemma}
Lemmas \ref{lemma:relaxV} and \ref{lemma:relaxphi} show that for any $P_0$ satisfying Assumption \ref{ass:prior} and $\varphi$ satisfying Assumption \ref{assum:phi_weak}, there exists $\hat{P}_0$ satisfying Assumption \ref{assum:S_strong} and $\hat{\varphi}$ satisfying Assumption \ref{assum:phi_strong}, such that the free entropies for $(P_0, \varphi)$ and $(\hat{P}_0, \hat{\varphi})$ are arbitrarily close as $d \to \infty$: this allows us to relax Assumptions \ref{assum:S_strong} and \ref{assum:phi_strong} to Assumptions~\ref{ass:prior} and \ref{assum:phi_weak} in the proof of Theorem~\ref{thm:fentropy_symmetric}.

\section{Proof of Theorem \ref{theo:overlap}}\label{sec_app:sketch_overlap}
The proof of Theorem \ref{theo:overlap} follows from \cite[Theorem 2]{barbier2019optimal}, so we will only describe its sketch. 
To obtain an upper bound on the overlap $Q_d \coloneqq (1/d)\Tr[s S^\star]$, we consider the following model
\begin{equation}
\label{eq:model_w_spike}
\left\{
\begin{aligned}
&Y\sim P_{\out}(\cdot|\Tr[GS^*])\\
&Y'=\sqrt{\frac{\lambda}{d^{2p-1}}}(S^*)^{\otimes2p}+Z',
\end{aligned}\right.
\end{equation}
where $Z'\in \mcS_d^{\otimes2p} \iid \mcN(0,1)$.
The additional side information comes from a so-called \emph{spiked tensor model}, which is considered in Appendix~\ref{sec_app:spike}. A combination of Theorems \ref{thm:fentropy_symmetric} and \ref{theo:tensor_main} gives its free entropy (see also \cite[Section 5.3]{barbier2019optimal} for details):
\begin{equation*}
\lim_{d\to\infty}f_d=\sup_{q\in[0,\rho]}\inf_{r\geq0}\psi_{P_0}(r+4p\lambda q^{2p-1})+\alpha\Psi_{\out}(q)+\frac{1}{4}+p\lambda\rho q^{2p-1}+\frac{1}{4}r(\rho-q)-\frac{\lambda(2p-1)q^{2p}}{2}.
\end{equation*}
Similarly to \cite[Section 5.3]{barbier2019optimal}, we can use the I-MMSE theorem with respect to the signal-to-noise ration $\lambda$ to obtain
\begin{equation*}
\frac{1}{d^{2p}}\text{MMSE}((S^*)^{\otimes2p}|Y,Y',G)\to\rho^{2p}-q^*(\alpha,\lambda)^{2p},
\end{equation*}
and thus (since the side-information channel in eq.~\eqref{eq:model_w_spike} can only reduce the MMSE with respect to the original 
observation model):
\begin{equation*}
\liminf_{d\to\infty}\frac{1}{d^{2p}}\text{MMSE}((S^*)^{\otimes2p}|Y,G)\geq\rho^{2p}-q^*(\alpha)^{2p}.
\end{equation*}
As the left side of this last equation is equal $\rho^{2p}-\mathbb{E}[Q_d^{2p}]$ by the Nishimori identity (Proposition~\ref{prop:nishimori}), 
we obtain an upper bound on the overlap:
\begin{equation*}
\limsup_{d\to\infty}\mathbb{E}[Q_d^{2p}]\leq q^*(\alpha)^{2p},
\end{equation*}
for any $p \geq 1$. 
This implies 
\begin{equation}
\lim_{d\to\infty}\mathbb{P}(|Q_d|\geq q^*(\alpha)+\epsilon)=0
\label{eq:upper_bound_overlap}
\end{equation}
for all $\epsilon>0$.
The lower bound on the overlap 
\begin{equation}
\lim_{d\to\infty}\mathbb{P}(|Q_d|\leq q^*(\alpha)-\epsilon)=0
\label{eq:lower_bound_overlap}
\end{equation}
is obtained in the exact same way as \cite[Section~5.3.2]{barbier2019optimal}, thus we omit its proof. Combining eqs. \eqref{eq:upper_bound_overlap} and \eqref{eq:lower_bound_overlap}, we obtain the convergence of the overlap, and thus the MMSE.

\myskip
Finally, we note that all the above arguments still hold if we replace $G_\mu$ with $\Phi_\mu$ under Assumption~\ref{assum:CLT}, because a direct generalization of Theorem~\ref{thm:fentropy_symmetric} to the model of eq.~\eqref{eq:model_w_spike} shows that the limiting free entropy is unchanged upon replacing $G_\mu$ by $\Phi_\mu$.

\section{Proof of Theorem \ref{theo:2layerNN}}\label{sec_app:reduction_2nn}
According to Section \ref{subsec:2layerNN}, it remains to prove eq. \eqref{eq:partial_lambda_f} and Lemma \ref{lemma:diff_t0_1}.

\subsection{\texorpdfstring{Proof of eq. \eqref{eq:partial_lambda_f}}{}}

The I-MMSE theorem \cite{guo2005mutual} implies that (this can also easily be re-derived from eq.~\eqref{eq:def_fnd} using the Nishimori identity, see Proposition~\ref{prop:nishimori}), 
for any $\Lambda \geq 0$:
\begin{align}\label{eq:I_MMSE}
     \frac{\partial}{\partial\Lambda} f_{d}(\Lambda) &= -\frac{d}{2} \EE \left[(S^\star_{12} - \langle S_{12} \rangle_{t,\Lambda})^2\right].
\end{align}
Moreover, for $\Lambda = 0$, by permutation invariance of the law of $S$, we have
\begin{align}
\nonumber
 \frac{\partial}{\partial\Lambda} f_{d}(0) &= -\frac{1}{2(d-1)} \sum_{i \neq j}\EE \left[(S^\star_{ij} - \langle S_{ij} \rangle_{t,0})^2\right], \\ 
 &=
 - \frac{1}{2(d-1)} \left[d  \, \EE \, \tr[(S^\star - \langle S \rangle_{t,0})^2] - \sum_{i=1}^d \EE[(S^\star_{ii} - \langle S_{ii} \rangle_{t,0})^2] \right].
\end{align}
In particular, this implies that (using again permutation invariance and the Nishimori identity)
\begin{align}\label{eq:I_MMSE_Lambda_0_first}
    \nonumber
     \left|\frac{\partial}{\partial\Lambda} f_{d}(0) + \frac{1}{2} \EE \, \tr[(S^\star - \langle S \rangle_{t,0})^2] \right|
    &\leq 
    \frac{d}{2(d-1)} \EE[(S^\star_{11})^2 - \langle S_{11} \rangle_{t,0}^2] + 
    \frac{1}{2(d-1)} \EE \, \tr[(S^\star)^2 - \langle S \rangle_{t,0}^2], \\ 
    \nonumber
    &\leq 
    \frac{d}{2(d-1)} \EE[(S^\star_{11}-1)^2 - \langle S_{11}-1 \rangle_{t,0}^2] + 
    \frac{1}{2(d-1)} \EE \, \tr[(S^\star)^2 - \langle S \rangle_{t,0}^2], \\ 
    \nonumber
    &\leq 
    \frac{d}{2(d-1)} \EE[(S^\star_{11}-1)^2]  + 
    \frac{1}{2(d-1)} \EE \, \tr[(S^\star)^2], \\ 
    &\leq \frac{C(\kappa)}{d},
\end{align}
for some $C(\kappa) > 0$ (we used the form of $S^\star = (1/m) \sum_{k=1}^m w_k w_k^\T$ for $(w_k)_{k=1}^m \iid \mcN(0, \Id_d)$), which gives \eqref{eq:partial_lambda_f}.

\subsection{Proof of MMSE equivalence: Lemma \ref{lemma:diff_t0_1}}

Lemma \ref{lemma:diff_t0_1} relies on a main lemma, which establishes that the derivative of the interpolating free entropy goes to zero uniformly. 
It is proven in Appendix \ref{sec_app:partial_t_f}.
\begin{lemma}
\label{lemma:partial_t_f}
\begin{equation*}
\lim_{d\to\infty} \sup_{\Lambda \geq 0}\left|\frac{\partial f_{d}(t, \Lambda)}{\partial t}\right| =0,
\end{equation*}
uniformly in $t\in[t_0,1]$.
\end{lemma}
By the fundamental theorem of calculus and Jensen's inequality, we have  
\begin{align*}
\sup_{\Lambda \geq 0} |f_{d}(1, \Lambda) - f_{d}(t_0, \Lambda)| &= \sup_{\Lambda \geq 0} \left|\int_0^1 \frac{\partial f_{d}(t, \Lambda)}{\partial t} \, \rd t\right|, \\ 
&\leq \int_{t_0}^1 \sup_{\Lambda \geq 0} \left|\frac{\partial f_{d}(t, \Lambda)}{\partial t}\right| \, \rd t.
\end{align*}
By the uniform convergence proven in Lemma~\ref{lemma:partial_t_f}, this implies
\begin{align}\label{eq:equivalence_free_entropy_uniform}
  \sup_{\Lambda \geq 0} |f_{d}(1, \Lambda) - f_{d}(t_0, \Lambda)| \leq h_d,
\end{align}
for some $h_d \to 0$ as $d \to \infty$. 
Notice that by the I-MMSE theorem, for any $t \in [t_0,1]$, $f_{d}(t, \Lambda)$ is a convex function of $\Lambda$ because
\begin{align}
\label{eq:d2fnd_Lambda}
    \frac{\partial^2 f_d(t, \Lambda)}{\partial \Lambda^2} 
    = \frac{d^2}{2} 
    \EE\left[\left(\langle S_{12}^2 \rangle_{t,\Lambda} - \langle S_{12}\rangle_{t,\Lambda}^2\right)^2\right]\geq0.
\end{align}
By said convexity we have, for any $\Lambda > 0$: 
\begin{align}
\label{eq:bound_fnd_convexity_1}
    \nonumber
    \left(\frac{\partial f_d(t_0, \Lambda)}{\partial \Lambda}\right)_{\Lambda = 0} &\leq \frac{f_d(t_0, \Lambda) - f_d(t_0, 0)}{\Lambda}, \\ 
    \nonumber
    &\aleq \frac{2 h_d}{\Lambda} + \frac{f_d(1, \Lambda) - f_d(1, 0)}{\Lambda}, \\ 
    &\bleq \frac{2 h_d}{\Lambda} + \left(\frac{\partial f_d(1, \Lambda)}{\partial \Lambda}\right)_{\Lambda}.
\end{align}
where we used eq.~\eqref{eq:equivalence_free_entropy_uniform} in $(\rm a)$, and again convexity in $(\rm b)$.
By a symmetric argument, we get
\begin{align}\label{eq:bound_fnd_convexity_2}
    \left(\frac{\partial f_d(1, \Lambda)}{\partial \Lambda}\right)_{\Lambda = 0}
    &\leq \frac{2 h_d}{\Lambda} + \left(\frac{\partial f_d(t_0, \Lambda)}{\partial \Lambda}\right)_{\Lambda}.
\end{align}
By the fundamental theorem of analysis and eqs.~\eqref{eq:bound_fnd_convexity_1} and \eqref{eq:bound_fnd_convexity_2}, 
we get that for all $\Lambda > 0$:
\begin{align}\label{eq:bound_df_equivalence}
\nonumber
    \left|\left(\frac{\partial f_d(t_0, \Lambda)}{\partial \Lambda}\right)_{\Lambda = 0} - \left(\frac{\partial f_d(1, \Lambda)}{\partial \Lambda}\right)_{\Lambda = 0} \right| 
    &\leq \frac{2 h_d}{\Lambda} + \max_{t \in \{t_0,1\}} \int_0^\Lambda \frac{\partial^2 f_d(t, u)}{\partial u^2} \, \rd u, \\
    &\leq \frac{2 h_d}{\Lambda} + \sum_{t \in \{t_0,1\}} \int_0^\Lambda \frac{\partial^2 f_d(t, u)}{\partial u^2} \, \rd u,
\end{align}
where we also used that $\partial^2_u f_d(t, u) \geq 0$ by convexity. By eq. \eqref{eq:d2fnd_Lambda} we have for any $\Lambda \geq 0$:
\begin{align}
    \frac{\partial^2 f_d(t, \Lambda)}{\partial \Lambda^2} 
    \leq 
    \frac{d^2}{2} 
    \EE[\langle S_{12}^2 \rangle_{t,\Lambda}^2]\leq 
    \frac{d^2}{2} 
    \EE[ (S^\star_{12})^4]= \frac{3}{2\kappa^2} + \frac{3}{\kappa^4 d^2},
\end{align}
where we used the Cauchy-Schwartz inequality and the Nishimori identity (Proposition~\ref{prop:nishimori}).
By taking taking $\Lambda = \sqrt{h_d}$, we obtain an upper bound of the right hand side of eq.~\eqref{eq:bound_df_equivalence}
\begin{align}
\label{eq:bound_d2fnd_Lambda}
    \inf_{\Lambda > 0} \left\{ \frac{2 h_d}{\Lambda} + \sum_{t \in \{0,1\}} \int_0^\Lambda \frac{\partial^2 f_d(t, u)}{\partial u^2} \, \rd u\right\} 
    \leq C(\kappa) \sqrt{h_d},
\end{align}
for some $C(\kappa)>0$. We finish the proof of Lemma \ref{lemma:diff_t0_1} by combining eqs. \eqref{eq:bound_df_equivalence} and \eqref{eq:bound_d2fnd_Lambda}.
    
\subsection{Proof of free entropy equivalence: Lemma \ref{lemma:partial_t_f}}
\label{sec_app:partial_t_f}
Recall that the interpolating model is 
\begin{align}
\nonumber
v_{t,\mu} &\coloneqq \Tr[\Phi_\mu S^\star] + (1-t)\sqrt{d} [\tr \, S^\star - 1] 
+ \sqrt{d(1-t)} \Delta \left(\frac{\|z_\mu\|^2}{m} - 1\right)  \\ 
&+ \frac{2\sqrt{\Delta d(1-t)}}{m} \sum_{k=1}^m z_{\mu,k} \left(\frac{x_\mu^\T w_k^\star}{\sqrt{d}}\right)+ \sqrt{\tDelta t} \zeta_\mu,
\end{align}
with a side information channel
\begin{align}
Y' = \sqrt{\Lambda} S^\star_{12} + \frac{\xi}{\sqrt{d}}.
\end{align}
Its free entropy is given by
\begin{equation*}
\begin{aligned}
f_d(t,\Lambda)&= \EE_{\{x_\mu\}_{\mu=1}^n} \frac{1}{d^2}\int\mcD W^\star\int dY'\left[\prod_{\mu=1}^n dv_\mu P_\out^{(t)}(v_\mu | W^\star,x_\mu)\right] \, e^{-\frac{d}{4} (Y' - \sqrt{\Lambda} S^\star_{12})^2}\\
&\qquad\log \int \mcD W \left[\prod_{\mu=1}^n P_\out^{(t)}(v_\mu | W,x_\mu) \right] \, e^{-\frac{d}{4} (Y' - \sqrt{\Lambda} S_{12})^2}.
\end{aligned}
\end{equation*}
Note that the output channel can be written as
\begin{equation*}
P_\out^{(t)}(v_\mu | W^\star,x_\mu)=\mathbb{E}_{z,\zeta}\delta(v_\mu-\tv(t,x_\mu,W^\star,z_\mu,\zeta_\mu)),
\end{equation*}
where
\begin{align*}
\tv(t,x_\mu,W^\star,z_\mu,\zeta_\mu) &\coloneqq \Tr[\Phi_\mu S^\star] + (1-t)\sqrt{d} [\tr \, S^\star - 1] 
    + \sqrt{d(1-t)} \Delta \left(\frac{\|z_\mu\|^2}{m} - 1\right) 
    \\ 
    &+ \frac{2\sqrt{\Delta d(1-t)}}{m} \sum_{k=1}^m z_{\mu,k} \left(\frac{x_\mu^\T w_k^\star}{\sqrt{d}}\right) + \sqrt{\tDelta t} \zeta_\mu.
\end{align*}
Thus we have (with $z = \{z_\mu\}$, $x = \{x_\mu\}$, $\zeta = \{\zeta_\mu\}$):
\begin{equation*}
f_d(t,\Lambda)=\frac{1}{d^2}\EE_{x,W^\star,z,\zeta,Y'}\log\int \mcD W \left[\prod_{\mu=1}^n P_\out^{(t)}(\tv(t,x_\mu,W^\star,z_\mu,\zeta_\mu)  | W,x_\mu) \right] \, e^{-\frac{d}{4} (Y' - \sqrt{\Lambda} S_{12})^2}.
\end{equation*}
Its derivative can be computed as
\begin{equation}
\label{eq:partial_t_f}
\frac{\partial}{\partial t}f_d(t,\Lambda)=\frac{1}{d^2}\EE_{x,W^\star,z,\zeta,Y'}\sum_{\mu=1}^n\left\langle \frac{\partial}{\partial t}\log P_\out^{(t)}(\tv(t,x_\mu,W^\star,z_\mu,\zeta_\mu)  | W,x_\mu) \right\rangle,
\end{equation}
where the Gibbs bracket is defined as:
\begin{equation*}
\langle g(W)\rangle\coloneqq \frac{\int \mcD W g(W)\left[\prod_{\mu=1}^n P_\out^{(t)}(\tv(t,x_\mu,W^\star,z_\mu,\zeta_\mu)  | W,x_\mu) \right] \, e^{-\frac{d}{4} (Y' - \sqrt{\Lambda} S_{12})^2}}{\int \mcD W \left[\prod_{\mu=1}^n P_\out^{(t)}(\tv(t,x_\mu,W^\star,z_\mu,\zeta_\mu)  | W,x_\mu) \right] \, e^{-\frac{d}{4} (Y' - \sqrt{\Lambda} S_{12})^2}}.
\end{equation*}
Notice that this Gibbs bracket depends on the realization of $x, W^\star, z, \zeta, Y'$.
The following two lemmas give a bound of the right hand side of eq. \eqref{eq:partial_t_f}. 
They are proven in Appendices \ref{sec_app:proof_partial_t_logPout} and \ref{sec_app:proof_square_partial_t_logPout}.
\begin{lemma}
\label{lemma:partial_t_logPout}
\noindent
There exists $\delta_0(\Delta,\kappa,t_0)>0$ such that for any $\delta<\delta_0$ and $d>d_0(\delta,\Delta,\kappa,t_0)$, as long as
\begin{equation}
\left|\frac{x_\mu^\T Sx_\mu}{d}-1\right|\leq\frac{\delta}{\sqrt[4]{d}},\ |\sqrt{d}(\tr S-1)|\leq\frac{\delta}{\sqrt[4]{d}}, |\tv_\mu-\Tr[\Phi_\mu S^\star]|^2<2t_0\tDelta^2\sqrt[4]{d},
\label{eq:three_events}
\end{equation}
we have
\begin{equation*}
\left|\frac{\partial}{\partial t}\log P_\out^{(t)}(\tv_\mu|W,x_\mu)\right|\leq C(\Delta,\kappa,t_0)\delta,
\end{equation*}
for some $C(\Delta,\kappa,t_0)>0$.
\end{lemma}
\begin{lemma}
\label{lemma:square_partial_t_logPout}
\noindent
For any $\mu=1,\cdots,n$, we have
\begin{equation*}
\EE_{x,W^\star,z,\zeta,Y'}\left\langle\left(\frac{\partial}{\partial t}\log P_\out^{(t)}(\tv(t,x_\mu,W^\star,z_\mu,\zeta_\mu)  | W,x_\mu)\right)^2\right\rangle\leq\frac{\rm{poly}(d)}{t_0^{12}},
\end{equation*}
where $\rm{poly}(d)$ represents a polynomial of $d$.
\end{lemma}
In the following, we will call $\delta<\delta_0$ ``sufficiently small $\delta$'' and $d>d_0$ ``sufficiently large $d$''.

\myskip
We denote the three events in eq.~\eqref{eq:three_events} as $A_{1,\mu}(S,x_\mu),A_{2}(S),A_{3,\mu}(S^\star,x_\mu,z_\mu,\zeta_\mu)$. 
With $A_\mu\coloneqq A_{1,\mu}\cap A_2\cap A_{3,\mu}$, and we have
\begin{equation*}
\begin{aligned}
\left|\frac{\partial}{\partial t}f_{n,d}(t,\Lambda)\right|&=\left|\frac{1}{d^2}\EE_{x,W^\star,z,\zeta,Y'}\sum_{i=1}^n\left\langle \frac{\partial}{\partial t}\log P_\out^{(t)}(\tilde{v}_\mu(t,x_\mu,W^\star,z_\mu,\zeta_\mu)  | W,x_i) (\indi\{A_\mu\}+\indi\{A_\mu^\complement\})\right\rangle\right|\\
&\leq\alpha C(\Delta,\kappa,t_0)\delta+\frac{1}{d^2}\sum_{\mu=1}^n\left[\EE_{x,W^\star,z,\zeta,Y'}\left\langle\indi\{A_\mu^\complement\})\right\rangle\right]^{1/2}\\&\qquad\qquad\left[\EE_{x,W^\star,z,\zeta,Y'}\left\langle\left(\frac{\partial}{\partial t}\log P_\out^{(t)}(\tilde{v}_\mu(t,x_\mu,W^\star,z_\mu,\zeta_\mu)  | W,x_\mu)\right)^2\right\rangle\right]^{1/2},
\end{aligned}
\end{equation*}
using Lemma \ref{lemma:partial_t_logPout} and the Cauchy-Schwartz inequality. By the union bound and the Nishimori identity, we have
\begin{equation}
\begin{aligned}
&\EE_{x,W^\star,z,\zeta,Y'}\left\langle\indi\{A_\mu^\complement\})\right\rangle\\&\leq\EE_{x,W^\star,z,\zeta,Y'}\left\langle\indi\{A_{1,\mu}^\complement(S,x_\mu)\}+\indi\{A_{2}^\complement(S)\}+\indi\{A_{3,\mu}^\complement(S^\star,x_\mu,z_\mu,\zeta_\mu)\}\right\rangle\\
&=\mathbb{P}(A_{1,\mu}^\complement(S^\star,x_\mu))+\mathbb{P}(A_{2}^\complement(S^\star))+\mathbb{P}(A_{3,\mu}^\complement(S^\star,x_\mu,z_\mu,\zeta_\mu)).
\end{aligned}
\label{eq:A_A1_A2_A3}
\end{equation}
Now let us control the three terms on the right hand side. By Bernstein’s inequality and the Hanson-Wright inequality
for any $0<\delta<1$, it is easy to show that with probability at least $1-2e^{-C(\kappa)d\delta^2}$:
\begin{equation*}
\left|\frac{x_\mu^\T S^\star x_\mu}{d}-1\right|\leq\delta,\ |\sqrt{d}(\tr S^\star-1)|\leq\delta,
\end{equation*}
Details on this classical derivation can be found in~\cite[Appendix~D.5]{maillard2024bayes}.
Thus:
\begin{equation}\label{eq:bound_A1_A2}
\mathbb{P}(A_{1,\mu}^\complement),\mathbb{P}(A_{2}^\complement)\leq 2e^{-C(\kappa)\sqrt{d}\delta^2}.
\end{equation}
for sufficiently small $\delta$.
To control $\mathbb{P}(A_{3,\mu}^\complement)$, we recall that
\begin{align}
\nonumber
\tilde{v}_\mu-\Tr[\Phi_\mu S^\star] \coloneqq& (1-t)\sqrt{d} [\tr \, S^\star - 1] +\sqrt{d(1-t)} \Delta \left(\frac{\|z_\mu\|^2}{m} - 1\right) \\& 
+ \frac{2\sqrt{\Delta d(1-t)}}{m} \sum_{k=1}^m z_{\mu,k} \left(\frac{x_\mu^\T w_k^\star}{\sqrt{d}}\right)+ \sqrt{\tDelta t} \zeta_\mu.
\end{align}
For the first term, we have
\begin{equation}\label{eq:bound_term_1}
\mathbb{P}(|(1-t)\sqrt{d} [\tr \, S^\star - 1] |>\varepsilon)\leq\mathbb{P}(|\sqrt{d} [\tr \, S^\star - 1] |>\varepsilon)\leq 2e^{-Cd\varepsilon^2},
\end{equation}
for any $\eps < 1$, by Bernstein's inequality. 
For the second term, the classical Laurent-Massart bound of~\cite{laurent2000adaptive} shows that
\begin{equation*}
\mathbb{P}(|\|z_\mu\|^2-m|>2\sqrt{m}\varepsilon+2\varepsilon^2)\leq 2e^{-\varepsilon^2},
\end{equation*}
and thus
\begin{equation}\label{eq:bound_term_2}
\mathbb{P}\left(\sqrt{d(1-t)} \Delta \left|\frac{\|z_\mu\|^2}{m} - 1\right|>\varepsilon\right)=\mathbb{P}\left(|\|z_\mu\|^2-m|>\frac{\kappa\sqrt{d}\varepsilon}{\Delta\sqrt{1-t}}\right)\leq 2e^{-C(\kappa,\Delta,t_0)\varepsilon^2}.
\end{equation}
For the third term, we have
\begin{equation*}
\frac{2\sqrt{\Delta d(1-t)}}{m} \sum_{k=1}^m z_{\mu,k} \left(\frac{x_\mu^\T w_k^\star}{\sqrt{d}}\right)\Big|W^\star,x_\mu\sim\mathcal{N}\left(0,\frac{4\Delta^2}{\kappa}\frac{x_\mu^\T S^\star x_\mu}{d}\right),
\end{equation*}
and thus
\begin{equation*}
\mathbb{P}\left(\left|\frac{2\sqrt{\Delta d(1-t)}}{m} \sum_{k=1}^m z_{k} \left(\frac{x_\mu^\T w_k^\star}{\sqrt{d}}\right)\right|>\varepsilon\Big|W^\star,x_\mu\right)\leq 2\exp\left\{-\frac{\kappa}{8\Delta^2}\frac{d}{x_\mu^\T S^\star x_\mu}\eps^2\right\},
\end{equation*}
which gives
\begin{align}
    \label{eq:bound_term_3}
    \nonumber
&\mathbb{P}\left(\left|\frac{2\sqrt{\Delta d(1-t)}}{m} \sum_{k=1}^m z_{\mu,k} \left(\frac{x_\mu^\T w_k^\star}{\sqrt{d}}\right)\right|>\varepsilon\right)
\\
    \nonumber
&\leq\mathbb{P}\left(\left|\frac{2\sqrt{\Delta d(1-t)}}{m} \sum_{k=1}^m z_{\mu,k} \left(\frac{x_\mu^\T w^\star_k}{\sqrt{d}}\right)\right|>\varepsilon\middle|\frac{x_\mu^\T S^\star x_\mu}{d}<2\right)+\mathbb{P}\left(\frac{x_\mu^\T S^\star x_\mu}{d}>2\right)\\
&\leq2e^{-\frac{\kappa\varepsilon^2}{16\Delta^2}}+e^{-C(\kappa)\sqrt{d}},
\end{align}
using again eq.~\eqref{eq:bound_A1_A2}.
For the last term, we have
\begin{equation}\label{eq:bound_term_4}
\mathbb{P}(|\sqrt{\tDelta t_0}\zeta_\mu|>\varepsilon)\leq 2e^{-\frac{\varepsilon^2}{2\tDelta t_0}}.
\end{equation}
Combining all the bounds of eqs.~\eqref{eq:bound_term_1},\eqref{eq:bound_term_2},\eqref{eq:bound_term_3},\eqref{eq:bound_term_4},
and choosing $\varepsilon\coloneqq \frac{1}{2}\tDelta\sqrt{t_0}\sqrt[8]{d}$, we have
\begin{equation*}
\begin{aligned}
\mathbb{P}(A_{3,\mu}^\complement)&\leq \mathbb{P}(|(1-t)\sqrt{d} [\tr \, S^\star - 1] |>\varepsilon)+\mathbb{P}\left(\sqrt{d(1-t)} \Delta \left|\frac{\|z_\mu\|^2}{m} - 1\right|>\varepsilon\right)\\&+\mathbb{P}\left(\left|\frac{2\sqrt{\Delta d(1-t)}}{m} \sum_{k=1}^m z_{\mu,k} \left(\frac{x_\mu^\T w_k^\star}{\sqrt{d}}\right)\right|>\varepsilon\right)+\mathbb{P}(|\sqrt{\tDelta t_0}\zeta_\mu|>\varepsilon)\\
&\leq e^{-C(\kappa,\Delta,t_0)\sqrt[4]{d}}
\end{aligned}
\end{equation*}
for some constant $C(\kappa,\Delta,t_0)>0$. Therefore, we have
\begin{equation*}
\EE_{x,W^\star,z,\zeta,Y'}\left\langle\indi\{A_\mu^\complement\}\right\rangle\leq e^{-C(\kappa,\Delta,t_0)\sqrt[4]{d}}
\end{equation*}
according to eq. \eqref{eq:A_A1_A2_A3}, which gives
\begin{equation}
\left|\frac{\partial}{\partial t}f_{n,d}(t,\Lambda)\right|\leq\alpha C(\Delta,\kappa,t_0)\delta+\alpha e^{-\sqrt{C(\Delta,\kappa,t_0)}\sqrt[8]{d}}\frac{\text{poly}(d)}{t_0^6},
\end{equation}
where we use Lemma \ref{lemma:square_partial_t_logPout}. Notice that all these bounds did not depend on $\Lambda$ and are valid uniformly in $\Lambda$.
Therefore, we have
\begin{equation*}
\limsup_{d \to \infty}\sup_{\Lambda \geq 0}\left|\frac{\partial}{\partial t}f_{n,d}(t,\Lambda)\right|\leq\alpha C(\Delta,\kappa,t_0)\delta,
\end{equation*}
which proves Lemma \ref{lemma:partial_t_f} by taking the limit $\delta \downarrow 0$.

\subsection{Proof of Lemma \ref{lemma:partial_t_logPout}}
\label{sec_app:proof_partial_t_logPout}
From now on we drop the index $\mu$ for notational simplicity. Recall that we have assumed 
\begin{equation*}
\left|\frac{x^\T Sx}{d}-1\right|\leq\frac{\delta}{\sqrt[4]{d}},\ |\sqrt{d}(\tr S-1)|\leq\frac{\delta}{\sqrt[4]{d}},\ |\bar{v}|^2<2t_0\tDelta^2\sqrt[4]{d},
\end{equation*}
where $\bar{v}\coloneqq \tilde{v}-\Tr[\Phi S^\star]$.
Note that $P_\out^{(t)}(v | W,x)$ can also be written as
\begin{equation*}
\begin{aligned}
P_\out^{(t)}(v | W,x)&\coloneqq\frac{1}{\sqrt{2\pi\tilde{\Delta}}t}\EE_z\exp\left\{-\frac{1}{2\tilde{\Delta}t}\left(v-\Tr[\Phi S] + (1-t)\sqrt{d} [\tr S - 1]\right.\right. 
\\ &\left.\left.\qquad+ \sqrt{d(1-t)} \Delta \left(\frac{\|z\|^2}{m} - 1\right) 
+\frac{2\sqrt{\Delta d(1-t)}}{m} \sum_{k=1}^m z_{k} \left(\frac{x^\T w_k}{\sqrt{d}}\right) \right)^2\right\}.
\end{aligned}
\end{equation*}
By using the identity
\begin{equation*}
\frac{1}{\sqrt{2\pi\sigma^2}}e^{-z^2/2\sigma} = \frac{1}{2\pi}\int dpe^{-\sigma^2p^2/2+ipz},
\end{equation*}
we have
\begin{equation*}
\begin{aligned}
P_\out^{(t)}(v | W,x)&=\frac{1}{2\pi}\int dp\exp\left\{-\frac{\tilde{\Delta}tp^2}{2}+ip\left(v-\Tr[\Phi S] + (1-t)\sqrt{d} [\tr S - 1]\right)\right\}
\\ &\prod_{k=1}^m\EE_z\exp\left\{ip\left(\sqrt{d(1-t)} \Delta \left(\frac{z^2}{m} - 1\right) 
+\frac{2\sqrt{\Delta d(1-t)}}{m}z \left(\frac{x^\T w_k^\star}{\sqrt{d}}\right) \right)\right\}\\
&=\frac{1}{2\pi}\int dp\exp\left\{-\frac{\tilde{\Delta}tp^2}{2}+ip\left(v-\Tr[\Phi S] + (1-t)\sqrt{d} [\tr S - 1]\right)\right\}\\ &\qquad\exp\left\{-\frac{\kappa d}{2}\log\left(1-\frac{2ip\Delta\sqrt{1-t}}{\kappa\sqrt{d}}\right)-ip\sqrt{d(1-t)}\Delta-\frac{2p^2\Delta(1-t)x^\T Sx}{\kappa d\left(1-\frac{2ip\Delta\sqrt{1-t}}{\kappa\sqrt{d}}\right)}\right\},
\end{aligned}
\end{equation*}
where we explicitly calculate the expectation w.r.t. $z\sim\mathcal{N}(0,1)$, and use that $m=\kappa d$. 
An important limit to notice is that
\begin{align*}
    g_d(t,p)&\coloneqq-\frac{\tDelta tp^2}{2}-\frac{\kappa d}{2}\log\left(1-\frac{2ip\Delta\sqrt{1-t}}{\kappa\sqrt{d}}\right)-ip\sqrt{d(1-t)}\Delta-\frac{2p^2\Delta(1-t)}{\kappa\left(1-\frac{2ip\Delta\sqrt{1-t}}{\kappa\sqrt{d}}\right)}
    \\&=-\frac{\tDelta p^2}{2}+\mcO\left(\frac{p^3}{\sqrt{d}}\right),
\end{align*}
by direct calculation. Therefore, we will truncate the integral on $|p|\leq\sqrt[8]{d}$, 
so that that $g_d(t,p)+\frac{\tDelta p^2}{2}=o\left(\frac{1}{\sqrt[4]{d}}\right)$ uniformly for $|p| \leq \sqrt[8]{d}$. We estimate the truncation error by the domination function
\begin{equation}
\begin{aligned}
&\left|\exp\left\{-\frac{\tilde{\Delta}tp^2}{2}+ip\left(v-\Tr[\Phi S] + (1-t)\sqrt{d} [\tr S - 1]\right)\right\}\right.\\ &\left.\qquad\exp\left\{-\frac{\kappa d}{2}\log\left(1-\frac{2ip\Delta\sqrt{1-t}}{\kappa\sqrt{d}}\right)-ip\sqrt{d(1-t)}\Delta-\frac{2p^2\Delta(1-t)x^\T Sx}{\kappa d\left(1-\frac{2ip\Delta\sqrt{1-t}}{\kappa\sqrt{d}}\right)}\right\}\right|\\&\leq\exp\left\{-\frac{\tilde{\Delta}tp^2}{2}\right\}.
\end{aligned}
\label{eq:ineq_domination}
\end{equation}
We have
\begin{tiny}
    \begin{align}
\label{eq:bound_Pout-1}
\nonumber
&|P_{\out}^{(t)}(v|W,x)\sqrt{2\pi\tilde{\Delta}}e^{\frac{1}{2\tilde{\Delta}}(v-\Tr[\Phi S])^2}-1|
\\
\nonumber
&\coloneqq \left|\sqrt{\frac{\tDelta}{2\pi}}\int dp \exp\left\{ip(1-t)\sqrt{d}[\tr S-1]-\frac{2p^2\Delta(1-t)}{\kappa\left(1-\frac{2ip\Delta\sqrt{1-t}}{\kappa\sqrt{d}}\right)}\left(\frac{x^\T Sx}{d}-1\right)+g_d(t,p)+ip\left(v-\Tr[\Phi S]\right)+\frac{1}{2\tilde{\Delta}}(v-\Tr[\Phi S])^2\right\}-1\right|
\\
\nonumber
&\overset{(a)}{\leq}\sup_{|a_1|,|a_2|,|a_3|,|a_4|\leq\delta/\sqrt[4]{d}}\left|\sqrt{2\pi\tilde{\Delta}}\int_{|p|<\sqrt[8]{d}} \rd p \exp\left\{ipa_1+\frac{2p^2\Delta}{\kappa}a_2(1+a_3)-\frac{\tDelta}{2}\left(p-\frac{i}{\tDelta}\left(v-\Tr[\Phi S]\right)\right)^2+a_4p^2\right\}-1\right|
\\
\nonumber
&+\sqrt{2\pi\tilde{\Delta}}\int_{|p|>\sqrt[8]{d}}dpe^{-\tDelta tp^2/2}e^{\frac{1}{2\tilde{\Delta}}(v-\Tr[\Phi S])^2}
\\
\nonumber
&\leq\sup_{|a_1|,|a_2|,|a_3|,|a_4|\leq\delta/\sqrt[4]{d}}\left|\sqrt{2\pi\tilde{\Delta}}\int \rd p \exp\left\{ipa_1+\frac{2p^2\Delta}{\kappa}a_2(1+a_3)-\frac{\tDelta}{2}\left(p-\frac{i}{\tDelta}\left(v-\Tr[\Phi S]\right)\right)^2+a_4p^2\right\}-1\right|
\\
\nonumber
&+\sqrt{2\pi\tDelta}\int_{|p|>\sqrt[8]{d}} \rd p \left|\exp\left\{ipa_1+\frac{2p^2\Delta}{\kappa}a_2(1+a_3)-\frac{\tDelta}{2}\left(p-\frac{i}{\tDelta}\left(v-\Tr[\Phi S]\right)\right)^2+a_4p^2\right\}\right|+\sqrt{2\pi\tilde{\Delta}}\int_{|p|>\sqrt[8]{d}}\rd pe^{-\tDelta tp^2/2}e^{\frac{1}{2\tilde{\Delta}}(v-\Tr[\Phi S])^2},
\\
\nonumber
&\overset{(b)}{\leq}\sup_{|a_1|,|a_2|,|a_3|,|a_4|\leq\delta/\sqrt[4]{d}}\left|\sqrt{2\pi\tilde{\Delta}}\int \rd p \exp\left\{ipa_1+\frac{2p^2\Delta}{\kappa}a_2(1+a_3)-\frac{\tDelta}{2}\left(p-\frac{i}{\tDelta}\left(v-\Tr[\Phi S]\right)\right)^2+a_4p^2\right\}-1\right|
\\
&+\sqrt{2\pi\tDelta}\int_{|p|>\sqrt[8]{d}}(1+e^{\delta p^2})\rd pe^{-\tDelta tp^2/2}e^{\frac{1}{2\tilde{\Delta}}(v-\Tr[\Phi S])^2}.
    \end{align}
\end{tiny}
In (a) we used the domination function of eq.~\eqref{eq:ineq_domination}, and the fact that
\begin{equation*}
\left|\frac{1}{1-\frac{2ip\Delta\sqrt{1-t}}{\kappa\sqrt{d}}}-1\right|\leq\frac{\delta}{\sqrt[4]{d}},\ |g_d(t,p)+\frac{\tDelta p^2}{2}|\leq\frac{\delta}{\sqrt[4]{d}}p^2
\end{equation*}
for all $p\leq\sqrt[8]{d}$ and $d$ large enough. In (b) we use
\begin{equation*}
\begin{aligned}
&\left|\exp\left\{ipa_1+\frac{2p^2\Delta}{\kappa}a_2(1+a_3)-\frac{\tDelta}{2}\left(p-\frac{i}{\tDelta}\left(v-\Tr[\Phi S]\right)\right)^2+a_4p^2\right\}\right|\\&=\exp\left\{\frac{2p^2\Delta}{\kappa}a_2(1+a_3)-\frac{\tDelta}{2}p^2+\frac{1}{2\tDelta}\left(v-\Tr[\Phi S]\right)^2+a_4p^2\right\}\\&\leq\exp\left\{-\frac{\tDelta}{2}p^2+\frac{1}{2\tDelta}\left(v-\Tr[\Phi S]\right)^2+\delta p^2\right\}
\end{aligned}
\end{equation*}
for large enough $d$. 
The first term on the right hand side of \eqref{eq:bound_Pout-1} can be controlled because
    \begin{align}
        \label{eq:bound_first_term_Pout}
        \nonumber
    &\sup_{|a_1|,|a_2|,|a_3|,|a_4|\leq\delta/\sqrt[4]{d}}\left|\sqrt{\frac{\tilde{\Delta}}{2\pi}}\int \rd p \exp\left\{ipa_1+\frac{2p^2\Delta}{\kappa}a_2(1+a_3)-\frac{\tDelta}{2}\left(p-\frac{i}{\tDelta}\left(v-\Tr[\Phi S]\right)\right)^2+a_4p^2\right\}-1\right|\\
        \nonumber
    &\leq\sup_{|a_1|\leq\delta/\sqrt[4]{d},|a_2|\leq(1+4\Delta/\kappa)\delta/\sqrt[4]{d}}\left|\sqrt{\frac{\tilde{\Delta}}{2\pi}}\int \rd p \exp\left\{ipa_1+a_2p^2-\frac{\tDelta}{2}\left(p-\frac{i}{\tDelta}\left(v-\Tr[\Phi S]\right)\right)^2\right\}-1\right|\\
        \nonumber
    &=\sup_{|a_1|\leq\delta/\sqrt[4]{d},|a_2|\leq(1+4\Delta/\kappa)\delta/\sqrt[4]{d}}\left|\exp\left\{\frac{\bar{v}^2}{2\tDelta}-\frac{(\bar{v}+a_1)^2}{2\tDelta(1+a_2)}\right\}-1\right|\\
    &\leq\left|e^{C(\Delta,\kappa)\delta\bar{v}^2/\sqrt[4]{d}+C(\Delta,\kappa)\delta}-1\right|
    \end{align}
for some $C(\Delta,\kappa)>0$. The second term on the right hand side of \eqref{eq:bound_Pout-1} can be controlled by 
\begin{equation*}
\int_{|p|>\sqrt[8]{d}}\rd p(1+e^{\delta p^2})e^{-\tDelta tp^2/2}e^{\bar{v}^2/2\tDelta}\leq\frac{e^{-\tDelta t\sqrt[4]{d}/4}}{\tDelta t\sqrt[4]{d}}e^{\tDelta t\sqrt[8]{d}}\leq C(\Delta,\kappa,t_0)\delta
\end{equation*}
for $d$ large enough, where we recall that $|\bar{v}|^2\leq2t\tDelta^2\sqrt[4]{d}$. This finally gives the bound
\begin{equation*}
|P_{\out}^{(t)}(v|W,x)\sqrt{2\pi\tilde{\Delta}}e^{\frac{1}{2\tilde{\Delta}}(v-\Tr[\Phi S])^2}-1|\leq C(\Delta,\kappa,t_0)\delta,
\end{equation*}
where we use $|e^x-1|\leq 2|x|$ for $x$ sufficiently small.
Similarly, we have
\begin{tiny}
    \begin{align}
    \label{eq:estimate_partial_t_Pout}
    \nonumber
    &\left|\frac{\partial}{\partial t}P_{\out}^{(t)}(v|W,x)\right|\sqrt{2\pi\tilde{\Delta}}e^{\frac{1}{2\tilde{\Delta}}(v-\Tr[\Phi S])^2}
    \\
    \nonumber
    &=\left|\frac{1}{2\pi}\int dp\exp\left\{ip(1-t)\sqrt{d}[\tr S-1]-\frac{2p^2\Delta(1-t)}{\kappa\left(1-\frac{2ip\Delta\sqrt{1-t}}{\kappa\sqrt{d}}\right)}\left(\frac{x^\T Sx}{d}-1\right)+g_d(t,p)+ip\left(v-\Tr[\Phi S]\right)+(v-\Tr[\Phi S])^2/2\Delta\right\}\right.
    \\
    \nonumber
    &\left.\left(-p\sqrt{d}[\tr S-1]-\frac{\partial}{\partial t}\frac{2p^2\Delta(1-t)}{\kappa\left(1-\frac{2ip\Delta\sqrt{1-t}}{\kappa\sqrt{d}}\right)}\left(\frac{x^\T Sx}{d}-1\right)+\frac{\partial}{\partial t}g_d(t,p)\right)\right|
    \\
    &\leq C(\Delta,\kappa,t_0)\delta
    \\
    \nonumber
    &+\sup_{|a_i|<\delta/\sqrt[4]{d},\ i=1,\cdots,7}\left|\sqrt{2\pi\tilde{\Delta}}\int dp \exp\left\{ipa_1+\frac{2p^2\Delta}{\kappa}a_2(1+a_3)-\frac{\tDelta}{2}\left(p-\frac{i}{\tDelta}\left(v-\Tr[\Phi S]\right)\right)^2+a_4p^2\right\}(-pa_1-\frac{2p^2\Delta}{\kappa}a_4(1+a_5)+a_6+a_7p^2)\right|,
    \end{align}
\end{tiny}
where we truncate the integral again on $|p| \leq \sqrt[8]{d}$, and estimate the truncation error in the same way. The first term on the right side of eq. \eqref{eq:estimate_partial_t_Pout} corresponds to the truncation error. In  eq. \eqref{eq:estimate_partial_t_Pout} we also use the fact that after truncation, $
\left|\frac{\partial}{\partial t}g_d(t,p)\right|\leq\frac{\delta}{\sqrt[4]{d}}$ for $p\leq\sqrt[8]{d}$ and $d$ sufficiently large. 
Treating the integral on the right hand side of eq.~\eqref{eq:estimate_partial_t_Pout} in exactly the same way as above (see eq.~\eqref{eq:bound_first_term_Pout}), we reach:
\begin{equation*}
\left|\frac{\partial}{\partial t}P_{\out}^{(t)}(v|W,x)\right|\sqrt{2\pi\tilde{\Delta}}e^{\frac{1}{2\tilde{\Delta}}(v-\Tr[\Phi S])^2}\leq C(\Delta,\kappa,t_0)\delta
\end{equation*}
for some $C(\Delta,\kappa,t_0)>0$, which gives
\begin{equation*}
\begin{aligned}
\left|\frac{\partial}{\partial t}\log P_{out}^{(t)}(v|W,x)\right|&=\frac{\left|\frac{\partial}{\partial t}P_{\out}^{(t)}(v|W,x)\right|\sqrt{2\pi\tilde{\Delta}}e^{\frac{1}{2\tilde{\Delta}}(v-\Tr[\Phi S])^2}}{|P_{\out}^{(t)}(v|W,x)\sqrt{2\pi\tilde{\Delta}}e^{\frac{1}{2\tilde{\Delta}}(v-\Tr[\Phi S])^2}|}\leq C(\Delta,\kappa,t_0)\delta
\end{aligned}
\end{equation*}
for some $C(\Delta,\kappa,t_0)>0$ and sufficiently small $\delta$. This proves Lemma \ref{lemma:partial_t_logPout}.

\subsection{Proof of Lemma \ref{lemma:square_partial_t_logPout}}
\label{sec_app:proof_square_partial_t_logPout}
Starting from
\begin{equation*}
\begin{aligned}
P_\out^{(t)}(v | W,x)&\coloneqq\frac{1}{\sqrt{2\pi\tilde{\Delta}}t}\EE_z\exp\left\{-\frac{1}{2\tilde{\Delta}t}\left(v-\Tr[\Phi S] + (1-t)\sqrt{d} [\tr S - 1]\right.\right. 
\\ &\left.\left.\qquad+ \sqrt{d(1-t)} \Delta \left(\frac{\|z\|^2}{m} - 1\right) 
+\frac{2\sqrt{\Delta d(1-t)}}{m} \sum_{k=1}^m z_{k} \left(\frac{x^\T w_k}{\sqrt{d}}\right) \right)^2\right\},
\end{aligned}
\end{equation*}
we have
\begin{tiny}
\begin{equation}
\begin{aligned}
&\frac{\partial}{\partial t}\log P_\out^{(t)}(v | W,x)\\&=-\frac{1}{2t}+\mathbb{E}_{z\sim P_z}\left[\frac{1}{2\tDelta t^2}\left(v-\Tr[\Phi S] + (1-t)\sqrt{d} [\tr S - 1]+ \sqrt{d(1-t)} \Delta \left(\frac{\|z\|^2}{m} - 1\right) 
+\frac{2\sqrt{\Delta d(1-t)}}{m} \sum_{k=1}^m z_{k} \left(\frac{x^\T w_k}{\sqrt{d}}\right) \right)^2\right]\\
&+\mathbb{E}_{z\sim P_z}\left[\frac{1}{2\tDelta t}\left(v-\Tr[\Phi S] + (1-t)\sqrt{d} [\tr S - 1]\qquad+ \sqrt{d(1-t)} \Delta \left(\frac{\|z\|^2}{m} - 1\right) 
+\frac{2\sqrt{\Delta d(1-t)}}{m} \sum_{k=1}^m z_{k} \left(\frac{x^\T w_k}{\sqrt{d}}\right) \right)\right.\\&\left.\left(-\sqrt{d} [\tr S - 1]-\frac{1}{2}\sqrt{\frac{d}{1-t}} \Delta \left(\frac{\|z\|^2}{m} - 1\right) 
-\frac{1}{m}\sqrt{\frac{\Delta d}{(1-t)}} \sum_{k=1}^m z_{k} \left(\frac{x^\T w_k}{\sqrt{d}}\right) \right)\right]
\end{aligned}
\label{eq:partial_t_logPout}
\end{equation}
\end{tiny}
where the distribution of $z$ is given by
\begin{equation*}
\begin{aligned}
P_z(z)&=\mathcal{C}\exp\left\{-\frac{1}{2}z^2-\frac{1}{2\tilde{\Delta}t}\left(v-\Tr[\Phi S] + (1-t)\sqrt{d} [\tr S - 1]\right.\right. 
\\ &\left.\left.\qquad+ \sqrt{d(1-t)} \Delta \left(\frac{\|z\|^2}{m} - 1\right) 
+\frac{2\sqrt{\Delta d(1-t)}}{m} \sum_{k=1}^m z_{k} \left(\frac{x^\T w_k}{\sqrt{d}}\right) \right)^2\right\},
\end{aligned}
\end{equation*}
with $\mathcal{C}$ a normalization factor. Notie that we have
\begin{equation*}
\sup_{\|u\|=1}P_z(|u^\T z|>\xi)\leq\min\{2\mathcal{C}e^{-\xi^2/2},1\}\leq\min\{2e^{-\xi^2/2\log(2\mathcal{C})},1\}.
\end{equation*}
We also have, by Jensen's inequality:
\begin{tiny}
\begin{equation*}
\begin{aligned}
\log\mathcal{C}&\coloneqq-\log\int Dz\exp\left\{-\frac{1}{2\tilde{\Delta}t}\left(v-\Tr[\Phi S] + (1-t)\sqrt{d} [\tr S - 1]+ \sqrt{d(1-t)} \Delta \left(\frac{\|z\|^2}{m} - 1\right) 
+\frac{2\sqrt{\Delta d(1-t)}}{m} \sum_{k=1}^m z_{k} \left(\frac{x^\T w_k}{\sqrt{d}}\right) \right)^2\right\}\\
&\leq\frac{1}{2\tilde{\Delta}t}\mathbb{E}_{z\sim\mathcal{N}(0,1)}\left[\left(v-\Tr[\Phi S] + (1-t)\sqrt{d} [\tr S - 1] + \sqrt{d(1-t)} \Delta \left(\frac{\|z\|^2}{m} - 1\right)
+\frac{2\sqrt{\Delta d(1-t)}}{m} \sum_{k=1}^m z_{k} \left(\frac{x^\T w_k}{\sqrt{d}}\right)\right)^2\right]\\
&=\frac{1}{2\tilde{\Delta}t}\left((v-\Tr[\Phi S] + (1-t)\sqrt{d} [\tr S - 1])^2+\frac{2(1-t)\Delta^2}{\kappa}+\frac{4(1-t)\Delta}{\kappa}\frac{x^\T Sx}{d}\right).
\end{aligned} 
\end{equation*}
\end{tiny}
Therefore $P_z$ is a sub-Gaussian distribution with variance proxy $\log(2\mathcal{C})$. Its moments are thus bounded by
\begin{equation*}
\mathbb{E}_{z\sim P_z}[\|z\|^p]\leq m(\log(2\mathcal{C}))^pp^{p/2}\eqqcolon mM_p.
\end{equation*}
Taking it into eq. \eqref{eq:partial_t_logPout}, we have
\begin{tiny}
\begin{equation*}
\begin{aligned}
&\left(\frac{\partial}{\partial t}\log P_\out^{(t)}(v | W,x)\right)^2\\
&\leq\frac{9}{4t^2}+\frac{81}{4\tDelta^2 t^4}\left((v-\Tr[\Phi S] + (1-t)\sqrt{d} [\tr S - 1])^2+d(1-t)\Delta^2\left(\frac{M_4}{m}+2M_2+1\right)+4\Delta d(1-t)mM_2\sum_{k=1}^m\left(\frac{x^\T w_k}{\sqrt{d}}\right)^2\right)^2\\
&+\frac{9}{4\tDelta^2t^2}(v-\Tr[\Phi S] + (1-t)\sqrt{d} [\tr S - 1])^2d(\tr S-1)^2+\frac{9}{\tDelta^2t^2}\left(d(\tr S - 1)^2+d\Delta^2\left(\frac{M_4}{m}+2M_2+1\right)\right)^2\\
&+\frac{9}{4\tDelta^2t^2}\left(v-\Tr[\Phi S]+\frac{3}{2}(1-t)\sqrt{d}[\tr S-1]\right)^2\left(\sqrt{\frac{d}{1-t}}\Delta(M_2+1)+2\frac{1}{m}\sqrt{\frac{\Delta d}{1-t}}M_1 \sum_{k=1}^m \left(\frac{x^\T w_k}{\sqrt{d}}\right)\right)^2\\
&\leq\frac{\text{poly}\left(v,\Tr[\Phi S],\tr S, x^\T Sx,\{x^\T w_k\}_{k=1}^m\right)}{t^{12}},
\end{aligned}
\end{equation*}
\end{tiny}
where we uses eq. \eqref{eq:inequality_sum_p}. An important observation is that 
\begin{equation*}
    \EE_{x,W^\star,z,\zeta,Y'}\left\langle\text{poly}\left(v,\Tr[\Phi S],\tr S,x^\T Sx,\{x^\T w_k\}_{k=1}^m\right)\right\rangle   
\end{equation*}
can be bounded by polynomials of $d$. To see this, consider variables $(A,A^\star,B,B^\star)$ (e.g. $(x^\T Sx,x^\T S^\star x)$). 
We have by the Cauchy-Schwarz inequality and the Nishimori identity (Proposition~\ref{prop:nishimori}):
\begin{equation*}
\left|\mathbb{E}\left\langle\sum_{i,j=1}^Ka_{ij}A^iB^j\right\rangle\right|\leq\sum_{i,j=1}^K|a_{ij}|\sqrt{\mathbb{E}\langle A^{2i}\rangle\mathbb{E}\langle B^{2j}\rangle}=\sum_{i,j=1}^K|a_{ij}|\sqrt{\mathbb{E}[(A^\star)^{2i}]\mathbb{E}[ (B^\star)^{2j}]}
\end{equation*}
Therefore, the expectation of polynomials of $A,B$ is bounded by polynomials of $d$ if the moments of $A^*,B^*$ are bounded by polynomials of $d$. As the moments of $v,\Tr[\Phi S^\star],x^\T S^\star x,\{x^\T w_k^\star\}_{k=1}^m$ are bounded by polynomials of $d$, we have
\begin{equation*}
\begin{aligned}
\EE_{x,W^\star,z,\zeta,Y'}\left\langle\left(\frac{\partial}{\partial t}\log P_\out^{(t)}(\tilde{v}(t,x,W^\star,z,\zeta)  | W,x_i)\right)^2\right\rangle&\leq\EE_{x,W^\star,z,\zeta,Y'}\left\langle\frac{\text{poly}(\tilde{v},W,x)}{t^{12}}\right\rangle\\&\leq\frac{\text{poly}(d)}{t_0^{12}},
\end{aligned}
\end{equation*}
which finishes the proof of Lemma \ref{lemma:square_partial_t_logPout}.

\section{Spiked tensor model}\label{sec_app:spike}
\subsection{Main results}
This section presents the results concerning the following ``spiked tensor'' model
\begin{equation*}
    Y=\sqrt{\frac{\lambda}{d^{p-2}}}(S^\star)^{\otimes p}+Z,
\end{equation*}
where $S\sim P_0$ is a symmetric matrix and $Z\in\mcS_d^{\otimes p}$ has i.i.d.\ $\mathcal{N}(0,1)$ elements.
$p\geq2$ is an integer. This model is primarily used for the proof of Theorem~\ref{theo:overlap} in our paper, but it might be of independent interest.

\myskip
Its free entropy is given by
\begin{equation}
f_d^{\text{spike}}\coloneqq \frac{1}{d^2}\mathbb{E}_{S^\star,Z}\left[\log\int P_0(\rd S)e^{-H(S;S^\star,Z)}\right],
\label{eq:spike_entropy}
\end{equation}
where
\begin{equation*}
\begin{aligned}
H(S;S^\star,Z)\coloneqq \frac{1}{2}\sum_{i_1,j_1,\cdots,i_p,j_p}&\left(\frac{\lambda}{2d^{p-2}}S_{i_1,j_1}^2\cdots S_{i_p,j_p}^2-\frac{\lambda}{d^{p-2}}S_{i_1,j_1}S_{i_1,j_1}^\star \cdots S_{i_p,j_p}S_{i_p,j_p}^\star \right.\\&\left.-\sqrt{\frac{\lambda}{d^{p-2}}}Z_{i_1,j_1,\cdots,i_p,j_p}S_{i_1,j_1}\cdots S_{i_p,j_p}\right).
\end{aligned}
\end{equation*}
The RS potential is defined as
\begin{equation}\label{eq:def_fRS_spike}
f^{\text{spike}}_{\RS}(q)\coloneqq -\frac{1}{2}\lambda(p-1)q^p+\psi_{P_0}(2p\lambda q^{p-1})+\frac{1}{4}+\frac{1}{2}p\lambda\rho q^{p-1},
\end{equation}
where $\psi_{P_0}$ is defined in eq. \eqref{eq:psi_0_out}.
The following theorem states our main results regarding the asymptotics of the free entropy.
\begin{theorem}
    \noindent
Under Assumption~\ref{ass:prior},
\begin{equation*}
\lim_{d\to\infty}f^{\text{spike}}_d=\inf_{q\in[0,\rho]}f^{\text{spike}}_{RS}(q).
\end{equation*}
\label{theo:tensor_main}
\end{theorem}
We notice that Theorem~\ref{theo:tensor_main} can be generalized to the rectangular setting (i.e.\ $S \in \bbR^{d \times L}$) in a straightforward manner.
We now focus on the proof of Theorem~\ref{theo:tensor_main}.

\subsection{Interpolation model}
\label{sec:spike_interpolate}
For notational simplicity, we consider here $p=2$, but it is easy to generalize the proof to any $p \geq 2$ (see \cite[Section 3]{barbier2019adaptive}).
We will first prove Theorem~\ref{theo:tensor_main} under Assumption~\ref{assum:S_strong}. All lemmas in Sections~\ref{sec:spike_interpolate} and~\ref{sec:spike_concentration} are under Assumptions~\ref{assum:S_strong}, 
and we then relax this assumption to Assumption~\ref{ass:prior} in Section~\ref{subsec:relax_S_spike}.

\myskip
The interpolation model reads
\begin{equation*}
\begin{aligned}
&Y=\sqrt{\lambda(1-t)}S^\star\otimes S^\star+Z\\
&Y'=\sqrt{d}(\sqrt{R_d(t,\epsilon)}S^\star+Z'),
\end{aligned}
\end{equation*}
with $Z'\sim\GOE(d)$ independent of $Z$. The interpolating free entropy reads
\begin{equation*}
f_{d,\epsilon}(t)\coloneqq \frac{1}{d^2}\mathbb{E}\log\mathcal{Z}_{t,\epsilon}(Y,Y')\coloneqq \frac{1}{d^2}\mathbb{E}\log\int P_0(ds)e^{-H_{t,\epsilon}(s,Y,Y')},
\end{equation*}
where
\begin{equation}\label{eq:def_Hteps_spike}
\begin{aligned}
H_{t,\epsilon}(s,Y,Y')&\coloneqq\sum_{i_1,j_1,i_2,j_2=1}^d(1-t)\lambda\frac{s_{i_1j_1}^2s_{i_2j_2}^2}{2}-\sqrt{(1-t)\lambda}s_{i_1j_1}s_{i_2j_2}Y_{i_1j_1i_2j_2}\\
&+\frac{1}{4}\sum_{i,j=1}^{d}(Y_{ij}'-\sqrt{dR_d(t,\epsilon)}s_{ij})^2.
\end{aligned}
\end{equation}
and 
\begin{equation*}
R_d(t,\epsilon)\coloneqq \epsilon+4\lambda\int_0^t\rd u \, q_d(u,\epsilon),
\end{equation*}
with $\epsilon\in[\iota_d,2\iota_d]$ and $\iota_d\coloneqq \frac{1}{2}d^{-1/8}$. Accordingly, the Gibbs bracket is defined as
\begin{equation*}
\langle g(s)\rangle\coloneqq \frac{1}{\mathcal{Z}_{t,\epsilon}(Y,Y')}\int P_0(\rd s)g(s)e^{-H_{t,\epsilon}(s,Y,Y')}.
\end{equation*}

\begin{lemma}
\label{lemma:spike_f0f1}
\begin{equation*}
\begin{aligned}
&f_{d,\epsilon}(0)=f_d-\frac{1}{4}+O(\iota_d).\\
&f_{d,\epsilon}(1)=\psi_{P_0}\left(4\lambda\int_0^1q(t)\rd t\right)+O(\iota_d).
\end{aligned}
\end{equation*}
\end{lemma}
\begin{proof}[Proof of Lemma~\ref{lemma:spike_f0f1}]
It is analogous to Lemma~\ref{lemma:f0_f1}. We obtain the first equation from 
\begin{equation*}
\left|\frac{\rd f_{d,\epsilon}(0)}{\rd\epsilon}\right|=\frac{1}{4d}\left|\mathbb{E}\sum_{i,j=1}^d\langle\epsilon(S^\star_{ij}-s_{ij})+\sqrt{\frac{1}{\epsilon}}(S^\star_{ij}-s_{ij})Z_{ij}'\rangle\right|\leq M^2
\end{equation*}
and $f_{d,0}(0)=f_d-\frac{1}{4}$. We obtain the second equality from
\begin{equation*}
f_{d,\epsilon}(1)=\psi_{P_0}(R_d(1,\epsilon))
\end{equation*}
and the Lipschitz property of $\psi_{P_0}$ (Lemma~\ref{lemma:Lipschitz_psi}).
\end{proof}
\begin{lemma}
\label{lemma:spike_sum_rule}
\noindent
Denote $f_\RS = f_\RS^{\textrm{spike}}$ given in eq.~\eqref{eq:def_fRS_spike}. Then: 
\begin{equation*}
f_d=f_{\RS}\left(\int_0^1q(t)\rd t\right)-\frac{\lambda}{2}(\mcR_1-\mcR_2-\mcR_3)+O(\iota_d),
\end{equation*}
where
\begin{equation*}
\begin{aligned}
&\mcR_1\coloneqq \int_0^1\rd t\left(q(t)-\int_0^1\rd t'q(t')\right)^2,\\
&\mcR_2\coloneqq \int_0^1\rd t\mathbb{E}\langle(Q-\mathbb{E}\langle Q\rangle)^2\rangle,\\
&\mcR_3\coloneqq \int_0^1\rd t(\mathbb{E}\langle Q\rangle-q(t))^2.
\end{aligned}
\end{equation*}
\end{lemma}
\begin{proof}[Proof of Lemma~\ref{lemma:spike_sum_rule}]
Taking the derivative w.r.t $t$ and integrating by parts, we have
\begin{equation*}
\frac{\rd f_{d,\epsilon}(t,\epsilon)}{\rd t}=-\frac{\lambda}{2d^2}\sum_{i_1,j_1,i_2,j_2}\mathbb{E}[S^\star_{i_1j_1}S^\star_{i_2j_2}\langle s_{i_1j_1}s_{i_2j_2}\rangle]-\frac{\lambda q(t)}{d}\mathbb{E}\sum_{i,j=1}^d(S^\star_{ij}-\langle s_{ij}\rangle)^2.
\end{equation*}
Using the Nishimori identity (Proposition~\ref{prop:nishimori}), we have
\begin{equation*}
\frac{\rd f_{d,\epsilon}(t,\epsilon)}{\rd t}=-\frac{\lambda}{2}\mathbb{E}\langle Q^2\rangle-\lambda q(t)(\rho-\mathbb{E}\langle Q\rangle).
\end{equation*}
By Lemma~\ref{lemma:spike_f0f1}, we have 
\begin{equation*}
\begin{aligned}
f_d&=\psi_{P_0}\left(4\lambda\int_0^1q(t)\rd t\right)+\frac{1}{4}+\int_0^1\rd t \left[\frac{\lambda}{2}\mathbb{E}\langle Q^2\rangle+\lambda q(t)(\rho-\mathbb{E}\langle Q\rangle)\right]+O(\iota_d),\\
&=f_{\RS}\left(\int_0^1q(t)\rd t\right)+\frac{\lambda}{2}\left(\int_0^1\rd tq(t)\right)^2+\frac{\lambda}{2}\int_0^1\rd t\mathbb{E}\langle Q^2\rangle-\lambda\int_0^1\rd tq(t)\mathbb{E}\langle Q\rangle+O(\iota_d),
\end{aligned}
\end{equation*}
which gives the desired result.
\end{proof}

\myskip
We will need the following concentration lemma, which will be proven in Section~\ref{sec:spike_concentration}.
\begin{lemma}
    \noindent
There exists a constant $C(M,\kappa) > 0$ such that
\begin{equation*}
\frac{1}{\iota_d}\int_{\iota_d}^{2\iota_d}\rd\epsilon\int_0^1\rd t\mathbb{E}\langle(Q-\mathbb{E}\langle Q\rangle)^2\rangle\leq\frac{C(M,\kappa)}{d^{1/4}}.
\end{equation*}
\label{lemma:spike_concentration_overlap}
\end{lemma}
Combining Lemmas~\ref{lemma:spike_sum_rule} and~\ref{lemma:spike_concentration_overlap}, we can obtain Theorem~\ref{theo:tensor_main} under Assumption~\ref{assum:S_strong}, following the standard choice of interpolation (\cite{barbier2019adaptive}, i.e. $q(t)=q$ and $q(t)=\mathbb{E}\langle Q\rangle$), by repeating the argument used in the proof of Theorem~\ref{thm:fentropy_symmetric}. 
Lemma~\ref{lemma:spike_relaxV} allows then to relax Assumption~\ref{assum:S_strong}.

\subsection{Concentration of the overlap: proof of Lemma~\ref{lemma:spike_concentration_overlap}}
\label{sec:spike_concentration}
We first prove the concentration of the free entropy. We can rewrite the free entropy as
\begin{equation*}
\frac{1}{d^2}\log \mathcal{Z}_{t,\epsilon}=\frac{1}{d^2}\log\hat{\mathcal{Z}}_{t,\epsilon}-\frac{1}{4d}\sum_{i,j=1}^{d}(Z_{ij}')^2,
\end{equation*}
where
\begin{equation*}
\frac{1}{d^2}\log\hat{\mathcal{Z}}_{t,\epsilon}\coloneqq \frac{1}{d^2}\log\int P_0(\rd s)e^{-\hat{H}_t(s)},
\end{equation*}
and
\begin{equation*}
\begin{aligned}
\hat{H}_t(s)&\coloneqq\frac{1}{4}\sum_{i_1,j_1,i_2,j_2=1}^d(1-t)\lambda\frac{s_{i_1j_1}^2s_{i_2j_2}^2}{2}-(1-t)\lambda s_{i_1j_1}s_{i_2j_2}S^\star_{i_1j_1}S^\star_{i_2j_2}\\&-\sqrt{(1-t)\lambda}s_{i_1j_1}s_{i_2j_2}Z_{i_1j_1i_2j_2}+\frac{d}{4}\sum_{i,j=1}^{d}\left[R_1(t)(S_{ij}^\star -s_{ij})^2+2Z_{ij}'\sqrt{R_1(t)}(S_{ij}^\star -s_{ij})\right].
\end{aligned}
\end{equation*}
The concentration of the free entropy results from the following lemma.
\begin{lemma}
\begin{equation*}
\text{Var}\left(\frac{1}{d^2}\log\hat{\mathcal{Z}}_{t,\epsilon}\right)\leq\frac{C(\varphi,M,\alpha,\kappa)}{d^2}.
\end{equation*}
\label{lemma:spike_concentration_free_entropy}
\end{lemma}
\begin{proof}[Proof of Lemma~\ref{lemma:spike_concentration_free_entropy}]
We first consider $g\coloneqq \log\hat{\mathcal{Z}}_{t,\epsilon}/d^2$ as a function of $Z,Z'$. We have
\begin{equation*}
\begin{aligned}
\sum_{i_1,j_1,i_2,j_2=1}^d\left(\frac{\partial g}{\partial Z_{i_1j_1i_2j_2}}\right)^2&=d^{-4}(1-t)\lambda\sum_{i_1,j_1,i_2,j_2=1}^d\langle s_{i_1j_1}s_{i_2j_2}\rangle^2\\&\leq d^{-4}(1-t)\lambda\sum_{i_1,j_1,i_2,j_2=1}^d\langle s_{i_1j_1}^2s_{i_2j_2}^2\rangle\leq d^{-2}M^4,
\end{aligned}
\end{equation*}
where the first inequality is from Jensen's inequality, and the last inequality is a consequence of 
\begin{equation*}
    \sum_{i_1,j_1,i_2,j_2=1}^ds_{i_1j_1}^2s_{i_2j_2}^2=\left(\sum_{i,j=1}^ds_{ij}^2\right)^2\leq d^2M^4.
\end{equation*}
We also have
\begin{equation*}
\begin{aligned}
\sum_{i,j=1}^{d}\left(\frac{\partial g}{\partial Z_{ij}'}\right)^2=d^{-2}R(t)\sum_{i,j=1}^{d}(S_{ij}^\star -\langle s_{ij}\rangle)^2\leq 4d^{-1}KM^2,
\end{aligned}
\end{equation*}
where $K\coloneqq 1+\lambda\rho$ upper bounds $R(t)$. By Lemma~\ref{lemma:gaussian-poincare} we obtain
\begin{equation}
\mathbb{E}\left[\left(\frac{1}{d^2}\log\hat{\mathcal{Z}}_{t,\epsilon}-\frac{1}{d^2}\mathbb{E}_{Z,Z'}\log\hat{\mathcal{Z}}_{t,\epsilon}\right)^2\right]\leq\frac{C(M)}{d^2}.
\label{eq:spike_concentration_z}
\end{equation}
Next we write $S^\star=O\Lambda^\star O^T$ with $O$ drawn from the Haar measure on $\mcO(d)$, independently of $\Lambda^\star $. We denote $\ts\coloneqq O^TsO$ and $\tZ=O^TZO$, $\tZ'=O^T\tZ'O$. 
Notice that we have $\Lambda^\star = \Diag(\{\Lambda^\star_i\}_{i=1}^d)$: we will abuse a bit notations and denote $\Lambda^\star_{ij} \coloneqq \Lambda^\star_i \delta_{ij}$.
In this way the Hamitonian reads
\begin{equation*}
\begin{aligned}
\hat{H}_t(\ts)&\coloneqq\frac{1}{4}\sum_{i_1,j_1,i_2,j_2=1}^d\left[(1-t)\lambda\frac{\ts_{i_1j_1}^2\ts_{i_2j_2}^2}{2}-(1-t)\lambda \ts_{i_1j_1}\ts_{i_2j_2}\Lambda^\star _{i_1j_1}\Lambda^\star _{i_2j_2}\right.\\
&\left.-\sqrt{(1-t)\lambda}\ts_{i_1j_1}\ts_{i_2j_2}\tZ_{i_1j_1i_2j_2}\right]+\frac{d}{4}\sum_{i,j=1}^{d}(R_1(t)(\Lambda_{ij}^\star -\ts_{ij})^2+2\tZ_{ij}'\sqrt{R_1(t)}(\Lambda_{ij}^\star -\ts_{ij})).
\end{aligned}
\end{equation*}
Note that the distribution of $\ts$ is the same as that of $s$ (by rotation invariance of $P_0$), so we have
\begin{equation*}
\mathbb{E}_{O,Z,Z'}\frac{1}{d^2}\log\hat{\mathcal{Z}}_{t,\epsilon}=\mathbb{E}_{\tZ,Z'}\frac{1}{d^2}\log\int P_0(\rd \ts)Dwe^{-\hat{H}_t(\ts,w)}.
\end{equation*}
We can now use a very similar argument to the one detailed around eq.~\eqref{eq:tildeZ_independent_O}:
as $(Z,Z') \deq (\tZ, \tZ')$ (meaning equality in law), we have
\begin{equation}
\begin{aligned}
&\mathbb{E}\left[\left(\frac{1}{d^2}\mathbb{E}_{O,Z,Z'}\log\hat{\mathcal{Z}}_{t,\epsilon}-\frac{1}{d^2}\mathbb{E}_{Z,Z'}\log\hat{\mathcal{Z}}_{t,\epsilon}\right)^2\right]\\
&\leq\mathbb{E}\left[\left(\frac{1}{d^2}\mathbb{E}_{\tZ,\tZ'}\log\int P_0(\rd \ts)Dwe^{-\hat{H}_t(\ts,w)}-\frac{1}{d^2}\log\int P_0(\rd \ts)Dwe^{-\hat{H}_t(\ts,w)}\right)^2\right],\\
&\leq\frac{C(M)}{d^2}.
\end{aligned}
\label{eq:spike_concentration_O}
\end{equation}
We used eq.~\eqref{eq:spike_concentration_z} in the last inequality.
Finally we consider $g\coloneqq \mathbb{E}_{O,Z,Z'}\frac{1}{d^2}\log\hat{\mathcal{Z}}_{t,\epsilon}$ as a function of $\Lambda^\star $. We have
\begin{equation*}
\begin{aligned}
\sum_{i=1}^d\left(\frac{\partial g}{\partial\Lambda^\star _i}\right)^2&=d^{-4}\sum_{i=1}^d\left(\mathbb{E}_{O,Z,Z'}\left\langle\frac{\partial\hat{H}_t}{\partial\Lambda_i^\star}\right\rangle\right)^2,\\
&=d^{-4}\sum_{i=1}^d\left(\mathbb{E}_{O,Z,Z'}\left\langle - \frac{(1-t)\lambda}{2}\sum_{j=1}^n\ts_{ii}\ts_{jj}\Lambda^\star_{j}+\frac{dR_1(t)}{2}(\Lambda^\star _i-\ts_{ii})+2\tZ_{ii}'\sqrt{dR_1(t)}\right\rangle\right)^2,\\
&\leq 9(I_1+I_2+I_3).
\end{aligned}
\end{equation*}
The first term is
\begin{equation*}
\begin{aligned}
I_1&\coloneqq \frac{\lambda^2}{4} d^{-4}\sum_{i=1}^d\left(\mathbb{E}_{O,Z,Z'}\left\langle\sum_{j=1}^d\ts_{ii}\ts_{jj}\Lambda^\star _{j}\right\rangle\right)^2\\
&\leq \frac{\lambda^2}{4}d^{-2}M^4\sum_{i=1}^d\left(\mathbb{E}_{O,Z,Z'}\left\langle\ts_{ii}\right\rangle\right)^2\leq\frac{\lambda^2}{4} d^{-2}M^4\mathbb{E}_{O,Z,Z'}\left\langle\sum_{i=1}^d\ts_{ii}^2\right\rangle\leq\frac{\lambda^2}{4} d^{-1}M^6,
\end{aligned}
\end{equation*}
where we use $\left|\sum_{j=1}^d\ts_{jj}\Lambda^\star _{jj}\right|\leq \sqrt{\sum_{j=1}^d(\Lambda^\star _{jj})^2}\sqrt{\sum_{j=1}^d\ts_{jj}^2}\leq dM^2$ for the first inequality, and Jensen's inequality for the second inequality.
The second and third terms $I_2, I_3$ are the same as in Lemma~\ref{lemma:concentration_free_entropy2}, see eqs.~\eqref{eq:bound_I2}, \eqref{eq:bound_I3}. 
Combining these three terms, we have
\begin{equation}
\mathbb{E}\left[\left(\frac{1}{d^2}\mathbb{E}_{O,Z,Z'}\log\hat{\mathcal{Z}}_{t,\epsilon}-\frac{1}{d^2}\mathbb{E}\log\hat{\mathcal{Z}}_{t,\epsilon}\right)^2\right]\leq\frac{C(M,\kappa)}{d^2}
\label{eq:spike_concentration_S}
\end{equation}
by Lemma~\ref{lemma:poincare}. We finish the proof by combining eqs. \eqref{eq:spike_concentration_z}, \eqref{eq:spike_concentration_O} and \eqref{eq:spike_concentration_S}.
\end{proof}
We now sketch how we can end the proof of Lemma~\ref{lemma:spike_concentration_overlap} from Lemma~\ref{lemma:spike_concentration_free_entropy}.
Let (recall eq.~\eqref{eq:def_Hteps_spike}):
\begin{equation*}
\mathcal{L}\coloneqq \frac{1}{d^2}\frac{dH_{t,\epsilon}}{dR}=\frac{1}{2d}\sum_{i,j=1}^d\left(\frac{s_{ij}^2}{2}-s_{ij}S_{ij}^\star -\frac{s_{ij}Z_{ij}'}{2\sqrt{R}}\right).
\end{equation*}
We can now observe that $\mcL$ is exactly the same quantity analyzed in Section~\ref{sec:overlap_concentration}:
directly repeating the analysis given there to this setting yields Lemma~\ref{lemma:spike_concentration_overlap}.

\subsection{Relaxation of Assumption~\ref{assum:S_strong}}
\label{subsec:relax_S_spike}

We now show how to relax Assumption~\ref{assum:S_strong} to Assumption~\ref{ass:prior} in the proof of Theorem~\ref{theo:tensor_main}.
\begin{lemma}
    \noindent
Denote $f_d^{\text{spike}}(\mathcal{V})$ to be the free entropy of eq.~\eqref{eq:spike_entropy} corresponding to the prior $P_0$ with potential $\mathcal{V}$, and denote $\tilde{\mathcal{V}}_M$ the truncation of $\mathcal{V}$ to $[-M, M]$, with $\tilde{\mathcal{V}}_M(x) = +\infty$ if $|x| > M$. 
Notice that then the prior with potential $\tilde{\mathcal{V}}_M$ satisfies Assumption~\ref{assum:S_strong}. Suppose that $V$ satisfies Assumption~\ref{ass:prior}. Then, for $M > 0$ large enough, we have
\begin{equation*}
    \lim_{d\to\infty}|f_d^{\text{spike}}(\mathcal{V})-f_d^{\text{spike}}(\tilde{\mathcal{V}}_M)|=0.
\end{equation*}
\label{lemma:spike_relaxV}
\end{lemma}
\begin{proof}[Proof of Lemma~\ref{lemma:spike_relaxV}]
Following Lemma~\ref{lemma:relaxV}, we consider the interpolation
\begin{equation*}
\begin{aligned}
&Y_1=\sqrt{\lambda t}S_1^\star \otimes S_1^\star +Z_1,\\
&Y_2=\sqrt{\lambda(1-t)}S_2^\star \otimes S_2^\star +Z_2,
\end{aligned}
\end{equation*}
The joint distribution of $S_1^\star ,S_2^\star $ is the same as in Lemma~\ref{lemma:relaxV}. Specially, $S_1^\star$ satisfies Assumption~\ref{ass:prior} and 
$S_2^\star$ satisfies Assumption~\ref{assum:S_strong}. Its free entropy is defined as
\begin{equation*}
f_d(t)\coloneqq \frac{1}{d^2}\mathbb{E}\log\mathcal{Z}_{t}(Y_1,Y_2)\coloneqq \frac{1}{d^2}\mathbb{E}\log\int P_0(\rd s_1,\rd s_2)e^{-H_{t}(s_1,s_2,Y_1,Y_2)},
\end{equation*}
where
\begin{equation*}
\begin{aligned}
H_{t}(s_1,s_2,Y_1,Y_2)&\coloneqq\frac{1}{4}\sum_{i_1,j_1,i_2,j_2=1}^\rd t\lambda\frac{s_{1,i_1j_1}^2s_{1,i_2j_2}^2}{2}-\sqrt{t\lambda}s_{1,i_1j_1}s_{1,i_2j_2}Y_{1,i_1j_1i_2j_2}\\
&+\frac{1}{4}\sum_{i_1,j_1,i_2,j_2=1}^d(1-t)\lambda\frac{s_{2,i_1j_1}^2s_{2,i_2j_2}^2}{2}-\sqrt{(1-t)\lambda}s_{2,i_1j_1}s_{2,i_2j_2}Y_{2,i_1j_1i_2j_2}.
\end{aligned}
\end{equation*}
Thus we have $f_d^{\text{spike}}(\mathcal{V})=f(1)$ and $f_d^{\text{spike}}(\tilde{\mathcal{V}}_M)=f(0)$. Accordingly, the Gibbs bracket is defined as
\begin{equation*}
\langle g(s_1,s_2)\rangle\coloneqq \frac{1}{\mathcal{Z}_{t}(Y_1,Y_2)}\int P_0(\rd s_1,\rd s_2)g(s_1,s_2)e^{-H_{t}(s_1,s_2,Y_1,Y_2)}.
\end{equation*}
Taking the derivative and integrating by part, we obtain
\begin{equation*}
f_d'(t)=-\frac{\lambda}{4d^2}\mathbb{E}\left[\sum_{i_1,j_1,i_2,j_2=1}^dS^\star_{1,i_1j_1}S^\star_{1,i_2j_2}\langle s_{1,i_1j_1}s_{1,i_2j_2}\rangle-S^\star_{2,i_1j_1}S^\star_{2,i_2j_2}\langle s_{2,i_1j_1}s_{2,i_2j_2}\rangle\right],
\end{equation*}
and thus
\begin{equation*}
\begin{aligned}
|f_d'(t)|&\leq \frac{\lambda}{4d^2}\mathbb{E}\left[\left|\sum_{i_1,j_1,i_2,j_2=1}^dS^\star_{1,i_1j_1}S^\star_{1,i_2j_2}\langle s_{1,i_1j_1}s_{1,i_2j_2}\rangle-S^\star_{1,i_1j_1}S^\star_{1,i_2j_2}\langle s_{2,i_1j_1}s_{2,i_2j_2}\rangle\right|\right.\\&\left.\qquad+\left|\sum_{i_1,j_1,i_2,j_2=1}^dS^\star_{1,i_1j_1}S^\star_{1,i_2j_2}\langle s_{2,i_1j_1}s_{2,i_2j_2}\rangle-S^\star_{2,i_1j_1}S^\star_{2,i_2j_2}\langle s_{2,i_1j_1}s_{2,i_2j_2}\rangle\right|\right]\\
&\leq\frac{\lambda}{4d^2}\left(\mathbb{E}\|S_1^\star \otimes S_1^\star\|^2\mathbb{E}\langle\|s_1\otimes s_1-s_2\otimes s_2\|^2\rangle\right)^{1/2}\\&\qquad+\frac{\lambda}{4d^2}\left(\mathbb{E}\|S_1^\star \otimes S_1^\star -S_2^\star \otimes S_2^\star\|^2\mathbb{E}\langle\|s_2\otimes s_2\|^2\rangle\right)^{1/2}\\
&=\frac{\lambda}{4d^2}((\EE[\|\Lambda_1^\star\|^4])^{1/2}+(\EE[\|\Lambda_2^\star\|^4])^{1/2})\left(\mathbb{E}\|S_1^\star \otimes S_1^\star -S_2^\star \otimes S_2^\star\|^2\right)^{1/2},
\end{aligned}
\end{equation*}
where we used the Cauchy-Schwarz inequality and the Nishimori identity (Proposition~\ref{prop:nishimori}).
Recall that $\Lambda_1^\star$, $\Lambda_2^\star$ are the diagonal matrices of eigenvalues of $S_1^\star, S_2^\star$.
The tensor product can be expanded as
\begin{equation*}
\begin{aligned}
\|S_1^\star \otimes S_1^\star -S_2^\star \otimes S_2^\star \|^2&=\|(O\otimes O)(\Lambda_1^\star \otimes\Lambda_1^\star -\Lambda_2^\star \otimes\Lambda_2^\star )(O^T\otimes O^T)\|^2\\
&=\|\Lambda_1^\star \otimes\Lambda_1^\star -\Lambda_2^\star \otimes\Lambda_2^\star \|^2\\
&\leq 2(\|\Lambda_1^\star\|^2+\|\Lambda_2^\star\|^2)\|\Lambda_1^\star-\Lambda_2^\star\|^2,
\end{aligned}
\end{equation*}
which gives (using Proposition~\ref{prop:properties_P0_sym} and the Cauchy-Schwarz inequality)
\begin{equation*}
|f_d'(t)|\leq C(P_0)\lambda \left(\frac{1}{d^2}\mathbb{E}\|\Lambda_1^\star -\Lambda_2^\star\|^4\right)^{1/4}.
\end{equation*}
Notice that we used Proposition~\ref{prop:properties_P0_sym}(iv) that gives that $\EE[\|\Lambda^\star\|^{2p}] \leq C(P_0, p) \cdot d^p$.
Following the arguments in Lemma~\ref{lemma:relaxV}, and taking the notations from this paragraph, we 
obtain for a well-chosen coupling of $(S_1^\star, S_2^\star)$:
\begin{equation*}
|f_d'(t)|\leq C(P_0)\lambda \left(\EE[W_2(\mu_{S_1}, \mu_{S_2})^4]\right)^{1/4}.
\end{equation*}
The arguments detailed in the proof of Lemma~\ref{lemma:relaxV} can be repeated to obtain 
\begin{equation*}
\lim_{d \to \infty} \EE[W_2(\mu_{S_1}, \mu_{S_2})^4] = 0,
\end{equation*}
which gives then
$\lim_{d\to\infty}|f_d(0)-f_d(1)|=0$, and proves Lemma~\ref{lemma:spike_relaxV}.
\end{proof}
Finally it is worth noticing that for a general $p$, we will need to upper bound $\mathbb{E}\|\Lambda^\star\|^{2p}$: this bound is given by Proposition~\ref{prop:properties_P0_sym}(iv) as we mentioned, so that the proof arguments sketched above directly generalize to any $p \geq 3$.

\section{Technical lemmas}\label{sec_app:technical_lemmas}
This section contains several technical lemmas, which are used throughout the proofs of our main results.

\subsection{Nishimori identity}

We state here the Nishimori identity, a classical consequence of Bayes optimality.
\begin{proposition}[Nishimori identity]\label{prop:nishimori}
    \noindent
    Let $(X,Y)$ be random variables on a Polish space E. Let $k \in \bbN^\star$ and 
    $(X_1,\cdots,X_k)$ i.i.d.\ random variables sampled from the conditional distribution $\bbP(X|Y)$.
    We denote $\langle \cdot \rangle_Y$ the average with respect to $\bbP(X|Y)$, and $\EE[\cdot]$ the average with respect to
    the joint law of $(X,Y)$. Then, for all $f : E^{k+1} \to \bbK$ continuous and bounded:
    \begin{align}
        \EE [\langle f(Y,X_1,\cdots,X_k) \rangle_Y ] &= \EE [\langle (Y,X_1,\cdots, X_{k-1},X) \rangle_Y ].
    \end{align}
\end{proposition}
\begin{proof}[Proof of Proposition \ref{prop:nishimori}]
    The proposition arises as a trivial consequence of Bayes' formula:
   \begin{align*}
    \EE [\langle f(Y,X_1,\cdots,X_{k-1},X) \rangle_Y ] &= \EE_Y \EE_{X|Y} [\langle f(Y,X_1,\cdots,X_{k-1},X) \rangle_Y ],\\
    &= \EE_Y [\langle f(Y,X_1,\cdots,X_k) \rangle_Y ].
   \end{align*}
\end{proof}

\subsection{Potential truncation}
\begin{lemma}[Potential Truncation]
\label{lemma:truncation}
\noindent
Let the potential $\mathcal{V}$ satisfy Assumption \ref{ass:prior}, and $\tilde{\mathcal{V}}_M$ be the truncation of $\mathcal{V}$ to $[-M, M]$, according to Assumption \ref{assum:S_strong}, i.e.\ $\tilde{\mathcal{V}}_M(x) = +\infty$ for $|x| > M$.
There is $M > 0$ large enough (depending only on $\mathcal{V}$) such that
if $S$ is sampled according to a prior $P_0$ with potential $\tilde{\mathcal{V}}_M$, then $\mu_S$ almost surely weakly converges to $\mu_0$.
\end{lemma}
\begin{proof}[Proof of Lemma \ref{lemma:truncation}]
An important observation is that
\begin{equation*}
\bbP_{\tilde{\mathcal{V}}}(\{\lambda_i\}_{i=1}^d)\propto\bbP_\mathcal{V}(\{\lambda_i\}_{i=1}^d)\indi\{\{\lambda_i\}_{i=1}^d\subset[-M,M]\},
\end{equation*}
where $\bbP_{\tilde{\mathcal{V}}},\bbP_\mathcal{V}$ are the joint eigenvalue distributions under the priors with potential $\tilde{\mathcal{V}}_M,\mathcal{V}$, respectively.
Let $\text{dist}$ denote a distance that metrizes the weak topology, then \cite[Theorem 2.6.1]{anderson2010introduction} shows that
\begin{equation*}
\bbP_V(\text{dist}(\mu_S,\mu_0)>\epsilon)\leq e^{-C(\epsilon)d^2},
\end{equation*}
for any $\epsilon>0$ and some $C(\epsilon)>0$. Thus,
\begin{equation*}
\bbP_{\tilde{\mathcal{V}}}(\text{dist}(\mu_S,\mu_0)>\epsilon)=\frac{\bbP_\mathcal{V}(\text{dist}(\mu_S,\mu_0)>\epsilon\And\{\lambda_i\}_{i=1}^d\subset[-M,M])}{\bbP_V(\{\lambda_i\}_{i=1}^d\subset[-M,M])}\leq 2e^{-C(\epsilon)d^2}
\end{equation*}
for $d$ large enough and a large enough constant $M > 0$, where we used Lemma \ref{lemma:boundedness_eigenvalues}. By the Borel–Cantelli lemma, we get
\begin{equation*}
\bbP_{\tV}(\limsup_{d\to\infty}\text{dist}(\mu_S,\mu_0)>\epsilon)=0,
\end{equation*}
which finishes the proof.
\end{proof}

\subsection{Poincaré inequality for rotionally invariant priors}\label{subsec_app:poincare}

The following result is critical in our analysis, and is a direct consequence of existing results~\cite{anderson2010introduction,chafai2020poincare}.
\begin{lemma}[Poincaré Inequality for Rotionally Invariant Priors]
    \noindent
Under Assumption \ref{ass:prior}, denote the eigenvalues of $S$ to be $\{\lambda_i\}_{i=1}^d$ and their joint distribution to be $P_S$. Let $g:\bbR^{d}\to\bbR$ a function in $L^2(P_S)$ whose weak derivative (in the sense of distributions) belongs to $L^2(P_S)$. 
Then:
\begin{equation}
\text{Var}(g(\{\lambda_i\}_{i=1}^d))\leq \frac{c}{\kappa d}\mathbb{E}\left[\sum_{i=1}^d\left(\frac{\partial g}{\partial \lambda_{i}}\right)^2\right],
\label{eq:Poincare}
\end{equation}
where $c$ is the constant given in Assumption \ref{ass:prior} (note that one should take $\kappa = 1$ for the first case of Assumption~\ref{ass:prior}). The same holds under Assumption \ref{ass:prior_rec}, replacing the eigenvalues $\{\lambda_i\}_{i=1}^d$ by the singular values $\{\sigma_i\}_{i=1}^d$.
\label{lemma:poincare}
\end{lemma}

\begin{proof}[Proof of Lemma \ref{lemma:poincare}]
For the first case of Assumption~\ref{ass:prior}, the joint eigenvalue distribution can be written as
\begin{equation*}
P_S(\{\lambda_i\}_{i=1}^d)\propto e^{-d\sum_{i=1}^{d}\mathcal{V}(\lambda_i)}\Pi_{i<j}(\lambda_i-\lambda_j),
\end{equation*}
where $\lambda_1>\lambda_2>\cdots>\lambda_d$. By \cite[Theorem 1.1]{chafai2020poincare} we obtain the desired result. A brief explanation is that we can rewrite the joint distribution as
\begin{equation*}
P_S(\{\lambda_i\}_{i=1}^d)\propto e^{-\left(d\sum_{i=1}^{d}\mathcal{V}(\lambda_i)-\sum_{i<j}\log(\lambda_i-\lambda_j)\right)},
\end{equation*}
where the term inside the bracket is strictly convex in $\mathcal{I}$, so that it satisfies the Bakry-Emery condition (see e.g. \cite[Definition 4.4.16]{anderson2010introduction}). Such a condition is known to imply strong concentration inequalities, notably log-Sobolev inequalities, but also the Poincaré inequality that we require here.

\myskip
Tackling the second case in Assumption~\ref{ass:prior} is similar, because the joint eigenvalue distribution can be written as 
\begin{equation*}
P_S(\{\lambda_i\}_{i=1}^{\lceil\kappa d\rceil})\propto e^{-\left(d\sum_{i=1}^{\lceil\kappa d\rceil}\tilde{\mathcal{V}}(\lambda_i)-\sum_{i<j}\log(\lambda_i-\lambda_j)\right)},
\end{equation*}
where $\tilde{\mathcal{V}}(s)\coloneqq \mathcal{V}(s)-(1-\frac{\lceil\kappa d\rceil}{d})\log s$. The term inside the bracket is strictly convex  in $\mcB$ under Assumption \ref{ass:prior}, so we obtain again the Poincaré inequality from~\cite[Theorem~1.1]{chafai2020poincare}. 

\myskip
In the rectangular case, the joint singular value distribution reads~\cite{anderson2010introduction}
\begin{equation}\label{eq:joint_svalues}
P_S(\{\sigma_i\}_{i=1}^{\lceil\kappa d\rceil})\propto e^{-\left(d\sum_{i=1}^{\lceil\kappa d\rceil}\tilde{\mathcal{V}}(\sigma_i^2)-\sum_{i<j}\log(\sigma_i^2-\sigma_j^2)\right)},
\end{equation}
where $\tilde{\mathcal{V}}(s)\coloneqq \mathcal{V}(s)-2(1-\frac{\lceil\kappa d\rceil}{d})\log s$. 
Under Assumption \ref{ass:prior_rec}, the term inside the bracket in eq.~\eqref{eq:joint_svalues} is strictly convex in $\{\sigma_i\}$, so again we obtain the desired result.
\end{proof}

\subsection{Central limit theorem for simple statistics}

The following lemma clarifies the behavior of some simple statistics of $S$ and $\Phi$ under our assumptions.
\begin{lemma}[Central Limit Theorem]
\label{lemma:CLT_new}
\noindent We have (recall $\tr(\cdot) \coloneqq (1/d) \Tr[\cdot]$):
\begin{itemize}
    \item[$(a)$]Under Assumption~\ref{ass:prior}, $\EE[\tr S^2] \to \rho$ as $d \to \infty$. 
    \item[$(b)$] Under Assumptions~\ref{ass:prior} and \ref{assum:CLT}, $\Tr[\Phi S]$ converges in distribution to $\mathcal{N}(0,2\rho)$ for $\Phi\sim P_\Phi$ and $S\sim P_0$.
    \item[$(c)$]Under Assumption~\ref{ass:prior}, for $S \sim P_0$: 
    \begin{equation*}
        \EE[(\Tr[S^2] - \EE[\Tr S^2])^2] = \mcO_d(1).
    \end{equation*}
    In particular, $\EE[(\tr[S^2] - \rho)^2] = \smallO(1)$.
\end{itemize}
\end{lemma}
\begin{proof}[Proof of Lemma \ref{lemma:CLT_new}]
    Point $(a)$ was already proven in Proposition~\ref{prop:properties_P0_sym}.
 We now show $(b)$. For any bounded Lipschitz $\psi$ and any $M \geq 0$, we have, with $G \sim \GOE(d)$: 
\begin{align*}
    &|\EE[\psi(\Tr[\Phi S])]-\EE[\psi(\Tr[GS])]| \\ 
    &\leq 2 \|\psi\|_\infty \bbP(\|S\|_\op > M) + 
    |\EE[\indi\{\|S\|_\op \leq M\}(\psi(\Tr[\Phi S])-\psi(\Tr[GS]))]|.
\end{align*} 
By Proposition~\ref{prop:properties_P0_sym}(ii), $ \bbP(\|S\|_\op > M) \to 0$ as $d \to \infty$ for all $M > 0$ large enough.
Furthermore:
\begin{align*}
    |\EE[\indi\{\|S\|_\op \leq M\}(\psi(\Tr[\Phi S])-\psi(\Tr[GS]))]|
    &\leq \EE[\indi\{\|S\|_\op \leq M\}|\EE_{\Phi, G}(\psi(\Tr[\Phi S])-\psi(\Tr[GS]))| ], \\
    &\leq \sup_{\|S\|_\op \leq M}|\EE_{\Phi, G}(\psi(\Tr[\Phi S])-\psi(\Tr[GS]))|.
\end{align*}
The right-hand side of this last inequality goes to $0$ as $d \to \infty$ by Assumption~\ref{assum:CLT}.
In the end:
\begin{equation*}
   \lim_{d \to \infty} |\EE[\psi(\Tr[\Phi S])]-\EE[\psi(\Tr[GS])]| = 0.
\end{equation*}
Since $\Tr[GS] \sim \mcN(0, 2 \tr[S^2])$, we get point $(b)$ using $(a)$.
Point $(c)$ is a consequence the Poincaré inequality of Lemma~\ref{lemma:poincare}: using it with $g(\{\lambda_i\}) = \sum_i \lambda_i^2$, we get:
\begin{align*}
    \Var[\Tr[S^2]] \leq 4 \, \frac{c}{d} \, \EE\left[\sum_{i=1}^d \lambda_i^2\right] = \mcO(1), 
\end{align*}
which ends the proof using point $(a)$.
\end{proof} 

\subsection{Properties of Gaussian matrices}

The following result is a classical property of Gaussian random variables, here adapted to $\GOE(d)$ matrices.
\begin{lemma}[Gaussian Integral by Part]
    \noindent
Let $G\sim\text{GOE}(d)$ and $g:\bbR^{d\times d}\to\bbR^{d\times d}$ to be a continuous and differentiable function with $g_{ij}=g_{ji}$ and $\frac{\partial g_{ij}}{\partial G_{ij}}=\frac{\partial g_{ij}}{\partial G_{ji}}$ for $1\leq i,j\leq d$, then we have
\begin{equation*}
\mathbb{E}_G[\Tr[Gg(G)]]=\frac{2}{d}\mathbb{E}_G\left[\sum_{i,j=1}^d\frac{\partial g_{ij}}{\partial G_{ij}}\right].
\end{equation*}
\label{lemma:stein}
\end{lemma}
\begin{proof}[Proof of Lemma \ref{lemma:stein}]
We have
\begin{equation*}
\begin{aligned}
\mathbb{E}_G[\Tr[Gg(G)]]&=\sum_{i=1}^d\mathbb{E}_G[G_{ii}g_{ii}(G)]]+2\sum_{1\leq i<j\leq d}\mathbb{E}_G[G_{ij}g_{ij}(G)]]\\
&=\frac{2}{d}\sum_{i=1}^d\mathbb{E}_G\left[\frac{\partial g_{ii}}{\partial G_{ii}}\right]+\frac{2}{d}\sum_{1\leq i<j\leq d}\mathbb{E}_G\left[\frac{\partial g_{ij}}{\partial G_{ij}}+\frac{\partial g_{ij}}{\partial G_{ji}}\right]\\
&=\frac{2}{d}\mathbb{E}_G\left[\sum_{i,j=1}^d\frac{\partial g_{ij}}{\partial G_{ij}}\right],
\end{aligned}
\end{equation*}
where we use the standard Gaussian integral by part (Stein's lemma) with respect to $\sqrt{\frac{d}{2}}G_{ii}$ and $\sqrt{d}G_{ij}$.
\end{proof}

\myskip
Similarly, we can adapt the Gaussian Poincar\'e inequality to the matrix setting. 
\begin{lemma}[Gaussian Poincaré Inequality]
    \noindent
Let $Z\in\bbR^n\iid\mathcal{N}(0,1)$, $Z'\sim\text{GOE}(d)$ independent of $Z$ and $g$ to a continuous and differentiable function with $\frac{\partial g}{\partial Z'_{ij}}=\frac{\partial g}{\partial Z'_{ji}}$ for $1\leq i,j\leq d$.
Then we have
\begin{equation*}
\text{Var}(g(Z,Z'))\leq\mathbb{E}\left[\sum_{\mu=1}^n\left(\frac{\partial g}{\partial Z_\mu}\right)^2+\frac{2}{d}\sum_{i,j=1}^d\left(\frac{\partial g}{\partial Z'_{ij}}\right)^2\right]
\end{equation*}
\label{lemma:gaussian-poincare}
\end{lemma}
\begin{proof}[Proof of Lemma \ref{lemma:gaussian-poincare}]
By the standard Gaussian Poincaré inequality, we have
\begin{equation*}
\begin{aligned}
\text{Var}(g(Z,Z'))&\leq\mathbb{E}\left[\sum_{\mu=1}^n\left(\frac{\partial g}{\partial Z_\mu}\right)^2+\frac{1}{d}\sum_{1\leq i<j\leq d}\left(2\frac{\partial g}{\partial Z'_{ij}}\right)^2+\frac{2}{d}
\sum_{i=1}^d\left(\frac{\partial g}{\partial Z'_{ii}}\right)^2\right]\\
&=\mathbb{E}\left[\sum_{\mu=1}^n\left(\frac{\partial g}{\partial Z_\mu}\right)^2+\frac{2}{d}\sum_{i,j=1}^d\left(\frac{\partial g}{\partial Z'_{ij}}\right)^2\right],
\end{aligned}
\end{equation*}
which finishes the proof.
\end{proof}

\subsection{Properties of two auxiliary channels}

We present here the properties of two auxiliary channels. The first one is the matrix denoising problem:
\begin{equation*}
Y'=\sqrt{r}S+Z,
\end{equation*}
with $Z\sim\GOE(d)$. Its free entropy is given by
\begin{equation*}
\tilde{f}_d(r)=\frac{1}{d^2}\mathbb{E}_{Y'}\log\int P_0(\rd S)e^{-\frac{d}{4}\Tr[(Y'-\sqrt{r}S)^2]},
\end{equation*}
and has the following property.
\begin{lemma}
    \noindent
Under Assumption \ref{ass:prior},
\begin{equation*}
\lim_{d\to\infty}\tilde{f}_d(r)=\psi_{P_0}(r).
\end{equation*}
Moreover, $\psi_{P_0}$ is a non-increasing, $\frac{\rho}{4}$-Lipschitz and convex function.
\label{lemma:Lipschitz_psi}
\end{lemma}
\begin{proof}[Proof of Lemma \ref{lemma:Lipschitz_psi}]
By the I-MMSE theorem \cite{guo2005mutual}, we have
\begin{equation*}
\tilde{f}_d'(r)=\frac{d}{4}\text{MMSE}(S^*|Y)-\frac{d}{4}\mathbb{E}[\Tr(S^*)^2],
\end{equation*}
which gives $\tilde{f}_d'(r)\leq0$, $\liminf_{d\to\infty}\tilde{f}_d'(r)\geq-\frac{\rho}{4}$ and that $\tilde{f}_d(r)$ is convex.
Therefore, $\{\tilde{f}_d(r)\}_{d\geq1}$ are non-increasing, uniformly Lipschitz and strictly convex. According to \cite[Theorem 4.3]{maillard2024bayes}, $\lim_{d\to\infty}\tilde{f}_d(r)=\psi_{P_0}(r)$, and thus $\psi_{P_0}(r)$ is non-increasing, $\frac{\rho}{4}$-Lipschitz, and convex.
\end{proof}
In the rectangular model, for $S \in \bbR^{d \times L}$, and $\{Z'_{ij}\}_{i,j=1}^{d,L}\overset{iid}{\sim}\mathcal{N}(0,\frac{1}{\sqrt{dL}})$, the free entropy of the denoising problem is given by
\begin{equation*}
\tilde{f}^{\text{rec}}_d(r)\coloneqq \frac{1}{dL}\mathbb{E}_{Y'}\log\int P_0(\rd S)e^{-\frac{1}{2}\sqrt{dL}r\Tr[S^TS]+\sqrt{dLr}\Tr[(Y')^TS]},
\end{equation*}
with $Y' = \sqrt{r} S^\star + Z$,
and satisfies the following property.
Recall the definition of $\rho$ in Proposition~\ref{prop:properties_P0_rec}.
\begin{lemma}
\label{lemma:HCIZ}
\noindent
Under Assumption~\ref{ass:prior_rec}, we have
\begin{equation*}
\lim_{d\to\infty}\tilde{f}^{\text{rec}}_d(r)=\psi^{\text{rec}}_{P_0}(r).
\end{equation*}
Moreover $\psi_{P_0}^{\text{rec}}$ is a non-increasing, $\frac{\rho}{2}$-Lipschitz, and convex function.
\end{lemma}
\begin{proof}[Proof of Lemma \ref{lemma:HCIZ}]
By \cite[Theorem 4]{pourkamali2024rectangular}, we have
\begin{equation*}
\lim_{d\to\infty}\tilde{f}^{\text{rec}}_d(r)=-\frac{\rho r}{2}+\beta J[\hmu,\hnu],
\end{equation*}
where $\hmu,\hnu$ refer to the limiting symmetrized singular value distributions of $A\coloneqq\sqrt{r/\sqrt{\beta}}S,B\coloneqq\sqrt{r/\sqrt{\beta}}S+Z'/\sqrt[4]{\beta}$, and 
\begin{equation*}
J[\hmu,\hnu]\coloneqq \lim_{d\to\infty}\frac{1}{d^2}\log\int\int \mcD U \mcD V \, e^{d\Tr[A^\T U B V^\T]}
\end{equation*}
is a so-called rectangular spherical integral~\cite{guionnet2021large} (here $\mcD$ is the Haar measure over the orthogonal group).
Notice that $Z'' \coloneqq Z'/\sqrt[4]{\beta}$ has i.i.d.\ elements with variance $1/d$.
By \cite{guionnet2021large}, we have
\begin{equation*}
\begin{aligned}
J[\mu, \nu]&=-(\beta^{-1}-1)\int\log|x|\hnu(x)dx-\Sigma[\hnu]+\frac{1}{2}\left[\int x^2\hmu(x)dx+\int x^2\hnu(x)dx\right]\\
&-I_{\hmu}(\hnu)+const,
\end{aligned}
\end{equation*}
where $I_\hmu(\hnu)$ is the large deviations rate function of the symmetrized empirical singular value distribution of the matrix 
$A + Z''$, over the law of the matrix $Z''$ which has i.i.d.\ $\mcN(0,1/d)$ elements, and in the scale $d^2$. 
Recall as well that $\Sigma(\mu) \coloneqq \int \mu(\rd x) \mu(\rd y) \log |x-y|$.
As $B = A + Z''$ has limiting symmetrized singular value distribution given by $\hnu$, by definition (with $\text{dist}$ a distance metrizing the weak convergence):
\begin{align*}
I_\hmu(\hnu) &\coloneqq -\lim_{\delta\to0}\lim_{d\to\infty}\frac{1}{d^2}\log\bbP(\text{dist}(\hmu_B,\hnu)\leq\delta), \\ 
&= 0.
\end{align*}
Moreover, by definition, we have $\int x^2\hmu(x)dx+\int x^2\hnu(x)dx=2 r\rho/\beta+const$. Notably, this establishes eq.~\eqref{eq:psi_P_rec}.
Repeating the arguments of the proof of Lemma \ref{lemma:Lipschitz_psi}, we obtain similarly that $\psi^{\text{rec}}_{P_0}(r)$ is non-increasing, $\frac{\rho}{2}$-Lipschitz, and convex.
\end{proof}

\myskip
The second model we consider is a scalar channel involving $P_\out$. We obtain the following property from~\cite{barbier2019adaptive}.
\begin{lemma}[{\cite[Proposition 18]{barbier2019optimal}}]
    \noindent
Suppose that $P_{out}(\cdot|x)$ is the law of $\varphi(x,a)+\sqrt{\Delta}Z$, where $\varphi$ is a bounded function with bounded first and second derivatives, then $\Psi_{\out}$ is convex, $C^2$ and non-decreasing on $[0,\rho]$, and thus it is also Lipschitz on $[0,\rho]$.
\label{lemma:Psi}
\end{lemma}

\section{Sketch of the Generalization to the Rectangular Setting}\label{sec_app:generalization_rectangular}
We very briefly discuss how our main results transpose straightforwardly to the rectangular setting.

\begin{itemize}
    \item The properties of the spiked model (Appendix \ref{sec_app:spike}), as well as the universality results (Appendix~\ref{sec_app:universality}), do not rely on whether the matrix is symmetric or rectangular, and thus can be immediately stated for this case.
    \item 
    The interpolation studied in Appendix~\ref{sec_app:proof_GLM} reads in the rectangular setting:
    \begin{equation}
    \left\{  
     \begin{aligned}
    &Y_t\sim P_{\out}(\cdot|J_t)  \\  
    &Y_t'=\sqrt{R_1(t)}S^*+Z'
    \label{eq:rec_interpolation}
     \end{aligned}  
    \right.  
    \end{equation}
    with 
    \begin{equation*}
    J_{t,\mu}\coloneqq\sqrt{1-t}\Tr[G_\mu^\T S^*]+\sqrt{2R_2(t)}V_\mu+\sqrt{2\rho t-2R_2(t)+2s_d}W_\mu^*,
    \end{equation*}
    $Z'_{ij}\iid\mathcal{N}(0,\frac{1}{\sqrt{dL}})$ for $i\in[d],j\in[L]$ and $\{V_\mu,W_\mu^*\}_{\mu=1}^n\overset{iid}{\sim}\mathcal{N}(0,1)$. The free entropy of the first channel (the first equation of \eqref{eq:rec_interpolation}) remains the same as in the symmetric setting, and the free entropy of the second channel (the second equation of \eqref{eq:rec_interpolation}) is given by Lemma \ref{lemma:HCIZ}. The adaptive interpolation argument, the concentration inequalities, and the relaxation of the assumptions exactly follow Appendix \ref{sec_app:proof_GLM}, with all the key components: potential truncation (Lemma \ref{lemma:truncation}), central limit theorem (Lemma \ref{lemma:CLT_new}), the Poincaré inequality (Lemma \ref{lemma:poincare}), the Gaussian integral by part (Lemma \ref{lemma:stein}) and the Gaussian Poincaré Inequality (Lemma \ref{lemma:gaussian-poincare}), all extending straightforwardly to the rectangular case. More generally, all the properties of symmetric matrices that we use in this paper also extend to the rectangular case when considering singular values (Proposition \ref{prop:properties_P0_rec}).
\end{itemize}

\end{document}